\documentclass[12pt]{article}
\usepackage{amsmath}
\usepackage{graphicx,psfrag,epsf,color}
\usepackage{enumerate}
\usepackage{natbib}
\usepackage{subcaption}
\usepackage{amsthm}
\usepackage{mathtools}
\usepackage{amsfonts} 
\usepackage{bm}

\usepackage[unicode=true,pdfusetitle,
   bookmarks=true, 
   bookmarksnumbered=false, 
   bookmarksopen=false,
   breaklinks=false, 
   pdfborder={0 0 1}, 
   backref=false, 
   colorlinks=false]{hyperref}
   
\usepackage{setspace}
\usepackage{amssymb}
\usepackage{multirow}
\usepackage{array}
\usepackage{booktabs}
\usepackage{float}
\usepackage[font=small,labelfont=bf]{caption}
\setcitestyle{citesep={;}}
\usepackage[total={6.4in,8.6in}]{geometry}
\newtheorem{theorem}{Theorem}
\newtheorem{assumption}{Assumption}[section]
\newtheorem{corollary}[theorem]{Corollary}

\newtheorem{remark}{Remark}

\usepackage{authblk}

\usepackage{algorithm}
\usepackage{algorithmicx}
\usepackage{algpseudocode}

\begin{document}

\def\spacingset#1{\renewcommand{\baselinestretch}%
{#1}\small\normalsize} \spacingset{1}
%
%
\title{\bf \LARGE Policy Optimization and Statistical Inference for Online Contextual Matrix Games}
\author[1]{Liner Xiang} 
\author[2]{Yixin Wang} 
\author[1]{Hengrui Cai \thanks{Corresponding author. Email: hengrc1@uci.edu}}
\affil[1]{Department of Statistics, University of California, Irvine} 
\affil[2]{Department of Statistics, University of Michigan} 
 \date{}
 \maketitle  

\baselineskip=21pt

\begin{abstract} 

Online decision making often requires navigating a landscape shaped by both dynamic contexts and strategic interactions. In competitive pricing, for example, hotels must account for both dynamic contextual factors and rivals' strategic responses. Existing approaches address only part of this challenge: contextual bandits optimize single-agent decisions using observable features but ignore multi-player interactions, while online matrix games capture strategic behavior through Nash equilibrium but assume fixed payoffs, ignoring contextual information. How should agents act then when strategic payoffs evolve with contextual signals?
We introduce \emph{online contextual matrix games} to integrate contextual information into multi-player online games. We further propose \emph{OnGameLearn}, an online learning algorithm that efficiently balances exploration and exploitation across both player actions and contexts. This approach comes with statistical guarantees: tail bounds for the estimated payoff matrix, the convergence of the estimated Nash equilibrium, the asymptotic normality of the parameter estimators, and the sublinear regret bound. We also develop the notion of \emph{policy value} in matrix games and develop a doubly robust, $\sqrt{T}$-consistent estimator for it. Across simulated studies and a real-world hotel pricing application, we find that OnGameLearn effectively navigates the intertwined challenges of strategic and contextual decision-making.

\end{abstract}

\section{Introduction}

Online decision-making often requires navigating a landscape shaped by both dynamic contexts and strategic interactions among agents. Consider, for instance, two competing hotels dynamically setting nightly room rates. Every day, they decide whether to raise, lower, or maintain prices while responding to both market conditions and the anticipated actions of competitors.  During high-demand periods, both hotels may profit from higher prices, yet the optimal strategy for each depends critically on its rival’s choice. If one hotel raises prices while the other holds steady, the former risks losing customers; if both maintain low prices, they forgo potential revenue. 
Conversely, in low-demand periods, aggressive discounting may attract customers, though mutual underpricing can erode profitability. 
This illustrates a fundamental challenge in online learning -- \textbf{strategic contextual decision-making}, where agents adapt policies that jointly account for dynamic contexts and strategic interactions among competitors. Such problems arise in diverse domains, including dynamic pricing \citep{BesbesZeevi2009}, electricity markets \citep{KakadeEtAl2013}, etc.
 
Existing approaches address only fragments of this challenge. Traditional single-agent online learning methods, including multi-armed bandits \citep{auer2002finite}, contextual bandits \citep{langford2007epoch}, and online Markov decision processes \citep{neu2010online}, focus on balancing exploration and exploitation to maximize cumulative rewards. Contextual bandits further use observable features to guide decisions, such as hotel pricing based on customer demographics, booking patterns, and seasonal trends. However, these models do not capture settings where outcomes depend on the simultaneous actions of strategically interacting agents \citep{maiti2023instance}.

Conversely, online matrix games capture strategic behavior through the concept of \textit{Nash equilibrium} \citep{nash1951non}, describing how rational players respond to each other's strategies. Two-player matrix games have long been fundamental to game theory and online learning~\citep{cesa2006prediction,neumann2007theory}. Formalized through payoff matrices \citep{broom1997multi,owen2013game}, they have been widely applied in financial markets \citep{thakor1991game}, economics \citep{karlin2003mathematical}, and voting systems \citep{rivest2010optimal}. Classic examples such as Rock-Paper-Scissors illustrate how fixed payoff structures determine win--lose relationships.
Yet, existing online matrix games \citep{lattimore2020bandit,cai2023uncoupled,maiti2023instance,li2024optimistic} typically assume time-invariant payoffs, overlooking contextual signals that reshape competitive interactions \citep{berg1998matrix,o2021matrix,maiti2025limitations,lin2025randomised}. Most existing work focuses on regret analysis and online performance guarantees, where payoff uncertainty is mainly introduced to study these objectives~\citep{o2021matrix,lin2025randomised}. For example, \citet{cai2023uncoupled} establishes uncoupled no-regret learning and last-iterate convergence for zero-sum matrix and Markov games with bandit feedback. However, these studies do not characterize how contextual signals change the payoff matrix and equilibrium strategies, nor do they provide statistical inference for these quantities.

In many practical environments, payoffs depend critically on observable contextual or player-specific characteristics, as demonstrated by the success of contextual bandits \citep{li2011unbiased,lu2021bandit,tajik2024novel}, with emerging demand for statistical inference \citep{chen2021statistical,shen2024doubly}. Another related line of work is Markov games and multi-agent reinforcement learning~\citep{littman1994markov,zhang2021multi}, which model strategic interactions through action-dependent state transitions. In contrast, our framework treats contexts as exogenous factors that affect the current payoff matrix. For example, in hotel pricing, market conditions affect the revenue structure but are not controlled by the hotels' current pricing decisions.

\begin{figure}[!t]
    \centering
    \includegraphics[width=0.8\linewidth]{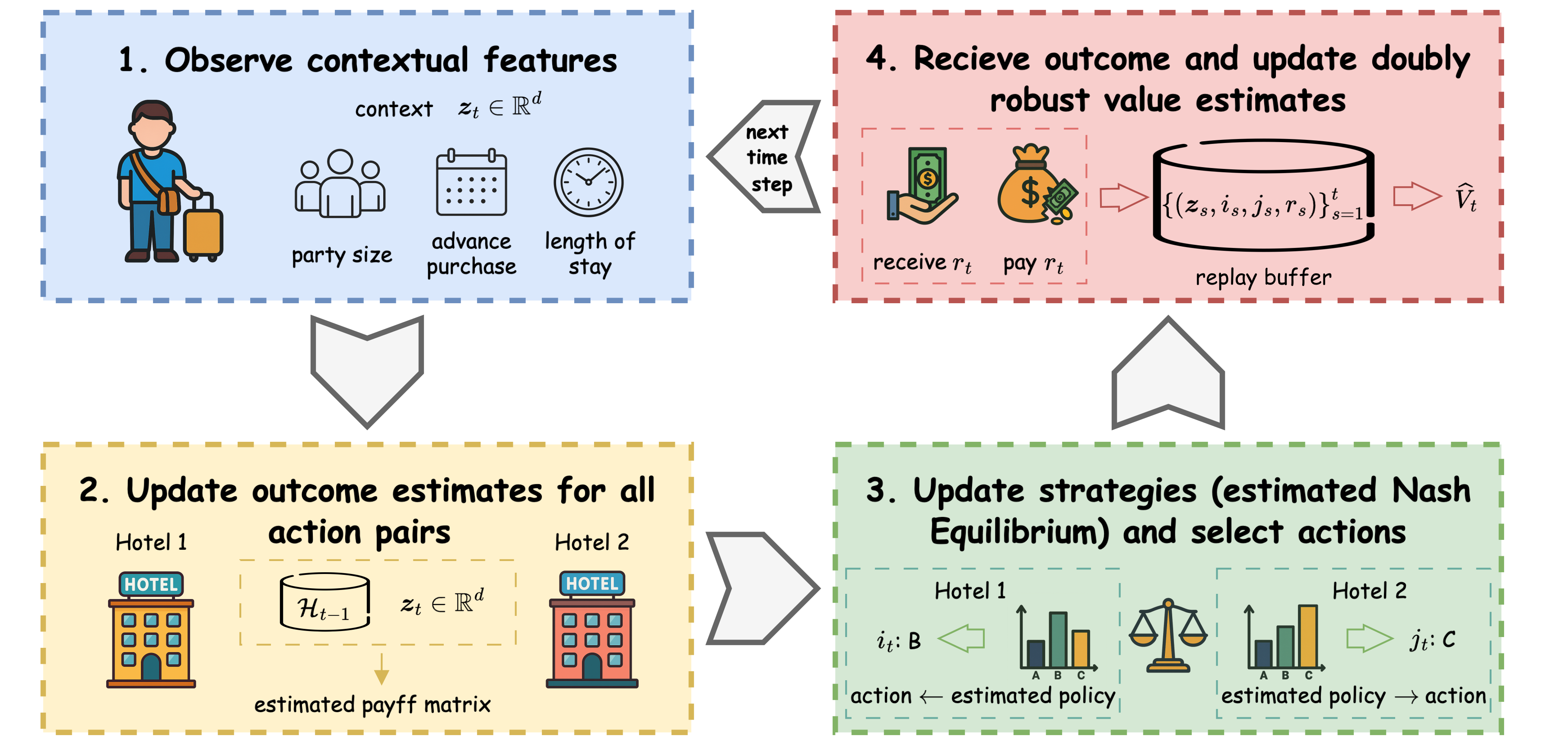}
    \caption{\textbf{Overview of OnGameLearn.} For contextual matrix games (e.g., hotel pricing), at each time step we: (1) observe the arriving customer's contextual features; (2) update the parameter estimates using historical data from the replay buffer and the payoff matrix; (3) compute the Nash equilibrium strategies and select actions accordingly; and (4) observe the outcome, update the doubly robust value estimate, and proceed to the next step. The multi-arm setting is simpler without contexts.}
    \label{fig:zero-sum}
\end{figure}

\underline{Neither framework handles the core question}: \textit{How should strategic agents act when payoffs evolve with contextual information?}
To address this fundamental challenge, we propose \textbf{online contextual matrix games} where the payoff matrix of each round depends on contextual information available at decision time. This formulation integrates contextual signals directly into multi-player strategic decision making, allowing both learning and equilibrium analysis to account for evolving environments. 
In the motivating hotel pricing example (Figure \ref{fig:zero-sum}), the payoff matrix -- representing profit outcomes for different pricing combinations between competing hotels -- continuously adapts to observable contexts such as party size, booking lead time, and length of stay. This in turn shapes each hotel's optimal strategy and the resulting payoffs. 
Focusing on competitive environments, we study the canonical case of two-player zero-sum games, where strategic equilibrium corresponds to the saddle point (i.e., Nash equilibrium) at which no player can gain any advantage by deviating unilaterally. This setting already poses fundamental challenges, while our framework could naturally extend to multi-player interactions once Nash equilibrium in such environments are reliably established.
Motivated by uncertainty and limited feedback in realistic online settings, we consider matrix games with \textit{noisy bandit feedback} \citep{lattimore2020bandit,o2021matrix}, where agents observe only stochastic, noisy feedback for chosen actions, capturing real-world perturbations rather than exact payoffs. 

Under the framework of online contextual matrix games with bandit feedback, we develop Nash equilibrium estimation procedures with provable convergence guarantees. 
This enables efficient optimal policy learning that balance exploration and exploitation across both player actions and contextual signals. Unlike contextual bandits with deterministic decision rules, matrix games require stochastic policies, i.e., probabilistic mixtures over actions, which our estimator handles effectively within the online learning paradigm \citep{li2011unbiased,dudik2011doubly,kallus2018policy}. To assess decision quality, we formalize the notion of \textit{policy value} for stochastic policies in matrix games and propose a doubly robust estimator for the game value under the Nash equilibrium. Moreover, under adaptively collected, non-i.i.d. data, we establish statistical inference procedures for the payoff matrix parameters, enabling principled strategy evaluation with quantified uncertainty. The resulting algorithm, \textbf{OnGameLearn} (Figure~\ref{fig:zero-sum}), unifies policy optimization and statistical inference under contextual and strategic uncertainty.

\textbf{Contributions.} The contribution of this work is three-fold. To our knowledge, 
we are the first to propose and study the online contextual matrix games. This shift substantially expands the practical applicability of online matrix games, as payoffs in real applications are rarely independent of observable features. To incorporate contextual information, we propose to parameterize the payoff matrix that adapts to the observed context at each time step, providing a formal treatment of contextual game learning with provable guarantees.

Second, under online matrix games with and without contexts, we develop a new online learning algorithm that includes regularization terms to learn Nash equilibrium and incorporates exploration algorithms, including Upper Confidence Bound (UCB)~\citep{li2011unbiased}, Thompson Sampling (TS)~\citep{agrawal2013thompson}, and $\epsilon$-Greedy~\citep{watkins1989learning,sutton1998reinforcement}, for valid statistical inference, enabling both players to dynamically update their strategies at each iteration.
The estimated Nash equilibrium is then obtained through efficient minimax optimization. We establish statistical guarantees for both the parameter estimates and the estimated Nash equilibrium. Unlike classical contextual bandits, where policies are typically smooth functions of model parameters, the Nash equilibrium is defined through a minimax problem, making standard theoretical tools less directly applicable. Prior work mainly studies algorithmic update rules and shows convergence to a Nash equilibrium through regret guarantees. In contrast, we directly establish the consistency of the Nash equilibrium induced by the estimated parameters. Building on this result, we further derive the asymptotic normality of the payoff matrix estimator in online matrix games and the parameter estimators in contextual matrix games, while accounting for dependence induced by adaptive data.

Third, we develop a $\sqrt{T}$-consistent doubly robust value estimator for online stochastic policies in matrix games, with and without contexts, and establish sublinear regret bounds for the proposed learning procedures. Unlike contextual bandits with deterministic policies \citep{shen2024doubly}, stochastic equilibrium strategies introduce additional inference challenges. We show that $\sqrt{T}$-consistency holds under a mild margin condition that accommodates both full-support and locally stable sparse Nash equilibria.

The rest of this paper is structured as follows. Section~\ref{sec:framework} formulates the problems of online matrix games and online contextual matrix games with bandit feedback. Section~\ref{sec:policy} presents our proposed OnGameLearn for both settings, with the theoretical results provided in Section~\ref{sec:theory}. To validate our approaches, Section~\ref{sec:simulation} reports numerical experiments for both the multi-armed and contextual settings, while Section~\ref{sec:real} demonstrates the practical utility of OnGameLearn via a real-world application. 
We also provide extended discussion about the convergence rate of estimated Nash equilibrium and analysis under model misspecification in Section~\ref{app:dis}.
Finally, Section~\ref{sec:conclusion} concludes the paper with future directions. All proofs and lemmas are provided in the supplementary material.

\section{Framework}
\label{sec:framework}

In this section, we present our framework for online matrix games with bandit feedback. Throughout this paper, unless otherwise specified, the terms \textit{online matrix games} and \textit{online contextual matrix games} refer to two-player zero-sum matrix games and two-player contextual zero-sum matrix games, respectively. We begin with the framework of online matrix games without contextual features, which serves as the foundation for developing the proposed online contextual matrix games. We then incorporate contextual information into this framework to formulate the setting of contextual matrix games, enabling decision-making that depends on observable contexts.

\smallskip
\noindent\textbf{Notation and Preliminaries.} We use $[T]$ to denote the set $\left\{1, 2, \cdots T\right\}$. By default, all vectors are column vectors. A vector $\bm{x}$ with entries $x_1, ..., x_d$ is written as $\bm{x}^\top = (x_1, \cdots, x_d)$, where $x_i$ is the $i$-th element. For a matrix $A$, let $A_{ij}$ and $[A_t]_{ij}$ be the entry in the $i$-th row and the $j$-th column of matrix $A$ and $A_t$. $\Delta_d$ denotes the $d$-dimensional probability simplex. \underline{Strongly Convex:} A function $f:X\rightarrow \mathbb{R}$ is strongly convex with respect to a norm $\Vert \cdot \Vert$, if  $f(\bm{x}_1) - f(\bm{x}_2) \ge \nabla f(\bm{x}_2)^\top (\bm{x}_1-\bm{x}_2)+({H}/{2} )\Vert \bm{x}_1-\bm{x}_2\Vert^2$, for any $\bm{x}_1, \bm{x}_2 \in X$ and $H>0$. Here, $\nabla f(\bm{x})$ denotes any subgradient of $f$ at $\bm{x}$. If $H=0$, then the function is convex. A function $f:X\rightarrow \mathbb{R}$ is strongly concave if $-f$ is strongly convex.  
\underline{Strongly Convex-Strongly Concave:}
A function $f(\bm{x}, \bm{y})$ is strongly convex-strongly concave if for any fixed $\bm{y}_0 \in Y$ , the function $f(\bm{x}, \bm{y}_0)$ is strongly convex in $\bm{x}$, and for any fixed $\bm{x}_0 \in X$, the function $f(\bm{x}_0, \bm{y})$ is strongly-concave in $\bm{y}$.

\subsection{Online Zero-sum Matrix Games} \label{subsec:2.1}
We study the following online zero-sum matrix games between two players who interact repeatedly over $T$ rounds. At each time $t$, the row player selects action $i_t \in [m]$ for $m$ number of choices, while the column player simultaneously chooses action $j_t \in [k]$ for $k$ number of choices, for some integers $m\geq 2$ and $k\geq 2$. Any such two-player zero-sum game can be characterized by a payoff matrix $A \in \mathbb{R}^{m\times k}$ \citep{v1928theorie,neumann2007theory}, which is \textit{unknown} in our setting. 
Our \textbf{objective} is threefold: (i) estimate the unknown payoff matrix, (ii) design optimal strategies for both players, and (iii) evaluate the induced game value. 
When actions $i_t$ and $j_t$ are played, the row player pays cost $r_t$ while the column player receives the reward $r_t$. Both players observe the opponent’s action and the corresponding reward under the current action pair, but obtain no information about the reward under other action pairs. This partial feedback setting is known as bandit feedback \citep{lattimore2020bandit}. Following \cite{o2021matrix} and \cite{li2024optimistic}, we model the outcome as $r_t = A_{i_t,j_t}+e_t$, where the noise term $e_t$ follows the sub-Gaussian distribution and conditioned on $i_t,j_t$ is independent of all previous information $\mathcal{H}_{t-1}$. Here we denote $\mathcal{H}_t = \left\{i_1,j_1,r_1,\dots, i_{t},j_{t},r_{t}\right\}$ as the information observed up to time $t$. We further assume that the conditional variance is $\mathbb{E}(e_t^2\vert i_t=i,j_t=j) = \sigma_{ij}^2$.

To characterize the optimal policies under strategic interactions, we introduce the \textbf{Nash equilibrium of the matrix game} as a pair of mixed strategies $(\bm{x}^\ast,\bm{y}^\ast)$ satisfying
$(\bm{x}^\ast)^\top A \bm{y} \le (\bm{x}^\ast)^\top A \bm{y}^\ast \le (\bm{x})^\top A \bm{y}^\ast$, 
for all $\bm{x}\in\Delta_m$ and $\bm{y}\in\Delta_k$, where $\Delta_m$ and $\Delta_k$ are the probability simplexes over the row and column players' action spaces, respectively.
This condition coincides with the \textit{saddle-point} definition of a general function $f: X \times Y \to \mathbb{R}$, such that
$f(\bm{x}^\ast, \bm{y}) \le f(\bm{x}^\ast, \bm{y}^\ast) \le f(\bm{x}, \bm{y}^\ast), \forall \bm{x}\in X, \bm{y}\in Y, $ 
when $f(\bm{x}, \bm{y}) = \bm{x}^\top A \bm{y}$. It is well known that if $f$ is convex–concave and $X, Y$ are convex and compact sets, then a saddle point always exists \citep{boyd2004convex}. Furthermore, if $f$ is strongly convex–strongly concave, the saddle point is unique. 

We define the \textit{policy value} of the matrix game $A$ under strategies $\bm{x}$ and $\bm{y}$ as $\bm{x}^\top A \bm{y}$.
Given the definition of the Nash equilibrium, we further define the \textit{optimal policy value} of the matrix game under the strategic equilibrium, denoted by $V_A^\ast$, as
\begin{equation}
    \label{equ:bandit_value}
    V_A^\ast \coloneqq (\bm{x}^\ast)^\top A \bm{y}^\ast = \min_{\bm{x} \in \Delta_m} \max_{\bm{y} \in \Delta_k} \bm{x}^\top A \bm{y} =  \max_{\bm{y} \in \Delta_k} \min_{\bm{x} \in \Delta_m} \bm{x}^\top A \bm{y},
\end{equation}
where the exchange of the min and max is guaranteed by Von Neumann’s minimax theorem \citep{v1928theorie}. 

\subsection{Online Contextual Zero-sum Matrix Games} 
\label{subsec:context}
We then introduce the online contextual matrix game setting by incorporating contextual features into the framework. At each time step $t$, contextual information $\bm{z}_t \in \mathbb{R}^d$ is revealed, and the payoff matrix $M(\bm{z}_t) \in \mathbb{R}^{m\times k}$ becomes a function of this context while preserving the action spaces $[m]$ and $[k]$ for both players. 
Given the chosen actions $i_t \in [m]$ and $j_t \in [k]$, we follow the linear contextual bandits formulation \citep{chen2021statistical,shen2024doubly} to model the observed outcome as 
$r_t = \bm{z}_t^{\top}\bm{\beta}^{i_t,j_t}+e_t$, 
where $\bm{\beta}^{ij} \in \mathbb{R}^d$ represents the parameter vector associated with action pair $(i,j)$. Similar to the online matrix games, the noise term $e_t$ is assumed to be sub-Gaussian and, conditional on $(i_t, j_t)$, independent of all past information $\mathcal{H}_{t-1}^{\mathrm{con}} = {\bm{z}_1, i_1, j_1, r_1, \ldots, \bm{z}_{t-1}, i_{t-1}, j_{t-1}, r_{t-1}}$. The conditional variance is also given by $\mathbb{E}(e_t^2 \mid i_t = i, j_t = j) = \sigma_{ij}^2$. Under this formulation, the contextual payoff matrix $M(\bm{z}_t)$ is defined such that each entry $M(\bm{z_t})_{ij} = \bm{z}_t^{\top}\bm{\beta}^{ij}$ encodes the expected outcome with action pair $(i,j)$ given context $\bm{z}_t$. 
Although we focus on a linear structure, the framework can be easily extended to more flexible parameterizations, such as basis expansions, provided that the payoff function can be consistently estimated and the resulting estimated Nash equilibrium converges.
When both players employ context-dependent mixed strategies $\phi(\bm{z}_t)$ and $\nu(\bm{z}_t)$, the expected outcome at time $t$ is $\phi(\bm{z}_t)^\top M(\bm{z}_t) \nu(\bm{z}_t)$. Under the distributional assumption $\bm{z}_t \sim \mathcal{P}_z$, we follow \citet{dudik2011doubly} to define the \textit{policy value} for contextual game strategies $\phi(\cdot)$ and $\nu(\cdot)$ as:
\begin{equation*}
    V(\phi,\nu) \coloneqq \mathbb{E}_{\bm{z}\sim \mathcal{P}_z} \left[\phi(\bm{z})^\top M(\bm{z}) \nu(\bm{z}) \right].
\end{equation*}
The optimal strategies $\phi^\ast(\bm{z})$ and $\nu^\ast(\bm{z})$ correspond to the Nash equilibrium of the contextual matrix game $M(\bm{z})$, yielding the \textit{optimal policy value}:
\begin{equation*}
    V^\ast \coloneqq \mathbb{E}_{\bm{z}\sim \mathcal{P}_z} \left[\min_{\bm{x}\in \Delta_{m}}\max_{\bm{y}\in \Delta_{k}} \bm{x}^\top M(\bm{z}) \bm{y} \right] = \mathbb{E}_{\bm{z}\sim \mathcal{P}_z} \left[\phi^\ast(\bm{z})^\top M(\bm{z}) \nu^\ast(\bm{z})\right].
\end{equation*}

\section{Proposed Method: OnGameLearn}
\label{sec:policy}

We now introduce our proposed method, OnGameLearn, where both players estimate the (contextual) payoff matrix and select actions using the UCB, TS or $\epsilon$-Greedy strategy for exploration, followed by doubly robust policy value estimation. We first present the method for online matrix games in Section~\ref{subsec:method_matrix}, and for its contextual version in Section~\ref{subsec:method_con}.

\subsection{Method for Online Matrix Games}
\label{subsec:method_matrix}

\textbf{Payoff matrix estimates.} 
In the online learning setting, where optimal strategies are unknown and only bandit feedback is observed. At time $t$, we update the payoff matrix estimates using the empirical mean of historical samples from $\mathcal{H}_{t-1}$ \citep{o2021matrix,wijewardena2025online,maiti2025limitations}. Note that only $[A_t]_{i_t,j_t}$ is observed with the noise and the information under other action pairs will be unknown. The online estimator is given by
\begin{equation}
\label{equ:ahat}
    [\widehat{A}_t]_{i j} = \left(\sum_{s=1}^{t}\mathbb{I}(i_s=i,j_s=j) \right)^{-1} \left( \sum_{s=1}^{t}r_s \mathbb{I}(i_s=i,j_s=j) \right),
\end{equation}
for $i=1,2,\dots,m$, $j=1,2,\dots,k$. Given the estimated payoff matrix $\widehat{A}_{t-1}$, we next describe how to compute the estimated Nash equilibrium at time $t$. To ``convexify'' the targeted objective,
we define the regularized objective function at time $t$ as:
\begin{equation}
    \mathcal{L}_t(\bm{x},\bm{y}) = \bm{x}^\top \widehat{A}_{t-1} \bm{y} - \eta_t R_X(\bm{x})+\eta_t R_Y(\bm{y}),
\end{equation}
where $R_X(\bm{x}) \coloneqq \sum_{i=1}^{m}x_i\ln x_i + \ln{m}$ and $ R_Y(\bm{y}) \coloneqq \sum_{i=1}^{k}y_i\ln y_i + \ln{k}$ are strongly convex regularization terms ensuring a unique saddle point of $\mathcal{L}_t(\bm{x},\bm{y})$. Similar regularization terms have been widely used in the  literature~\citep[see e.g.,][]{grnarova2017online,rivera2025online}. 
Then, we estimate the unique Nash equilibrium of this regularized objective function, which serves as the strategies of the two players.

\noindent\textbf{Estimated Nash equilibrium and behavior policy.}
Let $(\widehat{\bm{x}}_t,\widehat{\bm{y}}_t)$ denote the unique Nash equilibrium of $\mathcal{L}_t(\bm{x},\bm{y})$ and the saddle value of $\mathcal{L}_t$ as $\mathcal{L}_t(\widehat{\bm{x}}_t,\widehat{\bm{y}}_t)$ \citep{milgrom2002envelope}. 
We define $X^\ast:\mathbb{R}^{m \times k} \times \mathbb{R} \rightarrow \Delta_m$ and $Y^\ast:\mathbb{R}^{m \times k} \times \mathbb{R} \rightarrow \Delta_k$ as the set-valued functions for the Nash equilibrium of the two players with parameters $A$ and $\eta$, respectively.
Thus,
\begin{equation}
\label{equ:NE_bandit}
\begin{aligned}
    \widehat{\bm{x}}_t
     =
    X^\ast(\widehat{A}_{t-1},\eta_t)
    =
    \arg\min_{\bm{x} \in \Delta_m}
    \max_{\bm{y} \in \Delta_k}
    \mathcal{L}_t(\bm{x},\bm{y}), \quad
    \widehat{\bm{y}}_t
    =
    Y^\ast(\widehat{A}_{t-1},\eta_t)
    =
    \arg\max_{\bm{y} \in \Delta_k}
    \min_{\bm{x} \in \Delta_m}
    \mathcal{L}_t(\bm{x},\bm{y}).
\end{aligned}
\end{equation}

To ensure sufficient exploration, we consider three standard bandit algorithms: Upper Confidence Bound (UCB) \citep{li2011unbiased}, Thompson Sampling (TS) \citep{agrawal2013thompson}, and $\epsilon$-Greedy \citep{sutton1998reinforcement}. For UCB and TS, an exploratory payoff matrix $\widetilde{A}_t$ is constructed and its Nash equilibrium strategies are clipped to ensure sufficient exploration.  For $\epsilon$-Greedy, each player follows the estimated Nash equilibrium strategy with probability $1-\epsilon_t$ and explores uniformly over its action space with probability $\epsilon_t$. 
We denote the corresponding exploratory strategies of the two players as $\bm{\widetilde{x}}_t$ and $\bm{\widetilde{y}}_t$. 
For all three exploration algorithms, the resulting behavior policy has a unified expression as $\pi_t(i,j)$ based on the estimated payoff matrix with a clipping parameter $\epsilon_t$ to ensure positivity. Formally, $\pi_t(i,j) = \widetilde{x}_{ti}\widetilde{y}_{tj} \ge \epsilon_t^2/mk, \, i\in[m],\, j\in[k]$, where $\widetilde{x}_{ti}$ is the $i$-th element of vector $\bm{\widetilde{x}}_t$ and analogous definition is applied for $\widetilde{y}_{tj}$. 
Detailed constructions of the three exploration policies are provided in Section~\ref{app:bandit_alg} of the supplementary material.

\noindent\textbf{Doubly robust value estimator.} 
Doubly robust methods have been widely applied in online policy learning \citep{shen2024doubly,xu2024}, where consistency holds if either the propensity score model (similar to the estimated strategies in matrix games) or the outcome regression model is correctly specified. Given the estimated payoff matrix and estimated mixed strategies, we can derive an estimator for the value of matrix game $V_A^\ast$. Drawing inspiration from the doubly robust estimator for stochastic policy in offline settings \citep{jiang2016doubly}, we extend this approach to the online setting and propose the following doubly robust expected outcome estimator as the value estimator under OnGameLearn as
\begin{equation}
\label{equ:dr}
    \widehat{V}_T^{\text{DR}} \equiv \frac{1}{T}\sum_{t=1}^T \widehat{\bm{x}}_t^\top \widehat{A}_{t-1}\widehat{\bm{y}}_t + \frac{\widehat{x}_{t,i_t}\widehat{y}_{t,j_t}}{{\pi}_t(i_t,j_t)}(r_t - [\widehat{A}_{t-1}]_{i_t,j_t}).
\end{equation}

We name our method as \textbf{OnGameLearn}, which efficiently learns the Nash equilibrium and estimates game value. Algorithm \ref{BOMG} provides the detailed pseudocode. It runs through three key components iteratively: (1) a payoff matrix estimator that aggregates historical observations, (2) a strategy optimizer that computes regularized Nash equilibrium, and (3) a doubly robust value estimator that combines model predictions with actual rewards.

\begin{algorithm}[hpt]
\caption{Online Matrix Games with Bandit Feedback}
\label{BOMG}
\begin{algorithmic}[1]
\State \textbf{Input:} termination time $T$, \(\widehat{\bm{x}}_1 \in \Delta_{m} \), \(\widehat{\bm{y}}_1 \in \Delta_{k}\), 
regularization parameters $\left\{\eta_t \right\}_{t=1}^T$, exploration rate $\left\{\epsilon_t \right\}_{t=1}^T$.
\For{\(t = 1\) \textbf{to} \(T\)}
    \State Update estimated payoff matrix $\widehat{A}_t$ using Equation \eqref{equ:ahat}.
    \State Update estimated Nash Equilibrium $(\widehat{\bm{x}}_t,\widehat{\bm{y}}_t)$ by Equations \eqref{equ:NE_bandit}.
    \State Sample $i_t$ and $j_t$ by bandit algorithms for two players.
    \State Observe bandit feedback $r_t$ and set behavior policy $\pi_t(i,j)=\widetilde{x}_{ti} \widetilde{y}_{tj}$.
    \State Update the doubly robust value estimator $\widehat{V}_t^{\text{DR}}$ by Equation \eqref{equ:dr}.     
\EndFor
\end{algorithmic}
\end{algorithm}

To compute the Nash equilibrium at each time step $t$ (Step 4 in Algorithm~\ref{BOMG} and Step 5 in Algorithm~\ref{BOCG} later), we employ a classical approach, gradient descent–ascent (GDA) method \citep[see, e.g.,][]{du2019linear, lin2020gradient}. The computational complexity is discussed in Section~\ref{app:computation} of the supplementary material. Guided by Theorems~\ref{thm:tail_bandit} and~\ref{thm:saddlerate}, introduced in Section~\ref{sec:theory_bandit}, we choose $\epsilon_t$ such that $t \epsilon_t^2/ \log t \to \infty$ and $\eta_t$ is at most of the order $t^{-1/2}$.
The same order of parameters is applied in Algorithm~\ref{BOCG}, as justified by Theorems~\ref{conthm:tail} and~\ref{conthm:drrate} for online contextual matrix games.

\subsection{Method for Online Contextual Matrix Games}
\label{subsec:method_con}

\noindent\textbf{Parameter estimates and value estimation.} In contextual matrix games, we estimate the strategies sequentially using the observation history up to time $t$, denoted by $\mathcal{H}^{\mathrm{con}}_{t}$, through ordinary least squares. The estimator for each action pair $(i, j)$ is given by:
\begin{equation}
    \label{equ:para_con}
    \widehat{\bm{\beta}}^{ij}_t = \left(\frac{1}{t} \sum_{s=1}^{t} \mathbb{I}\left\{i_s=i,j_s=j \right\} \bm{z}_s \bm{z}_s^{\top}\right)^{-1} \left( \frac{1}{t} \sum_{s=1}^t \mathbb{I}\left\{i_s=i,j_s=j \right\} \bm{z}_s r_s \right),
\end{equation}
for $i\in[m]$ and $j \in [k]$. To ensure the uniqueness of the Nash equilibrium, we also add the same regularization terms $R_X$ and $R_Y$. Denote $\bm{x}^\top \widehat{M}_{t-1}(\bm{z}_t)\bm{y}-\eta_t R_X(\bm{x}) + \eta_t R_Y(\bm{y})$ as $\mathcal{L}_t^{\mathrm{con}}(\bm{x,\bm{y}})$, where $[\widehat{M}_{t-1}(\bm{z}_t)]_{ij}$ is constructed with entries $\bm{z}_t^\top \widehat{\bm{\beta}}^{ij}_{t-1}$. Then the estimated equilibrium strategies at time $t$ are given by,
\begin{equation}
\begin{aligned}
\label{equ:policy_con}
    \widehat{\phi}_t (\bm{z}_t) =X^\ast(\widehat{M}_{t-1}(\bm{z}_t),\eta_t)= \arg\min_{\bm{x} \in \Delta_m} \max_{\bm{y}\in \Delta_k} \mathcal{L}_t^{\mathrm{con}}(\bm{x,\bm{y}}), \\
    \widehat{\nu}_t (\bm{z}_t) =Y^\ast(\widehat{M}_{t-1}(\bm{z}_t),\eta_t)= \arg \max_{\bm{y}\in \Delta_k} \min_{\bm{x} \in \Delta_m} \mathcal{L}_t^{\mathrm{con}}(\bm{x,\bm{y}}).
\end{aligned}
\end{equation}
This regularized saddle-point problem can be
solved using a gradient descent–ascent procedure similarly. 
Following constructions in online matrix game setting, we extend UCB, TS, and $\epsilon$-Greedy to the contextual setting. 
For UCB and TS, the exploratory Nash equilibrium strategies are constructed from the context-dependent payoff matrix and clipped using the exploration level $\epsilon_t$, ensuring a positive lower bound on the selection probability of each action pair. While $\epsilon$-greedy mixes the estimated equilibrium strategies with uniform exploration. These guarantee the positivity condition required for estimation and inference. Denote $\widetilde{\phi}_t(\bm z_t)$ and $\widetilde{\nu}_t(\bm z_t)$ as the corresponding exploratory strategies of the two players. 
Consequently, the resulting behavior policy under all three exploration mechanisms is 
$\pi_t^{\mathrm{con}}(i,j\mid \bm z_t)
=
[\widetilde{\phi}_t(\bm z_t)]_i
[\widetilde{\nu}_t(\bm z_t)]_j
\geq
{\epsilon_t^2}/{mk}$, for 
$i\in[m]$, and $j\in[k]$.  Detailed constructions are provided in Section~\ref{app:bandit_alg} of the supplementary material.
The doubly robust value estimator is then expressed as
\begin{equation}
\label{equ:dr_con}
    \widehat{V}_T^{\text{conDR}} = \frac{1}{T}\sum_{t=1}^T \widehat{\phi}_t(\bm{z}_t)^\top \widehat{M}_{t-1}(\bm{z}_t)\widehat{\nu}_t(\bm{z}_t) + \frac{[\widehat{\phi}_t(\bm{z}_t)]_{i_t}[\widehat{\nu}_t(\bm{z}_t)]_{j_t}}{{\pi}_t^{\mathrm{con}}(i_t,j_t \vert \bm{z}_t)}(r_t - [\widehat{M}_{t-1}(\bm{z}_t)]_{i_t,j_t}).
\end{equation}
The detailed pseudocode is presented in Algorithm \ref{BOCG}.

\begin{algorithm}
\caption{Online Contextual Matrix Games with Bandit Feedback}
\label{BOCG}
\begin{algorithmic}[1]
\State \textbf{Input:} termination time $T$, \(\widehat{\bm{x}}_1 \in \Delta_{m} \), \(\widehat{\bm{y}}_1 \in \Delta_{k}\), regularization parameters $\left\{\eta_t \right\}_{t=1}^T$, exploration rate $\left\{\epsilon_t \right\}_{t=1}^T$.
\For{\(t = 1\) \textbf{to} \(T\)}
    \State Sample $d$-dimensional contextual information $\bm{z}_t \sim \mathcal{P}_z$.
    \State Update $\widehat{\bm{\beta}}_{t-1}^{ij}$ for all $i,j$ using Equation \eqref{equ:para_con} and $\widehat{M}_{t-1}(\bm{z}_t)$.
    \State Update strategies $\widehat{\phi}_t(\bm{z}_t)$ and $\widehat{\nu}_t(\bm{z}_t)$ by Equations \eqref{equ:policy_con}.
    \State Sample $i_t$ and $j_t$ using the bandit algorithms and observe $r_t$.
    \State Update the value estimator $\widehat{V}_t^{\text{conDR}}$ by Equation \eqref{equ:dr_con}.
\EndFor
\end{algorithmic}
\end{algorithm}

\section{Theoretical Results}
\label{sec:theory}

In this section, we establish the theoretical results for the estimation of both the payoff matrix or parameters and the policy value under the online matrix game and contextual matrix game settings, presented in Section~\ref{sec:theory_bandit} and Section~\ref{sec:theory_con}, respectively.

\subsection{Results for Online Matrix Games}
\label{sec:theory_bandit}
In this section, we present the theoretical results for the online matrix game setting. We first establish the tail bound of the payoff matrix estimator, followed by the convergence of the estimated Nash equilibrium as well as the behavior strategies and the asymptotic normality of the payoff matrix estimator. We then analyze the convergence rate of the saddle value under OnGameLearn and demonstrate the convergence rate of the doubly robust value estimator. Lastly, we provide the regret analysis of our proposed method OnGameLearn. Complete proofs are provided in Section~\ref{app:proof_bandit} of the supplementary material.

The theoretical analysis begins with a tail bound for the payoff matrix estimator under the exploration rate $\epsilon_t$. The proof is provided in Section~\ref{app:tail_bandit_proof} of the supplementary material.

\begin{theorem}[{\textbf{\emph{Tail bound of payoff matrix estimator}}}]
\label{thm:tail_bandit}
When applying UCB, TS, $\epsilon$-Greedy for two-player zero-sum online matrix games, with $\epsilon_t$ is non-increasing, then for any $h > 0$ and $i=1,2,\cdots,m,\, j =1,2,\cdots,k$, we have
\begin{equation*}
\mathbb P\left(
\left|
[\widehat A_t]_{ij}-A_{ij}
\right|
\le h
\right)
\ge
1
-
\exp\left\{
-\frac{t \epsilon_t^2}{8mk}
\right\}
-
2\exp\left\{
-\frac{t \epsilon_t^2 h^2}{4mk\sigma_{ij}^2}
\right\}.
\end{equation*}
\end{theorem}
Theorem \ref{thm:tail_bandit} establishes the consistency of the payoff matrix estimator under the condition that $t\epsilon_t^2 \to \infty$. 
Corollary~\ref{coro:consistent_bandit} can be derived straightforwardly as follows.

\begin{corollary}[\textbf{\emph{Consistency of payoff matrix estimator}}]
\label{coro:consistent_bandit}
When non-increasing $\epsilon_t$ satisfies $t\epsilon_t^2 \to \infty$ as $t \to \infty$, $[\widehat{A}_t]_{ij}$ is a consistent estimator for $A_{ij}$, $i \in [m]$, $j \in [k]$.
\end{corollary}

Having established the consistency of the payoff matrix estimator, we now turn to the estimation of the Nash equilibrium.

\begin{theorem}[{\textbf{\emph{Consistency with respect to the Nash equilibrium set}}}]
\label{thm:nash_con}
Let $\mathcal E(A)$ is the Nash equilibrium set of the zero-sum game with true payoff matrix $A$. For any $(\bm x',\bm y')\in\Delta_m\times\Delta_k$, define its distance to the equilibrium set as
\[
\operatorname{dist}
\left(
(\bm x',\bm y'),
\mathcal E(A)
\right)
=
\inf_{(\bm x,\bm y)\in\mathcal E(A)}
\sqrt{
\left\|
\bm x'-\bm x
\right\|_2^2
+
\left\|
\bm y'-\bm y
\right\|_2^2
}.
\]
Under the conditions stated in Corollary~\ref{coro:consistent_bandit}, if $\eta_t\to 0$, then 
$\operatorname{dist}
\left(
(\widehat{\bm x}_t,\widehat{\bm y}_t),
\mathcal E(A)
\right)
\stackrel{p}{\to}
0$.
\end{theorem}

Theorem~\ref{thm:nash_con} establishes the convergence of the estimated Nash equilibrium to the true equilibrium set under diminishing regularization parameters $\eta_t \to 0$ and consistent payoff matrix estimation. The proof is provided in Section~\ref{app:nash_con_proof} of the supplementary material, while the convergence rate of the estimated Nash equilibrium, including the contextual version, is further discussed in Section~\ref{app:con_rate_ne}.

Conditional on the history $\mathcal H_{t-1}$, we obtain the payoff matrix estimator $\widehat A_t$ and the corresponding estimated Nash equilibrium $(\widehat{\bm x}_t,\widehat{\bm y}_t)$ at time $t$. 
While convergence to the Nash equilibrium set does not require uniqueness, deriving an explicit asymptotic distribution for $\widehat A_t$ requires the estimated equilibrium policies to converge to a uniquely defined limit. Accordingly, we assume that the true payoff matrix admits a unique Nash equilibrium, as is common in the existing literature~\citep{bailey2018multiplicative,daskalakis2018last,wei2020linear,maiti2025limitations}.

\begin{assumption}[{\emph{Uniqueness}}]
\label{ass:unique}
    Payoff matrix $A$ has a unique Nash equilibrium $(\bm{x}^\ast,\bm{y}^\ast)$.
\end{assumption}
Assumption \ref{ass:unique} is commonly adopted in prior work \citep{bailey2018multiplicative,daskalakis2018last,wei2020linear,maiti2025limitations}, which enables us to prove convergence of $(\widetilde{\bm{x}}^\ast,\widetilde{\bm{y}}^\ast)$ to the unique Nash equilibrium $(\bm{x}^\ast,\bm{y}^\ast)$. 
The proofs are provided in Section~\ref{app:nash_behave_proof} of the supplementary material.

\begin{theorem}[{\textbf{\emph{Convergence of the behavior strategies}}}]
\label{thm:behave_con}
Suppose that all the conditions in Theorem~\ref{thm:nash_con} and
Assumption~\ref{ass:unique} hold. Then, with $\epsilon_t \to \epsilon_\infty$ and $c_t/\sqrt{(t-1)\epsilon_{t-1}^2} \to 0$, $\widetilde{\bm x}_t$ and $\widetilde{\bm y}_t$ converge in probability to limiting strategies $\widetilde{\bm x}^\ast$ and $\widetilde{\bm y}^\ast$, respectively. 
\end{theorem}
Theorem~\ref{thm:behave_con} characterizes the limiting behavior policy under each exploration strategy. In particular, for every action pair $(i,j)$, the conditional selection probability satisfies $\pi_t(i,j) \stackrel{p}{\to}  \widetilde{x}_i^\ast\widetilde{y}_j^\ast$, where $\widetilde{\bm{x}}^\ast = \left(\widetilde x^\ast_1, \widetilde x^\ast_2, \cdots, \widetilde x^\ast_{m}\right)^\top$ and $\widetilde{\bm{y}}^\ast = \left(\widetilde y^\ast_1, \widetilde y^\ast_2, \cdots, \widetilde y^\ast_{k}\right)^\top$. These limiting selection probabilities determine the asymptotic variance of the payoff matrix estimator. In the following theorem, we use this convergence to establish asymptotic normality of $\widehat A_t$ by applying the Martingale Central Limit Theorem~\citep{hall1980martingale}.

\begin{theorem}[{\textbf{\emph{Asymptotic normality of payoff matrix estimator}}}]
\label{thm:norm_matrix}
    Suppose all the conditions in Theorem~\ref{thm:behave_con} hold, for all $i \in [m]$, $j \in [k]$, we have
    $\sqrt{t}\left([\widehat{A}_t]_{ij} - A_{ij}\right) \stackrel{D}{\to} \mathcal{N}(0,s_{ij}^2)$,
    where $s_{ij}^2 = \left(\widetilde{x}_i^\ast\widetilde{y}_j^\ast\right)^{-1} \sigma_{ij}^2$.
     A consistent estimator for $s_{ij}^2$ is  given by,
    \begin{equation}
    \label{equ:var_est}
        \frac{t\sum_{s=1}^t \mathbb{I}(i_s=i,j_s=j) \widehat{e}_s^2}{\left(\sum_{s=1}^t \mathbb{I}(i_s=i,j_s=j)\right)^2},
    \end{equation}
    where $\widehat{e}_s = r_s-[\widehat{A}_t]_{i_s,j_s}$.
\end{theorem}

As shown the expressions of $s_{ij}^2$ and Theorem~\ref{thm:behave_con}, the asymptotic variance of the payoff matrix estimator depends on both the limiting exploration rate $\epsilon_\infty$ and the true Nash equilibrium. The full proof is in Section~\ref{app:norm_matrix_proof} of the supplementary material.

We next investigate the convergence of the doubly robust value estimator. Since the estimated Nash equilibrium is obtained for $\mathcal{L}_t$ rather than the estimated payoff matrix $\widehat{A}_t$ directly, it is necessary to use $\mathcal{L}_t$ as a bridge to quantify the distance between the optimal policy value $V_A^\ast$ and the estimator $\widehat{V}_T^{\text{DR}}$. To this end, we first establish Corollary~\ref{thm:NEvalue}. Building on the asymptotic normality of the payoff matrix estimator and its ${T}^{-1/2}$ convergence rate, we can derive the rate at which the saddle value of $\mathcal{L}_t$ approaches the optimal policy value $V_A^\ast$. The proof is provided in Section \ref{app:saddlevalue_rate} of the supplementary material.

\begin{corollary}[\textbf{{\emph{Convergence rate of saddle value}}}]
\label{thm:NEvalue}
Under the conditions of Theorem \ref{thm:norm_matrix}, and further assuming that the payoff matrix estimator is bounded, for all $t \ge 1$ and for each $i, j$, the following convergence rate holds for the saddle value:
$$\mathcal{L}_t(\widehat{\bm{x}}_t,\widehat{\bm{y}}_t) - V_A^\ast = O_p\left((t^{-1}+\eta_t^2)^{\frac{1}{2}}\right).$$
\end{corollary}

Corollary~\ref{thm:NEvalue} demonstrates the convergence rate of the saddle value $\mathcal{L}_t(\widehat{\bm{x}}_t,\widehat{\bm{y}}_t)$ under payoff matrix estimator $\widehat{A}_t$ and regularization parameter $\eta_t$. Specifically, when $\eta_t = \mathcal{O}(t^{-1/2})$, we have $\mathcal{L}_t(\widehat{\bm{x}}_t,\widehat{\bm{y}}_t) - V_A^\ast = O_p(t^{-1/2})$, which plays a pivotal role in proving Theorem \ref{thm:saddlerate}. 
The margin condition for the Nash Equilibrium in Assumption~\ref{ass:mix} is also required in Theorem~\ref{thm:saddlerate} to guarantee the $\sqrt{T}$-consistency of the doubly robust value estimator.

\begin{assumption}[{\emph{Margin}}] 
\label{ass:mix} 
For the Nash equilibrium $(\bm x^\ast,\bm y^\ast)$ of the true payoff matrix $A$, one of the following conditions holds: 
\textnormal{(i)} Positive equilibrium probabilities are bounded away from zero; or
\textnormal{(ii)} the equilibrium support is locally stable.
\end{assumption}

Assumption~\ref{ass:mix} is mild, covering both full-support equilibria and sparse equilibria with locally stable support. It rules out knife-edge cases where small perturbations change the equilibrium support and ensures $\sqrt{T}$-consistency of the doubly robust value estimator. A formal statement and proof are provided in Section~\ref{app:dr} of the supplementary material.

\begin{theorem}[\textbf{\emph{$\sqrt{T}$-consistency of value estimator}}]
\label{thm:saddlerate}
    Suppose that Assumption~\ref{ass:mix} and the conditions in Corollary~\ref{thm:NEvalue} hold. Assume further that $\eta_t = O(t^{-1/2})$. For UCB and TS, additionally assume that ${t\epsilon_t^2}/{\log t} \to \infty$, then
        $\sqrt{T} \left(\widehat{V}_T^{\text{DR}}- V_A^\ast \right) = O_p(1)$.
\end{theorem}
The optimal rate is constrained by the $\mathcal{O}(t^{-1/2})$ decay of the regularization parameter $\eta_t$; a slower decay of $\eta_t$ would lead to a slower convergence rate. The $\sqrt{T}$-consistency further enables the use of bootstrap methods~\citep{li2025proximal} to construct a confidence interval for the optimal game value, thereby completing the inference on the optimal value.

The last part is the regret analysis. 
Following~\citet{cai2023uncoupled}, for strategy pair $(\bm x_t,\bm y_t)\in\Delta_m\times\Delta_k$, define the Nash gap as $\operatorname{Gap}_A(\bm x_t,\bm y_t)
= \max_{\bm y\in\Delta_k}
\bm x^\top A\bm y_t
- \min_{\bm x\in\Delta_m}
\bm x_t^\top A\bm y$.
The cumulative Nash-gap regret is defined as $\operatorname{Reg}_A(T)
=
\sum_{t=1}^T
\operatorname{Gap}_A
\left(
\widetilde{\bm x}_t,
\widetilde{\bm y}_t
\right)$.
The Nash gap is nonnegative and equals zero if and only if the corresponding strategy pair is a Nash equilibrium. It measures the exploitability of the implemented behavior strategies and is therefore a natural regret criterion for evaluating convergence to the Nash equilibrium.

\begin{theorem}[\textbf{{\emph{\textbf{Regret bound}}}}]
\label{thm:regret_bandit}
For the UCB strategy, suppose that $c_t = O(\sqrt{\log t})$. Then, under each of the three exploration strategies, the cumulative Nash-gap regret satisfies 
\begin{equation*}
\operatorname{Reg}_A(T)
=
\widetilde O_p
\left(
\sqrt{mk}
\sum_{t=1}^T
\frac{1}{\epsilon_t\sqrt t}
+
\sum_{t=1}^T
\epsilon_t
+
\sum_{t=1}^T
\eta_t
\right).    
\end{equation*}
Let $\epsilon_t\asymp t^{-1/4}$ and $\eta_t = O(t^{-1/2})$. Then
$\operatorname{Reg}_A(T) = \widetilde O_p \left( \sqrt{mk}\,T^{3/4} \right)$.
\end{theorem}

Theorem~\ref{thm:regret_bandit} does not require Assumptions~\ref{ass:unique} and~\ref{ass:mix}, and thus allows multiple Nash equilibria and zero equilibrium probabilities. The regret bound consists of payoff-estimation error, exploration cost, regularization bias, and, for UCB and TS, an additional perturbation term. Choosing $\epsilon_t\asymp t^{-1/4}$ balances the leading terms and yields the smallest regret rate. Details are provided in Section~\ref{app:thm:rerget_proof} of the supplementary material.

\subsection{Results for Online Contextual Matrix Games}
\label{sec:theory_con}
Following Section~\ref{sec:theory_bandit}, we first establish the tail bound for the parameter estimator and the convergence of the estimated Nash equilibrium and behavior strategies. We then derive asymptotic normality, the saddle-value convergence rate, the $\sqrt{T}$-consistency of the doubly robust value estimator, and the regret bound. Complete proofs are provided in Section~\ref{app:proof_context} of the supplementary material. The corresponding theoretical results under model misspecification are established in Section~\ref{app:mis_model}.

We begin with Assumption~\ref{ass:bound}, a standard condition in contextual bandits \citep{chen2021statistical,tajik2024novel,shen2024doubly}, to derive the tail bound for the parameter estimator. The proof of Theorem~\ref{conthm:tail} is provided in Section~\ref{app:conthm:tail_proof}.

\begin{assumption}[\emph{Boundness}]
\label{ass:bound}
    There exists a constant positive $L_z$ such that $\Vert \bm{z} \Vert_\infty \le L_z$ for any $\bm{z}$ and $\Sigma = \mathbb{E}( \bm{z} \bm{z}^\top)$ has minimum eigenvalue $\lambda_{\min}(\Sigma) > \lambda$ for some $\lambda > 0$.
\end{assumption}

\begin{theorem}[{\textbf{\emph{Tail bound of parameter estimator}}}]
\label{conthm:tail}
With Assumption \ref{ass:bound} hold, when applying OnGameLearn with three bandit algorithms for online contextual matrix games, with $\epsilon_t$ is non-increasing, then for any $h > 0$ and $i=1,2,\cdots,m,\, j =1,2,\cdots,k$, we have
\small{
\begin{align*}
\mathbb P\left(
\left\|
\widehat{\bm\beta}_t^{ij}
-
\bm\beta^{ij}
\right\|_1
\le
h\right) \ge
1-d\exp\left\{
-
\frac{
t\epsilon_t^2\lambda}{
8mkL_z^2}
\right\}
-
\left(
1+tL_z^2
\right)^{d/2}
\exp\left\{
-
\frac{
t^2\epsilon_t^4\lambda^2h^2}{8dm^2k^2\sigma_{ij}^2
\left[
1+\frac{t\epsilon_t^2\lambda}{2mk}
\right]
}
\right\}.
\end{align*}
}
\end{theorem}
The first term vanishes under $t\epsilon_t^2\to\infty$, as in Corollary~\ref{coro:consistent_bandit}. In the contextual setting, the additional factor $\left(1+tL_z^2\right)^{d/2}$ requires the stronger condition ${t\epsilon_t^2}/{\log t}\to\infty$, which ensures consistency in Corollary~\ref{concoro:consistent}. Details are in Section~\ref{app:conthm:coro_tail_proof} of the supplementary material.

\begin{corollary}[{\textbf{\emph{Consistency of payoff matrix estimator}}}]
\label{concoro:consistent}
Suppose conditions in Theorem \ref{conthm:tail} are satisfied, when non-increasing $\epsilon_t$ satisfies $t\epsilon_t^2/\log t \to \infty$ as $t \to \infty$, $\widehat{\bm{\beta}}_t^{ij}$ is a consistent estimator for $\bm{\beta}^{ij}$, $i \in [m]$, $j \in [k]$.
\end{corollary}

With the consistency of the parameter estimators, which ensures that the payoff matrix is accurately estimated given the contexts $\bm{z}_t$ at time $t$, we next analyze the convergence of the estimated mixed strategies in the contextual setting.

\begin{theorem}[{\textbf{\emph{Convergence of estimated Nash Equilibrium}}}]
\label{conthm:nash_con}
Under the conditions of Corollary \ref{concoro:consistent}, if the parameter estimates $\widehat{\bm{\beta}}^{ij}_t$ remain bounded and $\eta_t \to 0$ as $t \to \infty$, we have
    $\operatorname{dist} \left( (\widehat{\phi}_t(\bm{z_t}), \widehat{\nu}_t(\bm{z_t})),\mathcal{E}(M(\bm{z}_t)) \right) \stackrel{p}{\to} 0.$
\end{theorem}

The above theorem shows that the estimated mixed strategies $\widehat{\phi}(\bm{z}_t)$ and $\widehat{\nu}(\bm{z}_t)$ converge in $\ell_2$-norm to the optimal strategy sets, respectively. Following the similar arguments as in Theorem~\ref{thm:nash_con}, we can similarly prove Theorem~\ref{conthm:nash_con}, with details provided in Section~\ref{app:conthm:nash_con_proof} of the supplementary material.

Similarly, while maintaining the idea from the online matrix games framework (Assumption \ref{ass:unique}), Assumption \ref{conass:unique} states the uniqueness of the Nash equilibrium to hold across all contextual features $\bm{z}\sim\mathcal{P}_z$ to show the convergence of the behavior strategies. The proofs are provided in Section~\ref{app:conthm:nash_behave_proof} of the supplementary material.

\begin{assumption}[\emph{Uniqueness}]
\label{conass:unique}
    For $\forall \bm{z} \sim \mathcal{P}_z$, $M(\bm{z})$ has a unique Nash equilibrium $(\phi^\ast(\bm{z}),\nu^\ast(\bm{z}))$.
\end{assumption}

\begin{theorem}[\textbf{\emph{Convergence of the contextual behavior strategies}}]
\label{conthm:behave_con}
Suppose that all the conditions in Theorem~\ref{conthm:nash_con} and
Assumption~\ref{conass:unique} hold. If $\epsilon_t\to\epsilon_\infty$ and $c_t/\sqrt{(t-1)\epsilon_{t-1}^2} \to 0$ as $t \to \infty$, then we have 
$\Vert \widetilde{\phi}_t(\bm z_t) - \widetilde{\phi}^{\ast}(\bm z_t) \Vert_2 \stackrel{p}{\to} 0$, and
$\Vert \widetilde{\nu}_t(\bm z_t) - \widetilde{\nu}^{\ast}(\bm z_t) \Vert_2 \stackrel{p}{\to} 0.$
\end{theorem}
Together with the convergence of the estimated parameters in Corollary~\ref{concoro:consistent}, the asymptotic normality of the online parameter estimator then also follows from the Martingale Central Limit Theorem \citep{hall1980martingale}.

\begin{theorem}[\textbf{{\emph{Asymptotic normality of estimated parameters}}}]
\label{conthm:para}
    Suppose conditions in Theorem \ref{conthm:behave_con} are satisfied, $\sqrt{t}\left(\widehat{\bm{\beta}}_t^{ij} - \bm{\beta}^{ij}\right) \overset{D}{\to} \mathcal{N}(0,S_{ij}^2)$, where
    \begin{equation}
    \label{var_true_con}
        S_{ij}^2 = \sigma_{ij}^2 \left\{ \int [\widetilde \phi^\ast(\bm{z})]_i [\widetilde \nu^\ast(\bm{z})]_j \bm{z} \bm{z}^\top \,d\mathcal{P}_z \right\}^{-1},
    \end{equation}
    \vspace{-0.2cm}
    with $\sigma_{ij}^2 = \mathbb{E}(e_s^2\vert i_s=i,j_s=j)$. A consistent estimator for $S_{ij}^2$ is given by
    \begin{equation}
    \label{var_est_con}
    \frac{\sum_{s=1}^{t} \mathbb{I}\left\{i_s=i,j_s=j \right\}\widehat{e}_s^2}{\sum_{s=1}^{t} \mathbb{I}\left\{i_s=i,j_s=j \right\}} \left( \frac{1}{t} \sum_{s=1}^{t} \mathbb{I}\left\{i_s=i,j_s=j \right\} \bm{z}_s \bm{z}_s^T \right)^{-1},
    \end{equation}
    where $\widehat{e}_s = r_s - \bm{z}_s^\top\widehat{\bm{\beta}}^{ij}_t$.
\end{theorem}

As demonstrated by Equation \eqref{var_true_con} and Theorem~\ref{conthm:behave_con}, the asymptotic variance of the parameter estimator depends on the optimal strategies $\phi^\ast(\cdot)$ and $\nu^\ast(\cdot)$ through the expectation’s integrand, with details provided in Section \ref{app:conthm:para_proof} of the supplementary material. This result is consistent with the asymptotic variance established in Theorem \ref{thm:norm_matrix} but consider the variability of contexts.
Similarly, to quantify the distance between $\widehat{V}_T^{\text{conDR}}$ and $V^\ast$, we first establish the convergence rate of $\mathcal{L}_t^{\mathrm{con}}$ toward $V^\ast$ as a bridge. The proof of Corollary~\ref{conthm:NEvalue} is in Section~\ref{app:conthm:NEvalue_proof} of the supplementary material.

\begin{corollary}[\textbf{{\emph{Convergence rate of saddle value}}}]
\label{conthm:NEvalue}
Under the conditions of Theorem~\ref{conthm:para}, and assuming that for each $i$, $j$, and any $\bm{z} \sim \mathcal{P}_z$, the estimated payoff entries $[\widehat{M}_{t-1}(\bm{z}_t)]_{ij}$ are bounded for all $t \ge 1$, the following convergence rate holds:
\[
\vert {\mathcal{L}}_t^{\mathrm{con}}(\widehat{\phi}_t(\bm{z}_t),\widehat{\nu}_t(\bm{z}_t)) - \phi^\ast(\bm{z}_t)^\top M(\bm{z}_t) \nu^\ast(\bm{z}_t) \vert = O_p\left((t^{-1}+\eta_t^2)^{1/2}\right).
\]
\end{corollary}

Corollary \ref{conthm:NEvalue} establishes that the convergence rate in the contextual setting matches the result for online matrix games (Corollary \ref{thm:NEvalue}), provided the context space is bounded. 
Furthermore, to establish the optimal $\sqrt{T}$-consistency of our doubly robust value estimator, we impose a margin condition analogous to Assumption~\ref{ass:mix}. The formal statement is given in Assumption~\ref{conass:margin_formal} in Section~\ref{app:conthm:drrate_proof} of the supplementary material.

\begin{assumption}[\emph{Margin}]
    \label{conass:margin}
    For any $\bm{z} \sim \mathcal{P}_z$, and the corresponding Nash equilibrium $\phi^\ast(\bm z), \nu^\ast(\bm z)$ of the true payoff matrix $M(\bm z)$, one of the following conditions holds: 
    \textnormal{(i)} The equilibrium probabilities are bounded away from zero; or 
\textnormal{(ii)} the support of the Nash equilibrium is locally stable around the true payoff matrix. 
\end{assumption}

Assumption~\ref{conass:margin} controls the inverse probability weights across contexts and then ensures the $\sqrt{T}$-consistency of the doubly robust value estimator stated as follows. 

\begin{theorem}[\textbf{{\emph{$\sqrt{T}$-consistency of value estimator}}}]
\label{conthm:drrate}
Suppose conditions in Corollary \ref{conthm:NEvalue} and Assumption \ref{conass:margin} are satisfied. Assume further that $\eta_t = O(t^{-1/2})$. Then, 
\begin{equation*}
        \sqrt{T} \left(\widehat{V}_T^{\text{conDR}}- V^\ast \right) = O_p(1).
    \end{equation*}
\end{theorem}
The proof of Theorem~\ref{conthm:drrate} follows the similar approach as that of Theorem~\ref{thm:saddlerate}, with proofs provided in Section~\ref{app:conthm:drrate_proof} of the supplementary material. We can still obtain the same convergence rate of $\widehat{V}_T^{\text{conDR}}$ even under the online contextual matrix game setting.

Finally, we also consider the regret bound for the contextual matrix games. For any context $\bm z_t$ and any strategy pair $(\bm x_t,\bm y_t)\in\Delta_m\times\Delta_k$, the
contextual Nash gap is $\operatorname{Gap}_{M(\bm z_t)} (\bm x_t,\bm y_t)=\max_{\bm y\in\Delta_k}\bm x^\top M(\bm z_t)\bm y_t-\min_{\bm x\in\Delta_m}\bm x_t^\top M(\bm z_t)\bm y$. The cumulative contextual Nash-gap regret is then given by $\operatorname{Reg}_{\mathrm{con}}(T)=\sum_{t=1}^T\operatorname{Gap}_{M(\bm z_t)}\left(\widetilde{\bm x}_t,\widetilde{\bm y}_t \right)$.

\begin{theorem}[{\textbf{\emph{Regret bound}}}]
\label{conthm:regret_bandit}
Suppose Assumption~\ref{ass:bound} holds. Let $\epsilon_t\asymp t^{-1/4}$ and $\eta_t=O(t^{-1/2})$.
For the UCB strategy, additionally suppose $c_t=O(\sqrt{\log t})$.
Then, under each of the three exploration strategies, the cumulative contextual Nash-gap regret satisfies
\[
\operatorname{Reg}_{\mathrm{con}}(T)
= \widetilde O_p
\left(
\sqrt{mk}\,T^{3/4}
\right).
\]
\end{theorem}
The proof of Theorem~\ref{conthm:regret_bandit} follows that of Theorem~\ref{thm:regret_bandit}, with the fixed payoff matrix $A$ replaced by the contextual matrix $M(\bm z_t)$. Assumption~\ref{ass:bound} is used to control the contextual payoff-matrix estimation error uniformly over the realized contexts. The remaining arguments are analogous to the online matrix setting.

\section{Empirical Studies}
\label{sec:simulation}

In this section, we evaluate the performance of OnGameLearn through empirical studies on online matrix games with and without contexts. 
The results corroborate the theoretical findings in Section~\ref{sec:theory}: OnGameLearn yields valid inference for the estimated payoff matrix or parameters, achieves convergence of the estimated Nash equilibrium, and attains $\sqrt{T}$-consistent value estimation.

\subsection{Simulation Studies for Online Matrix Games}
\label{sec:simu_multiarm}

\noindent\textbf{Simulation setup.} The data are generated as follows. We consider a $2\times2$ payoff matrix $A = ((2, 1), (0.5, 2))$ with noise modeled by $e_t|i_t=i,j_t=j \sim \mathcal{N}(0,0.1^2)$ for all $i \in [m]$ and $j \in [k]$. The Nash equilibrium of matrix $A$ is $\bm{x}^\ast = (0.6,0.4)^\top$ and $\bm{y}^\ast = (0.4,0.6)^\top$. Then the value of matrix game $A$ can be derived as $V^\ast_A=1.4$. We set the termination time $T=10,000$ with the exploration rate as $\epsilon_t = 0.1 t^{-1/4}$ and the regularization parameter $\eta_t = t^{-1/2}$. For the UCB strategy, we set $c_t=1$. We run 500 replications following Algorithm~\ref{BOMG} under both the UCB-based and $\epsilon$-Greedy strategies.
For baselines, we include EXP3~\citep{auer2002nonstochastic,o2021matrix} and Tsallis-INF~\citep{zimmert2021tsallis,ito2025instance}. EXP3 is a classical adversarial-bandit algorithm and has been widely used as a benchmark for matrix games with bandit feedback. Tsallis-INF is a no-regret bandit algorithm with strong guarantees in both stochastic and adversarial settings, and has also been studied for self-play in two-player zero-sum games.

\begin{figure}[!t]
    \centering
    \includegraphics[width=1.0\linewidth]{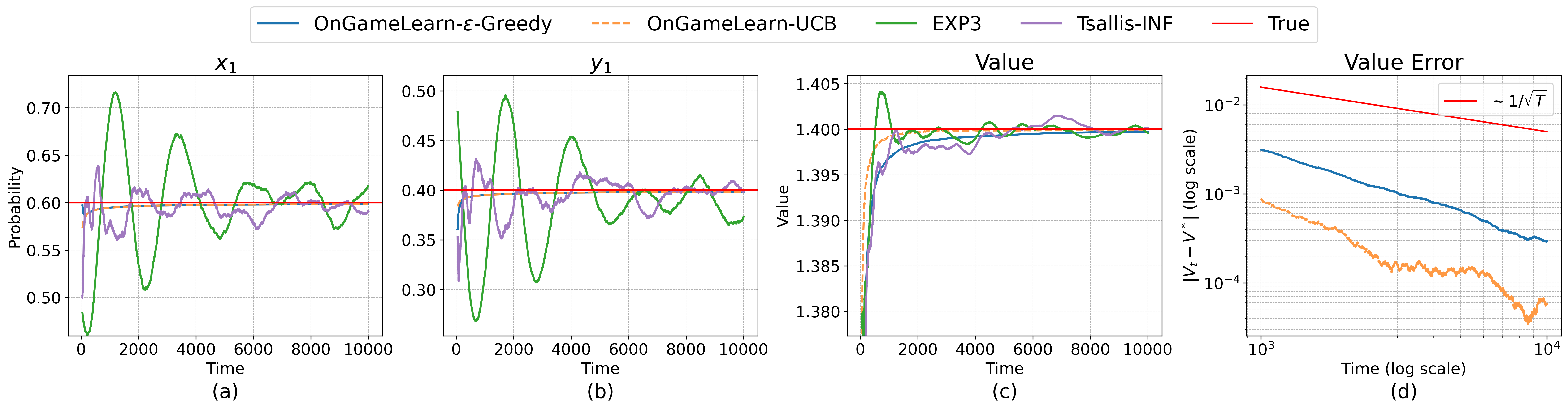}
    \caption{\textbf{Panels (a) and (b)} illustrate the convergence of the estimated Nash equilibrium components $\widehat{x}_{t1}$ and $\widehat{y}_{t1}$ toward their true values $x_1^\ast$ and $y_1^\ast$, respectively. The red horizontal lines indicate the corresponding true Nash equilibrium values. The curves for OnGameLearn under $\epsilon$-Greedy and UCB exploration almost completely overlap. \textbf{Panels (c) and (d)} illustrate the $\sqrt{T}$-consistency of the doubly robust value estimator $\widehat{V}_T^{\mathrm{DR}}$. Panel (c) shows the trajectory of the average estimated value $\widehat{V}_T^{\mathrm{DR}}$, with the red horizontal line indicating the optimal policy value $V_A^\ast=1.4$. Panel (d) presents a log--log plot of $|\widehat{V}_T^{\mathrm{DR}}-V_A^\ast|$. The red line represents the theoretical ${O}_p(T^{-1/2})$ convergence rate.}
    \label{fig:multi_NE}
\end{figure}

\begin{figure}
    \centering
    \includegraphics[width=0.75\linewidth]{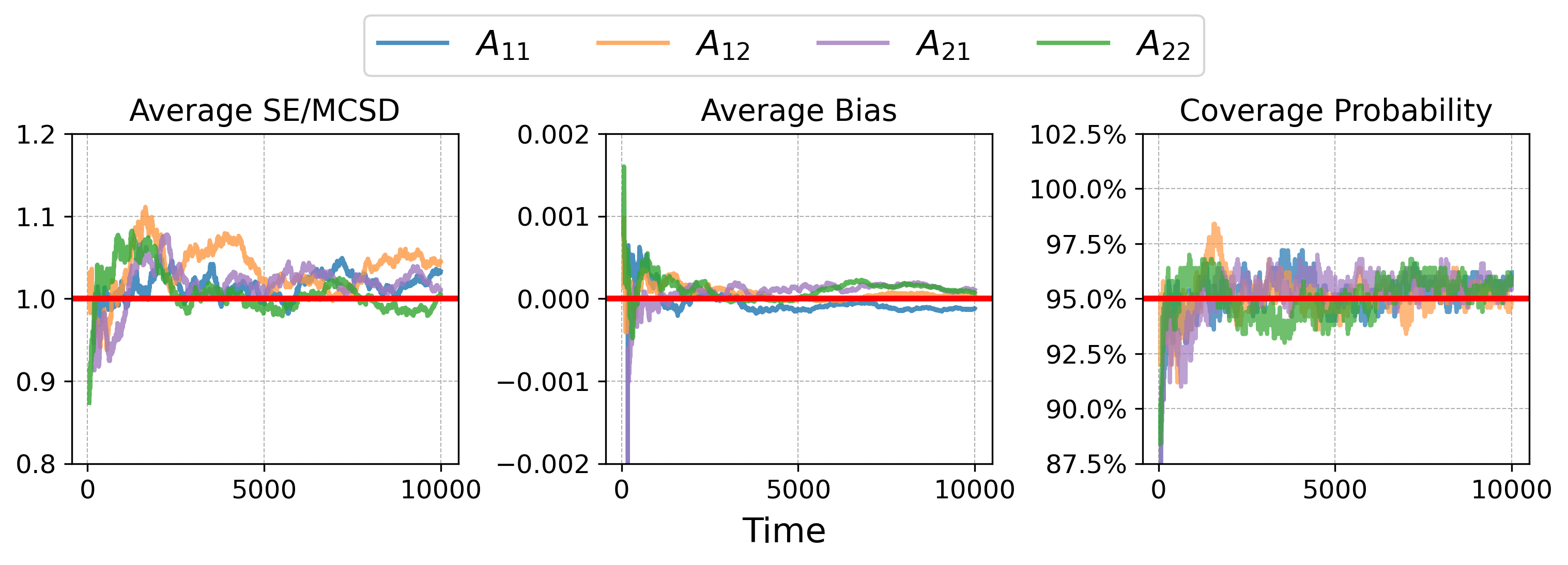}
    \caption{Consistency and asymptotic normality of payoff estimator $\widehat{A}_t$ under $\epsilon$-Greedy exploration. \textbf{Left panel}: the ratio between the standard error and the Monte Carlo standard deviation, with the red line indicating the nominal level of 1. \textbf{Middle panel}: the bias between the $\widehat{A}_t$ and $A$. \textbf{Right panel}: the coverage probabilities of the 95\% two-sided Wald-type confidence interval, with the red line indicating the nominal level of 95\%.}
    \label{fig:multi_est}
\end{figure}

\noindent\textbf{Convergence of estimated optimal policy.} 
Figure~\ref{fig:multi_NE} (a) and (b) show the convergence of $\widehat{x}_{t1}$ and $\widehat{y}_{t1}$ to the true equilibrium values $x_1^\ast$ and $y_1^\ast$, respectively. Since each player has two actions, the second component is determined by $\widehat{x}_{t2}=1-\widehat{x}_{t1}$ and $\widehat{y}_{t2}=1-\widehat{y}_{t1}$. The red horizontal lines denote the true equilibrium values. Both variants of OnGameLearn converge steadily and remain stable after the initial exploration period, while the baseline methods converge more slowly and exhibit larger fluctuations. In particular, Tsallis-INF assumes bounded payoffs, which is violated in our setting.

\noindent\textbf{Statistical inference of payoff matrix.} Figure \ref{fig:multi_est} presents the performance evaluation of OnGameLearn under $\epsilon$-Greedy through three key metrics: the average ratio of the standard error (SE) to the Monte Carlo standard deviation (MCSD), (2) the average parameter estimation bias computed, and (3) the coverage probability of 95\% two-sided Wald confidence intervals. These confidence intervals are constructed using the parameter estimates and their corresponding standard errors derived from Equation \eqref{equ:var_est} in Theorem \ref{thm:norm_matrix}. The results demonstrate strong empirical performance: the SE/MCSD ratio converges to 1, indicating proper variance estimation; the estimation bias approaches zero, showing unbiasedness; and the coverage probabilities for all $A_{ij}$ parameters attain the nominal 95\% level. Note that the baseline methods, EXP3 and Tsallis-INF, are designed for regret minimization and do not provide statistical inference for the payoff parameters.

\noindent\textbf{$\sqrt{T}$-consistent value estimator.} Figure \ref{fig:multi_NE} (c) illustrates the empirical performance of the doubly robust value estimator $\widehat{V}_T^\text{DR}$. The results demonstrate strong consistency: $\widehat{V}_T^\text{DR}$ converges to the optimal policy value $V^\ast_A$, as evidenced by the left panel. In contrast, the baseline methods exhibit larger fluctuations and slower stabilization; in particular, other baseline estimates remain noticeably more variable or display a persistent finite-sample bias over the observed time horizon. 
Furthermore, the log-log plot of estimation errors (Figure \ref{fig:multi_NE} (d)) exhibits parallel slopes aligned with the theoretical $1/\sqrt{T}$ benchmark, confirming the expected $\sqrt{T}$-convergence rate stated in Theorem \ref{thm:saddlerate}. 

Additional results for Tsallis-INF under different hyperparameter choices, as well as statistical inference results under UCB exploration, are provided in Section~\ref{app:matrix_simu} of the supplementary material.

\subsection{Simulation Studies for Online Contextual Matrix Games}
\label{sec:simu_context}
\noindent\textbf{Simulation setup.} 
We conduct experiments for contextual matrix games following Algorithm~\ref{BOCG}. The simulation setup incorporates contextual information with two covariates and set $d=3$ with accounting for 1 as the intercept. The context vector $\bm{z} = (1,z_1,z_2)$ is generated as follows: $z_1 \stackrel{\text{i.i.d}}{\sim} \text{Uniform}(0, 5)$, and $z_2$ follows a truncated normal distribution with mean zero, variance one, and support $[0, 10]$, ensuring Assumption \ref{ass:bound} is satisfied, while Assumption~\ref{ass:mix} is verified numerically. The game involves $m=2$ actions for the row player and $k=2$ actions for the column player. The true parameters are specified as: $\bm{\beta}^{11} = (1,2,2)^\top$, $\bm{\beta}^{12} = (-2,-1,1)^\top$, $\bm{\beta}^{21} = (-1,1,-2)^\top$, $\bm{\beta}^{22} = (3,1,2)^\top$. We set regularization parameter $\eta_t=10t^{-1/2}$. We use the same time horizon $T$, noise terms $e_t$, exploration rate $\epsilon_t$, and UCB hyperparameter $c_t$ as in Section~\ref{sec:simu_multiarm} with 500 replications. 
We also include the contextual extensions of the baseline methods considered in Section~\ref{sec:simu_multiarm}, namely, EXP4~\citep{auer2002nonstochastic} and LC-Tsallis-INF~\citep{kato25lc}.

\begin{figure}[!t]
    \centering
    \includegraphics[width=1.0\linewidth]{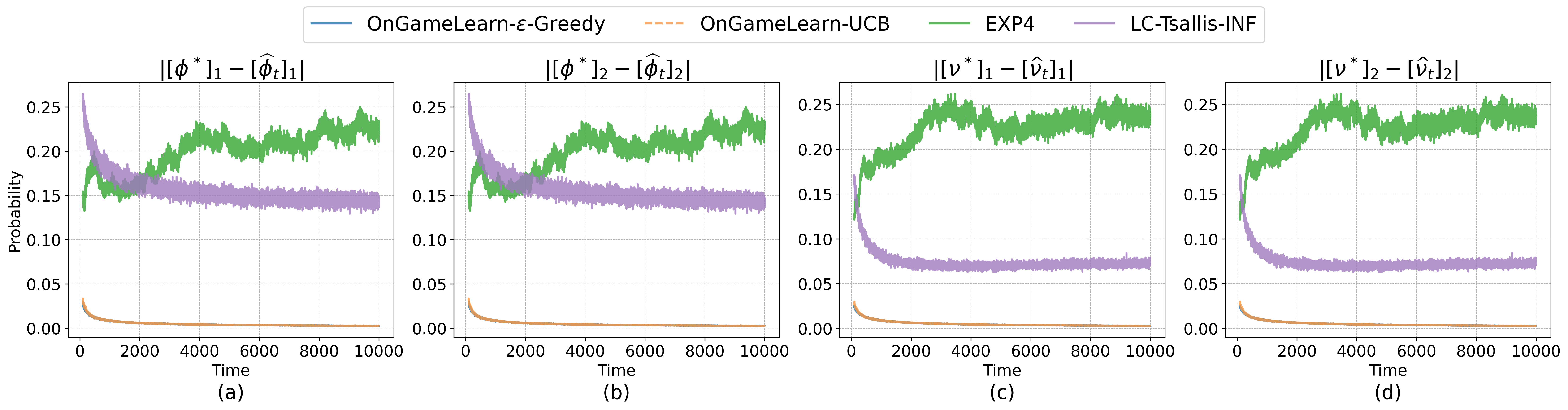}
    \caption{Convergence of the policy, measured by the distances between estimated mixed strategies $[\tilde{\phi}_t]_1$, $[\tilde{\phi}_t]_2$, $[\tilde{\nu}_t]_1$, $[\tilde{\nu}_t]_2$ and the true Nash equilibrium $[{\phi}^\ast]_1$, $[{\phi}^\ast]_2$, $[{\nu}^\ast]_1$, $[{\nu}^\ast]_2$. The curves for OnGameLearn under $\epsilon$-greedy and UCB exploration almost completely overlap.}
    \label{fig:con_NE}
\end{figure}

\begin{figure}[!t]
    \centering
    \includegraphics[width=1.0\linewidth]{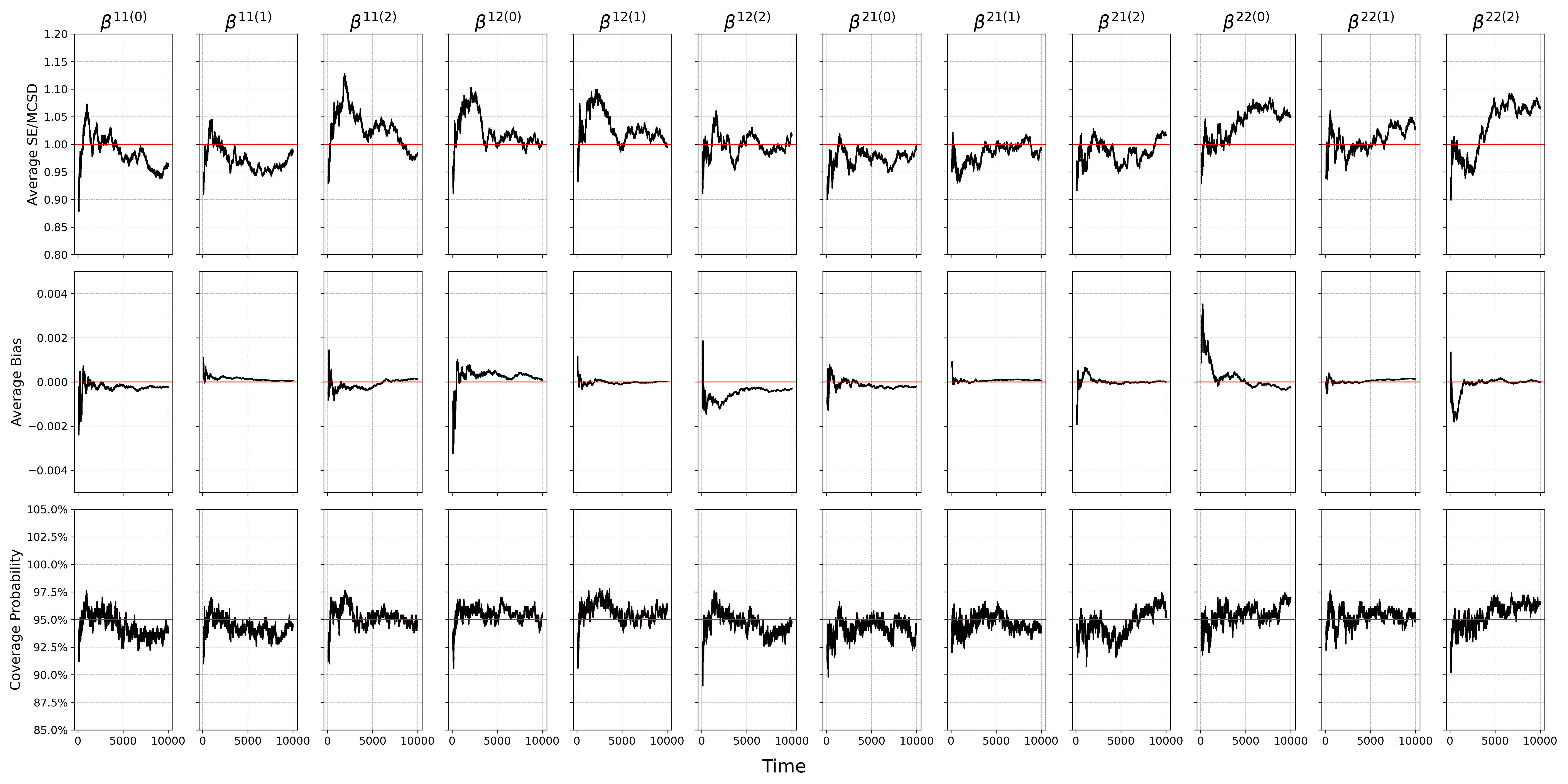}
    \caption{Consistency and asymptotic normality of estimated parameters $\widehat{\bm{\beta}}^{11}_t$, $\widehat{\bm{\beta}}^{12}_t$, $\widehat{\bm{\beta}}^{21}_t$, $\widehat{\bm{\beta}}^{22}_t$ under $\epsilon$-Greedy exploration}. Each row corresponds to a different evaluation metric, aligned with the left, middle, and right panels of Figure~\ref{fig:multi_est}.
    \label{fig:con_est}
\end{figure}

\begin{figure}[!t]
    \centering
    \includegraphics[width=0.7\linewidth]{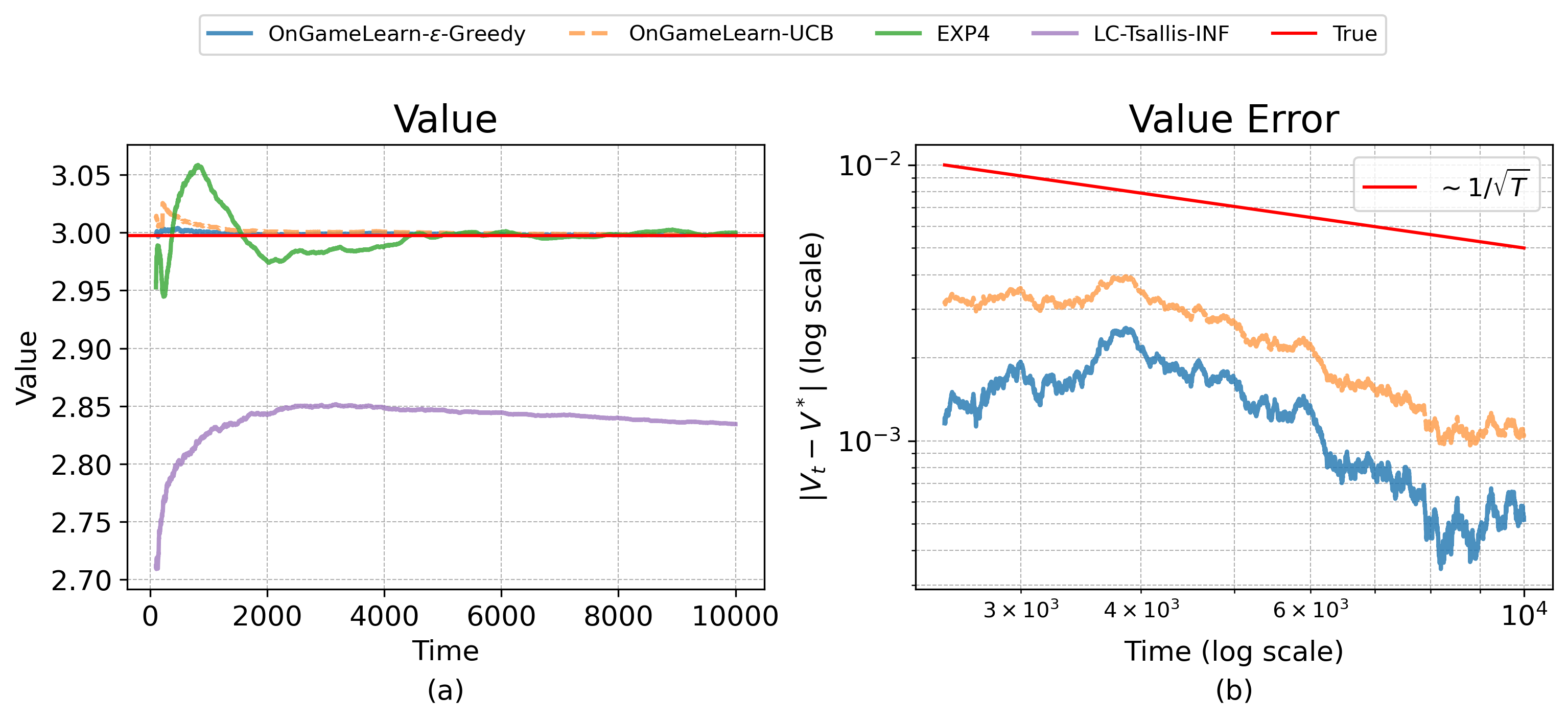}
    \caption{$\sqrt{T}$-consistency of the doubly robust value estimator $\widehat{V}_T^\text{conDR}$.}
    \label{fig:con_value}
\end{figure}

\noindent\textbf{Results.} Figures \ref{fig:con_NE}, \ref{fig:con_est}, and \ref{fig:con_value} present results analogous to those in Figures \ref{fig:multi_NE} and \ref{fig:multi_est}, demonstrating the effectiveness of our theoretical guarantees in the contextual setting. Across the three evaluation metrics, the parameter estimates exhibit stable and accurate performance. Moreover, the mixed strategies estimated by OnGameLearn progressively approach the true Nash equilibrium, as reflected by the decreasing distances over time. In contrast, as shown in Figure~\ref{fig:con_NE}, EXP4 and LC-Tsallis-INF are regret-based contextual bandit methods that generally do not provide last-iterate convergence to the Nash equilibrium. They also do not support statistical inference for the payoff parameters. Statistical inference figures under UCB exploration are provided in Section~\ref{app:context_simu} of the supplementary material.
The optimal policy value $V^\ast$ is approximated via Monte Carlo simulations. As shown in Figure~\ref{fig:con_value}, the doubly robust value estimator $\widehat{V}_T^{\mathrm{conDR}}$ continues to exhibit the expected $\sqrt{T}$-convergence rate established in Theorem~\ref{conthm:drrate}, with a slope closely matching the nominal benchmark. In contrast, LC-Tsallis-INF exhibits a visible gap between $V^\ast$ and $\widehat{V}_T^{\mathrm{conDR}}$, which may also be attributed to the unbounded Gaussian observation noise in our setting.
In this section, we report results for $m=k=2$. Results for larger action spaces are provided in Section~\ref{app:mk} of the supplementary material.

\section{Real Data Application}
\label{sec:real}

In this section, we evaluate OnGameLearn in a real-world hotel pricing application. The results show that OnGameLearn provides valid statistical inference and $\sqrt{T}$-rate convergence of the doubly robust value estimator under bounded equilibrium probabilities, while also yielding practical insights into competitive hotel pricing under Nash equilibrium strategies.

\noindent\textbf{Hotel pricing games.}
We use publicly available data from a large urban hotel chain \citep{bodea2009data}, containing room purchases and revenues for check-in dates from March 12 to April 15, 2007. We focus on ``Hotel 2'' and ``Hotel 4'' as two competing roadside hotels with similar room capacities, treating ``Hotel 2'' as the row player and ``Hotel 4'' as the column player.
For each hotel, the action is binary: $i_t=1$ or $j_t=1$ indicates a nightly rate above the average rate for the corresponding room type, while $i_t=0$ or $j_t=0$ indicates a lower rate. We use the lowest observed nightly rate as a proxy for the room cost. If a customer chooses ``Hotel 4'', the payoff is the resulting profit of ``Hotel 4''; if the customer chooses ``Hotel 2'', the payoff is the negative of the corresponding profit. Profit is calculated as the difference between the nightly rate and the estimated cost, multiplied by the number of rooms and length of stay.

\begin{figure}[!t]
    \centering
    \includegraphics[width=1.0\linewidth]{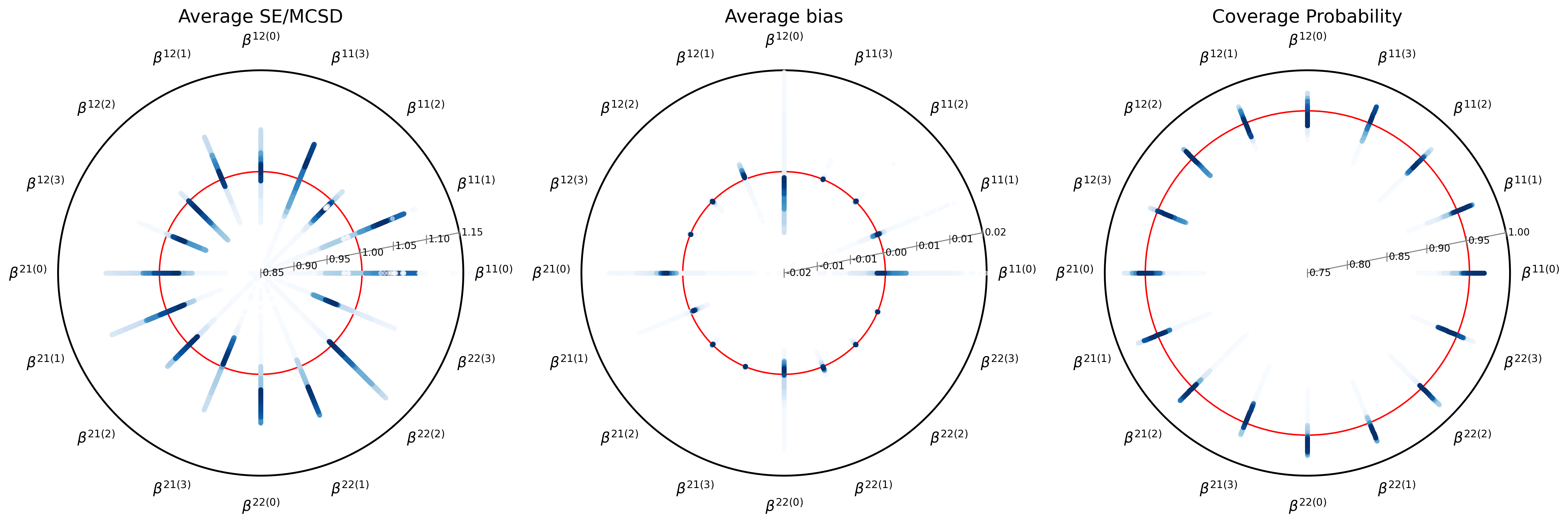}
    \caption{Consistency and asymptotic normality of the estimated parameters for the hotel dataset. \textbf{Each ray represents a parameter, with dot colors transitioning from light to dark as the time steps increase.} 
    \textbf{Left panel:} Average SE/MCSD, with the red circle indicating the nominal level of 1. \textbf{Middle panel:} Average bias, with the red circle indicating the nominal level of 0. \textbf{Right panel:} Coverage probability, with the red circle indicating the nominal level of 95\%.}
    \label{fig:real_infer}
\end{figure}

We incorporate three contextual features (with $d = 4$, including an intercept): advance purchase days, length of stay, and party size, as they capture key characteristics of the customers. To construct the offline dataset, we retain only transactions in which a product was actually purchased by the customer. To improve scaling, we apply a logarithmic transformation to the advance purchase days. We then treat the model fitted on the contextual features as the ground truth for the payoff. Using Algorithm~\ref{BOCG} with $\epsilon$-Greedy along with Theorems~\ref{conthm:para} and~\ref{conthm:drrate}, we evaluate the performance of the parameter estimates and the doubly robust value estimates from 500 replications, as shown in Figures~\ref{fig:real_infer} and \ref{fig:real_value}.

\noindent\textbf{Results.} Figure~\ref{fig:real_infer} reports three inference metrics: the average SE-to-MCSD ratio, average bias, and coverage probability. Dot colors become darker as the time steps increase, while the red circles indicate the nominal levels. Across all three metrics, OnGameLearn steadily approaches the corresponding nominal values.
In practice, Assumption~\ref{conass:margin} is difficult to verify directly. We therefore impose a small fixed lower bound on the estimated Nash equilibrium probabilities to stabilize the inverse probability weights. As shown in Figure~\ref{fig:real_value}, the doubly robust value estimator retains a $T^{-1/2}$ convergence rate, demonstrating the robustness of OnGameLearn.

\begin{figure}[!t]
    \centering
    \includegraphics[width=0.7\linewidth]{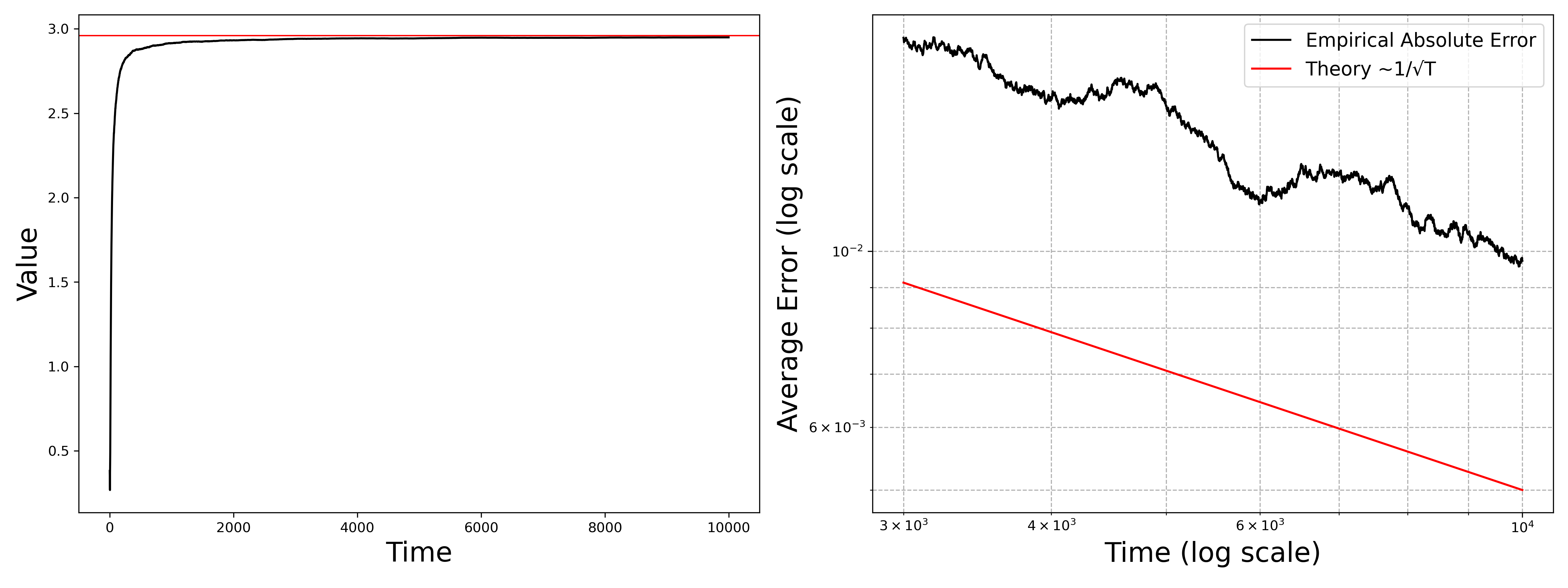}
    \caption{$\sqrt{T}$-consistency of the doubly robust value estimates for hotel dataset. The y-axis values of the left panel represent profit (scaled by a factor of 10).}
    \label{fig:real_value}
\end{figure}

The expected value of the zero-sum game between the two hotels is approximately \$29.4. This indicates that, on average, under the Nash equilibrium for both hotels, Hotel~2 is expected to incur a loss of about \$29.4 in profit per transaction, while Hotel~4 is expected to gain the same amount. This suggests that Hotel~4 is more likely to be chosen by customers than Hotel~2 when both adopt the Nash equilibrium strategies.

\section{Conclusion}
\label{sec:conclusion}
In this paper, we propose a policy learning and evaluation framework for online contextual matrix games. 
We introduce OnGameLearn in both multi-armed setting and the novel contextual matrix games setting with theoretical guarantees, including statistical inference for parameter estimation and a $\sqrt{T}$-consistent doubly robust value estimator. Numerical experiments and a real-world application demonstrate the effectiveness of our approach.

There are several promising directions for future work. First, building on the $\sqrt{T}$-consistency of our estimator, we plan to apply bootstrap methods \citep[see e.g.,][]{hong2020numerical,li2025proximal} to construct two-sided Wald-type confidence intervals for the optimal value, thereby completing the inference for optimal value evaluation. Second, we aim to extend the current two-player setting to more general multi-agent scenarios. In such settings, the optimal strategies are still characterized by the Nash equilibrium, where no agent can benefit from unilaterally deviating given others’ strategies. We will explore the corresponding statistical inference techniques in this multi-agent context.

\bibliography{citations}
\bibliographystyle{agsm}

\newpage

\appendix
\newtheorem{lemma}{Lemma}   
 
\counterwithin{figure}{section}
\counterwithin{table}{section}
\counterwithin{equation}{section}

This supplementary material provides additional experimental results and detailed proofs for the theoretical results of OnGameLearn in both online matrix games and online contextual matrix games. Section~\ref{app:bandit_alg} presents the specific behavior policies under different exploration strategies. Section~\ref{app:dis} provides further discussion on the convergence rate of the estimated Nash equilibrium, the computational complexity of the regularized Nash equilibrium, and robustness to model misspecification, while Section~\ref{app:exp} reports additional experimental results. Sections~\ref{app:proof_bandit} and~\ref{app:proof_context} contain the detailed statements and proofs of the theoretical results in Sections~\ref{sec:theory_bandit} and~\ref{sec:theory_con}, respectively. Section~\ref{app:sec:mis_model} provides additional analysis of contextual matrix games under model misspecification. Finally, Section~\ref{app:lemmas} collects the auxiliary results used in the proofs.



\section{Behavior Policy}
\label{app:bandit_alg}
In this section, we first introduce the estimated Nash equilibrium and behavior policy for online matrix games, and then extend these constructions to contextual matrix games.

Let $(\widehat{\bm{x}}_t,\widehat{\bm{y}}_t)$ denote the unique Nash equilibrium of $\mathcal{L}_t(\bm{x},\bm{y})$ and the saddle value of $\mathcal{L}_t$ as $\mathcal{L}_t(\widehat{\bm{x}}_t,\widehat{\bm{y}}_t)$ \citep{milgrom2002envelope}. 
We define $X^\ast:\mathbb{R}^{m \times k} \times \mathbb{R} \rightarrow \Delta_m$ and $Y^\ast:\mathbb{R}^{m \times k} \times \mathbb{R} \rightarrow \Delta_k$ as the set-valued functions for the Nash equilibrium of the two players with parameters $A$ and $\eta$, respectively. Thus,
\begin{equation}
\begin{aligned}
    \widehat{\bm{x}}_t
     =
    X^\ast(\widehat{A}_{t-1},\eta_t)
    =
    \arg\min_{\bm{x} \in \Delta_m}
    \max_{\bm{y} \in \Delta_k}
    \mathcal{L}_t(\bm{x},\bm{y}), \quad
    \widehat{\bm{y}}_t
    =
    Y^\ast(\widehat{A}_{t-1},\eta_t)
    =
    \arg\max_{\bm{y} \in \Delta_k}
    \min_{\bm{x} \in \Delta_m}
    \mathcal{L}_t(\bm{x},\bm{y}).
\end{aligned}
\end{equation}
To avoid getting trapped in a specific action pair during online learning, we incorporate the three commonly used bandit algorithms to explore other action pairs. The behavior policy is defined as ${\pi_t(i,j)}= \mathbb{E}[\mathbb{I}(i_t=i,j_t=j) \vert \mathcal{H}_{t-1}]$, which represents the conditional probability of selecting action pair $(i,j)$ at time $t$, given the history $\mathcal{H}_{t-1}$. \\
\textbf{Upper Confidence Bound (UCB) \citep{li2011unbiased}:} 
Let $\mathcal{N}_{t-1}^{ij}$ denote the number of times that the action pair 
$(i,j)$ has been observed prior to time $t$. Based on the history 
$\mathcal{H}_{t-1}$, we define the estimated standard deviation associated 
with the payoff entry $(i,j)$ as $\widehat{\sigma}_{t}^{ij}=1/\sqrt{\max\left\{1,\mathcal{N}_{t-1}^{ij}\right\}}$.
The corresponding optimistic estimate of the payoff matrix is given by
\[
[\widetilde{A}_{t}]_{ij} = [\widehat{A}_{t-1}]_{ij} + c_t\widehat{\sigma}_{t}^{ij},
\]
where $c_t>0$ is a tuning parameter that controls the level of exploration. We then compute the Nash equilibrium strategies based on the  optimistic payoff matrix: $X^\ast(\widetilde{A}_{t},\eta_t)$ and $Y^\ast(\widetilde{A}_{t},\eta_t)$.
To ensure sufficient exploration of the available actions, we apply clipping to the mixed strategy of each player. For any dimension $d$ and minimum probability threshold $\epsilon_t$, define the truncated simplex as
\[
    \Delta_d(\epsilon_t) =
    \left\{
        \bm{u}\in\mathbb{R}^d:
        \sum_{\ell=1}^{d} u_\ell = 1,
        \quad
        u_\ell \geq \epsilon_t,
        \quad
        \forall \ell\in[d]
    \right\}.
\]
Accordingly, the clipped strategies of the two players are
\[
    \widetilde{\bm{x}}_t =
    \operatorname{Clip}_m
    \left(
        X^\ast(\widetilde{A}_{t},\eta_t),
        \frac{\epsilon_t}{m}
    \right), \quad 
    \widetilde{\bm{y}}_t =
    \operatorname{Clip}_k
    \left(
        Y^\ast(\widetilde{A}_{t},\eta_t),
        \frac{\epsilon_t}{k}
    \right),
\] where the clipping operator is defined as the Euclidean projection onto the 
truncated simplex 
$\operatorname{Clip}_d(\bm{u},\epsilon_t) =
    \underset{\bm{u}_0\in\Delta_d(\epsilon_t)}
    {\arg\min}
    \left\|
        \bm{u}_0 -\bm{u}
    \right\|_2.$ \\
\textbf{Thompson Sampling (TS) \citep{agrawal2013thompson}.}
Following \citet{shen2024doubly}, we assume that the observed payoff associated with action pair $(i,j)$ follows a Gaussian distribution $\mathcal{N}(\widetilde{A}_{ij},\rho^2)$, where $\rho^2$ is a known variance parameter. Based on the history $\mathcal{H}_{t-1}$, the posterior distribution of the unknown payoff entry $A_{ij}$ at time $t$ is given by
\begin{equation}
  \widetilde{A}_{ij} \mid \mathcal{H}_{t-1}
    \sim
    \mathcal{N}
    \left(
        [\widehat{A}_{t-1}]_{ij},
        \frac{\rho^2}{
            \max\left\{1,\mathcal{N}_{t-1}^{ij}\right\}
        }
    \right).
\end{equation}
At each time $t$, we draw posterior samples as $[\widetilde{A}_t]_{ij}$, and then compute the Nash equilibrium strategies based on the sampled payoff matrix $\widetilde{A}_t$: $X^\ast(\widetilde{A}_t,\eta_t)$, and $Y^\ast(\widetilde{A}_t,\eta_t)$.
Similarly as UCB, the clipped strategies are $\widetilde{\bm{x}}_t=\operatorname{Clip}_m \left(X^\ast(\widetilde{A}_{t},\eta_t), \epsilon_t/m \right)$, and $\widetilde{\bm{y}}_t=\operatorname{Clip}_k \left( Y^\ast(\widetilde{A}_{t},\eta_t), \epsilon_t/k \right)$. \\
\textbf{$\epsilon$-Greedy \citep{sutton1998reinforcement}:}
At round $t$, each player independently explores with probability $\epsilon_t$ and follows the estimated equilibrium strategy with probability $1-\epsilon_t$. Specifically, the actions $i_t$ and $j_t$ are sampled from the mixed strategies 
\begin{equation*} 
\widetilde{\bm{x}}_t = (1-\epsilon_t)\widehat{\bm{x}}_t + \frac{\epsilon_t}{m}\bm{1}_m, \quad 
\widetilde{\bm{y}}_t = (1-\epsilon_t)\widehat{\bm{y}}_t + \frac{\epsilon_t}{k}\bm{1}_k.  
\end{equation*}
For the three exploration algorithms, the corresponding behavior policy is $\pi_t(i,j) = \widetilde{x}_{ti}\widetilde{y}_{tj} \ge \epsilon_t^2/mk, \, i\in[m],\, j\in[k].$

For contextual matrix games, we extend the exploration schemes above to the context-dependent setting, ensuring that each action pair is selected with probability at least $\epsilon_t^2/(mk)$. For UCB, the optimistic payoff estimate is
\begin{equation*}
[\widetilde M_{t-1}(\bm z_t)]_{ij}
= [\widehat M_{t-1}(\bm z_t)]_{ij} + c_t \sqrt{\bm z_t^\top \left(\sum_{s=1}^{t-1}
\mathbb I(i_s=i,j_s=j)\bm z_s\bm z_s^\top\right)^{-1} \bm z_t}.
\end{equation*}
The behavior policy is constructed with $\widetilde{\phi}_t(\bm z_t)=\operatorname{Clip}_m \left(X^\ast(\widetilde{M}_{t-1}(\bm{z}_t),\eta_t), \epsilon_t/m \right)$ and $\widetilde{\nu}_t(\bm z_t)=\operatorname{Clip}_k \left(Y^\ast(\widetilde{M}_{t-1}(\bm{z}_t),\eta_t), \epsilon_t/k \right)$. For TS, we draw samples from the distribution
\[
\widetilde{\bm\beta}^{ij}
\mid \mathcal H_{t-1}
\sim
\mathcal N\left(
\widehat{\bm\beta}_{t-1}^{ij},
\rho^2 \left(\sum_{s=1}^{t-1} \mathbb I(i_s=i,j_s=j)\bm z_s\bm z_s^\top \right)^{-1}
\right),
\]
as $\widetilde{\bm\beta}_{t}^{ij}$. The clipped strategies are then given by $\widetilde{\phi}_t(\bm z_t)=\operatorname{Clip}_m \left(X^\ast(\widetilde{M}_{t-1}(\bm{z}_t),\eta_t), \epsilon_t/m \right)$ and $\widetilde{\nu}_t(\bm z_t)=\operatorname{Clip}_k \left(Y^\ast(\widetilde{M}_{t-1}(\bm{z}_t),\eta_t), \epsilon_t/k \right)$ with $[\widetilde{M}_{t-1}(\bm{z}_t)]_{ij}=\bm{z}_t^\top \widetilde{\bm{\beta}}^{ij}_{t}$. For $\epsilon-$Greedy, similarly, we set $\widetilde{\phi}_t(\bm z_t) = (1-\epsilon_t)\widehat{\phi}_t(\bm z_t) + \frac{\epsilon_t}{m}\bm{1}_m$, and $\widetilde{\nu}_t(\bm z_t)= (1-\epsilon_t)\widehat{\nu}_t(\bm z_t) + \frac{\epsilon_t}{k}\bm{1}_k$.
The resulting behavior policy for the two players is therefore denoted as ${\pi}_t^{\mathrm{con}}(i,j \vert \bm{z}_t)=[\widetilde{\phi}_t(\bm z_t)]_i [\widetilde{\nu}_t(\bm z_t)]_j$.

\section{Discussion}
\label{app:dis}
In this section, we first establish the convergence rate of the estimated Nash equilibrium. We then discuss the case where the true model for the contextual matrix game is not linear.

\subsection{Convergence Rate of Estimated Nash Equilibrium}
\label{app:con_rate_ne}
Based on Theorems~\ref{thm:nash_con} and~\ref{conthm:nash_con}, we have established the consistency of the estimated Nash equilibrium. We now analyze its explicit convergence rate. The detailed proofs are provided in Section~\ref{app:thm:rate_nash_rate_proof} of the supplementary material.

\begin{theorem}[\textbf{\emph{Convergence rate of the estimated Nash equilibrium}}]
\label{thm:rate_nash_rate}
$(\widehat{\bm x}_t,\widehat{\bm y}_t)$ denotes the regularized Nash equilibrium computed from the estimated payoff matrix $\widehat A_{t-1}$ with regularization term $\eta_t$. Then,
\[
\operatorname{dist}
\left(
(\widehat{\bm x}_t,\widehat{\bm y}_t),
\mathcal E(A)
\right)
=
\widetilde O_p
\left(
\frac{\sqrt{mk}}{\epsilon_t\sqrt t}
+
\eta_t
\right).
\]
Moreover, under the additional realized error-bound condition in Assumption~\ref{conass:uniform_realized_error_bound}, the same argument extends to contextual matrix games. Specifically,
\[
\operatorname{dist}
\left(
(
\widehat\phi_t(\bm z_t),
\widehat\nu_t(\bm z_t)
),
\mathcal E(\bm M(\bm z_t))
\right)
=
\widetilde O_p
\left(
\frac{\sqrt{mk}}{\epsilon_t\sqrt t}
+
\eta_t
\right).
\]
\end{theorem}

\subsection{Computation of the Regularized Nash Equilibrium}
\label{app:computation}

For a given estimated payoff matrix $\widehat A_{t-1}$ in Eq.~\eqref{equ:NE_bandit}, or a given estimated contextual payoff matrix $\widehat M_{t-1}(\bm{z}_t)$ in Eq.~\eqref{equ:policy_con}, the optimizers are computed by solving the corresponding regularized saddle-point problem over the probability simplices. In our implementation, we use a gradient descent--ascent procedure.

For example, in the non-contextual setting, i.e., zero-sum online matrix games, at each inner iteration, the gradients with respect to $\bm{x}$ and $\bm{y}$ are computed as
\[
\nabla_{\bm{x}}
\left\{
\bm{x}^\top \widehat A_{t-1}\bm{y}
-\eta_t R_X(\bm{x})
+\eta_t R_Y(\bm{y})
\right\}
=
\widehat A_{t-1}\bm{y}
-\eta_t\nabla R_X(\bm{x}),
\]
and
\[
\nabla_{\bm{y}}
\left\{
\bm{x}^\top \widehat A_{t-1}\bm{y}
-\eta_t R_X(\bm{x})
+\eta_t R_Y(\bm{y})
\right\}
=
\widehat A_{t-1}^\top\bm{x}
+\eta_t\nabla R_Y(\bm{y}),
\]
followed by projection onto the probability simplices $\Delta_m$ and $\Delta_k$. The same procedure applies to the contextual setting after replacing $\widehat A_{t-1}$ with $\widehat M_{t-1}(\bm{z}_t)$.

For existence, the feasible sets $\Delta_m$ and $\Delta_k$ are compact and convex, and the regularized objective is continuous. Therefore, the optimization problems in Eq.~\eqref{equ:NE_bandit} and Eq.~\eqref{equ:policy_con} admit solutions for any fixed estimated payoff matrix. Moreover, the regularization terms are introduced to make the saddle-point problem well behaved and to ensure uniqueness of the regularized Nash equilibrium. Thus, for each observed context $\bm{z}_t$, the estimated strategies $\widehat\phi_t(\bm{z}_t)$ and $\widehat\nu_t(\bm{z}_t)$ are well-defined.

Regarding computational cost, in the non-contextual setting, each gradient descent--ascent iteration mainly involves matrix--vector multiplications with $\widehat A_{t-1}$, which cost $O(mk)$. Therefore, if $N$ inner iterations are used, the main computational cost for computing the optimizers in Eq.~\eqref{equ:NE_bandit} is $O(Nmk)$.

In the contextual setting, for each observed context $\bm{z}_t$, we first construct the estimated payoff matrix $\widehat M_{t-1}(\bm{z}_t)$. Since this matrix has $mk$ entries and each entry $\bm{z}_t^\top\widehat\beta_{t-1}^{ij}$ requires an inner product between two $d$-dimensional vectors, constructing $\widehat M_{t-1}(\bm{z}_t)$ costs $O(mkd)$. After that, solving the regularized minimax problem by gradient descent--ascent requires an additional $O(Nmk)$ cost. Thus, the main per-context computational cost for Eq.~\eqref{equ:policy_con} is $O(mkd + Nmk)$.

When a large number of decisions need to be made, the computational burden can be reduced by reducing the number of inner iterations $N$. In the non-contextual setting, the estimated payoff matrix typically changes gradually over time as new observations are collected. Therefore, the gradient descent--ascent algorithm can be warm-started from the optimizer obtained at the previous time point, rather than initialized from scratch. This can reduce the number of inner iterations required for convergence in practice.

\subsection{Extension to Model Misspecification}
\label{app:mis_model}
We now consider the contextual setting in which the working linear payoff model may be misspecified. Let $r_t=\varphi^{i_t,j_t}(\bm z_t)+e_t$ denote the true conditional mean payoff associated with action pair $(i,j)$, where $\{\varphi^{ij}(\bm z)\}_{i\in [m],j\in[k]}$ are unknown arbitrary measurable functions of $\bm z_t$, and $\bm z_t$ and $e_t$ satisfy the same conditions as before. For each action pair $(i,j)$, define the least-false parameter $\bm \beta_\dagger^{ij}=\arg\min_{\bm\beta\in\mathbb R^d}\mathbb E\left[\left\{\varphi^{ij}(\bm z_t)-\bm z_t^\top\bm\beta\right\}^2\right]$. Under model misspecification, follow \cite{chen2021statistical}, we use the inverse-probability-weighted least-squares estimator
\begin{equation*}
\widehat{\bm\beta}_{\dagger,t}^{ij}
=
\left\{
\frac{1}{t}
\sum_{s=1}^t
\frac{a_{\dagger,s}^{ij}}{
\pi_{\dagger,s}^{ij}}
\bm z_s\bm z_s^\top
\right\}^{-1}
\left\{
\frac{1}{t}
\sum_{s=1}^t
\frac{a_{\dagger,s}^{ij}}{\pi_{\dagger,s}^{ij}}
\bm z_sr_s
\right\},
\end{equation*}
where $a_{\dagger,s}^{ij}=\mathbb{I}(i_s=i,j_s=j)$ and $\pi_{\dagger,s}^{ij}=\mathbb{P}(i_s=i,j_s=j\vert \mathcal{H}_{s-1}^{\mathrm{con}}, \bm{z}_s)$.
Under the same exploration and regularity conditions as in the correctly specified case, the estimated payoff matrix converges to $M_\dagger(\bm z)$ with $[M_\dagger(\bm z)]_{ij}=\bm z^\top \bm \beta_\dagger^{ij}$. Consequently, the estimated contextual Nash policies converge to the Nash equilibrium policies of the pseudo-payoff game. The corresponding behavior policies also converge to the clipped or randomized versions of these limiting policies, depending on the exploration strategy. Moreover, the least-squares payoff estimator admits an asymptotic normal approximation. We also establish the consistency of the value estimator under model misspecification. 
Specifically, the value estimator converges to the value of the pseudo-true contextual game induced by Nash equilibrium of $M_\dagger(\bm z)$, rather than to the value of the game induced by the true conditional mean functions ${\varphi^{ij}(\bm z)}$.
We defer all formal statements and proofs to Section~\ref{app:sec:mis_model} of the supplementary material. The corresponding simulation results are provided in Section~\ref{app:mis_linear}.

\section{Additional Experiments Results}
\label{app:exp}

In this section, we report additional experimental results supplementing the simulations in Section~\ref{sec:simulation}. Specifically, Section~\ref{app:matrix_simu} presents results for Tsallis-INF under different hyperparameter choices, as well as statistical inference results under UCB exploration in the non-contextual online matrix game setting. Section~\ref{app:context_simu} reports additional results for the contextual matrix game setting under UCB exploration. Section~\ref{app:mis_linear} presents simulation results under model misspecification. Finally, Section~\ref{app:mk} provides results for settings with larger action spaces, where $m,k>2$.

\subsection{Additional Results for Online Matrix Games}
\label{app:matrix_simu}
For Tsallis-INF~\citep{zimmert2021tsallis}, we consider different learning-rate coefficients in the schedule $\eta_t=\eta_0/\sqrt{t+1}$. The original Tsallis-INF paper suggests $\eta_0=2.0$ under the standard bounded-loss setting. However, our simulation setting differs from the original theoretical setup because the observed payoff contains Gaussian noise and is therefore not uniformly bounded. Since importance-weighted bandit updates can be sensitive to large noisy observations, we evaluate several choices of $\eta_0$ under unbounded observation noise. Specifically, we use the default choice $\eta_0=2.0$ in Figure~\ref{fig:multi_NE_sens_2}, and additionally report results for $\eta_0=1.0$ in Figure~\ref{fig:multi_NE} and $\eta_0=4.0$ in Figure~\ref{fig:multi_NE_sens_4}.

\begin{figure}[ht]
    \centering
    \includegraphics[width=1.0\linewidth]{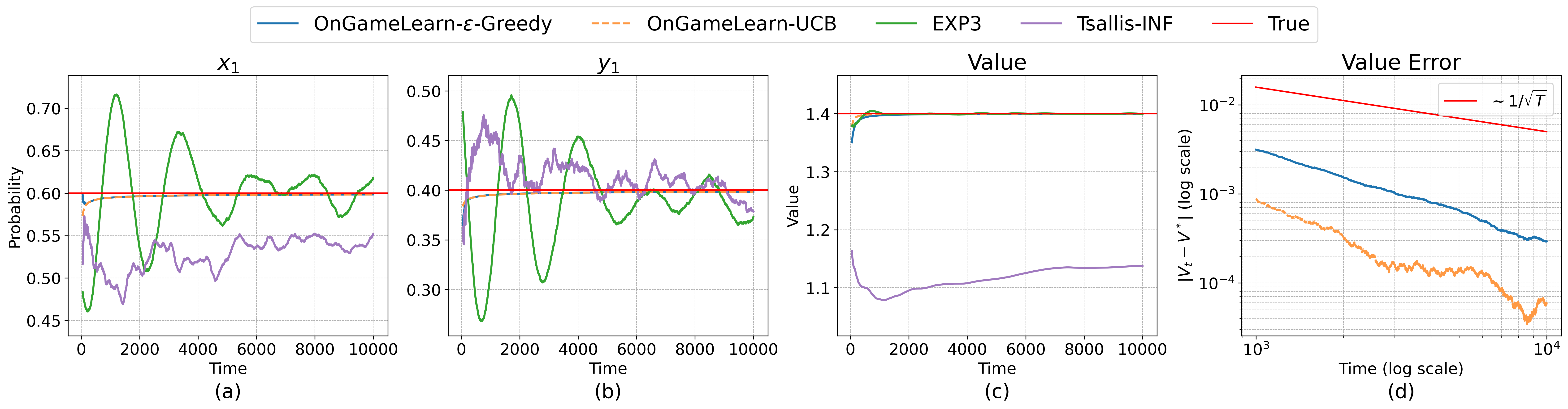}
    \caption{Convergence of the estimated Nash equilibrium and $\sqrt{T}$-consistency of the doubly robust value estimator $\widehat{V}_T^{\mathrm{DR}}$ when we choose $\eta_t = 2.0/\sqrt{t+1}$ for Tsallis-INF.}
    \label{fig:multi_NE_sens_2}
\end{figure}

\begin{figure}[ht]
    \centering
    \includegraphics[width=1.0\linewidth]{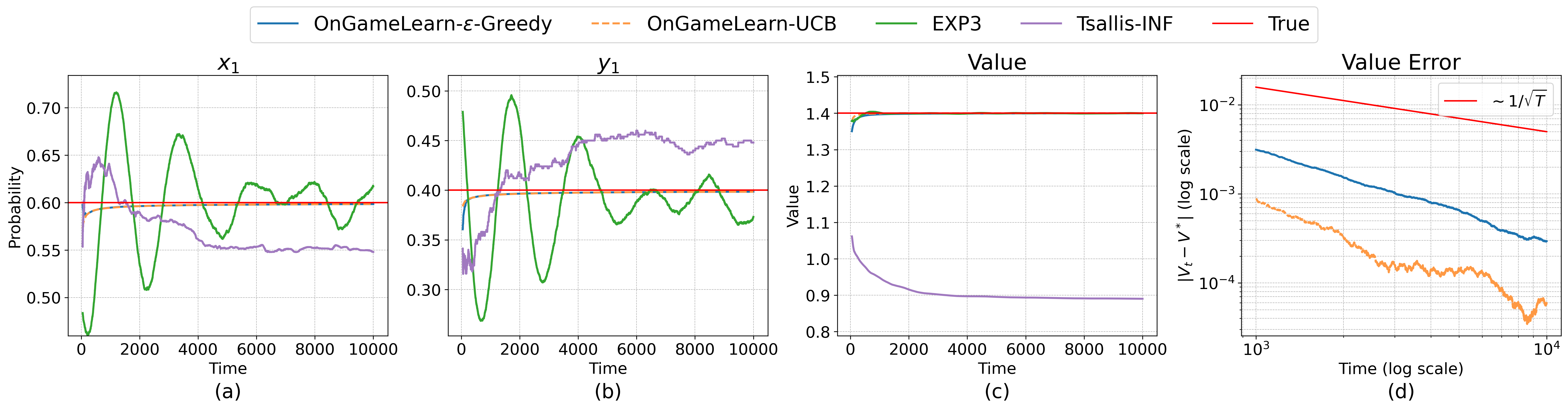}
    \caption{Convergence of the estimated Nash equilibrium and $\sqrt{T}$-consistency of the doubly robust value estimator $\widehat{V}_T^{\mathrm{DR}}$ when we choose $\eta_t = 4.0/\sqrt{t+1}$ for Tsallis-INF.}
    \label{fig:multi_NE_sens_4}
\end{figure}

In addition to the $\epsilon$-greedy results reported in the main text, we also provide statistical inference results for OnGameLearn under the UCB exploration strategy in Figure~\ref{fig:multi_est_ucb}. The results show that OnGameLearn continues to perform well under UCB across all three evaluation metrics, indicating that the proposed inference procedure is not limited to the $\epsilon$-greedy exploration strategy. These empirical findings are consistent with the theoretical guarantee in Theorem~\ref{thm:norm_matrix} in Section~\ref{sec:theory_bandit}.

\begin{figure}[ht]
    \centering
    \includegraphics[width=1.0\linewidth]{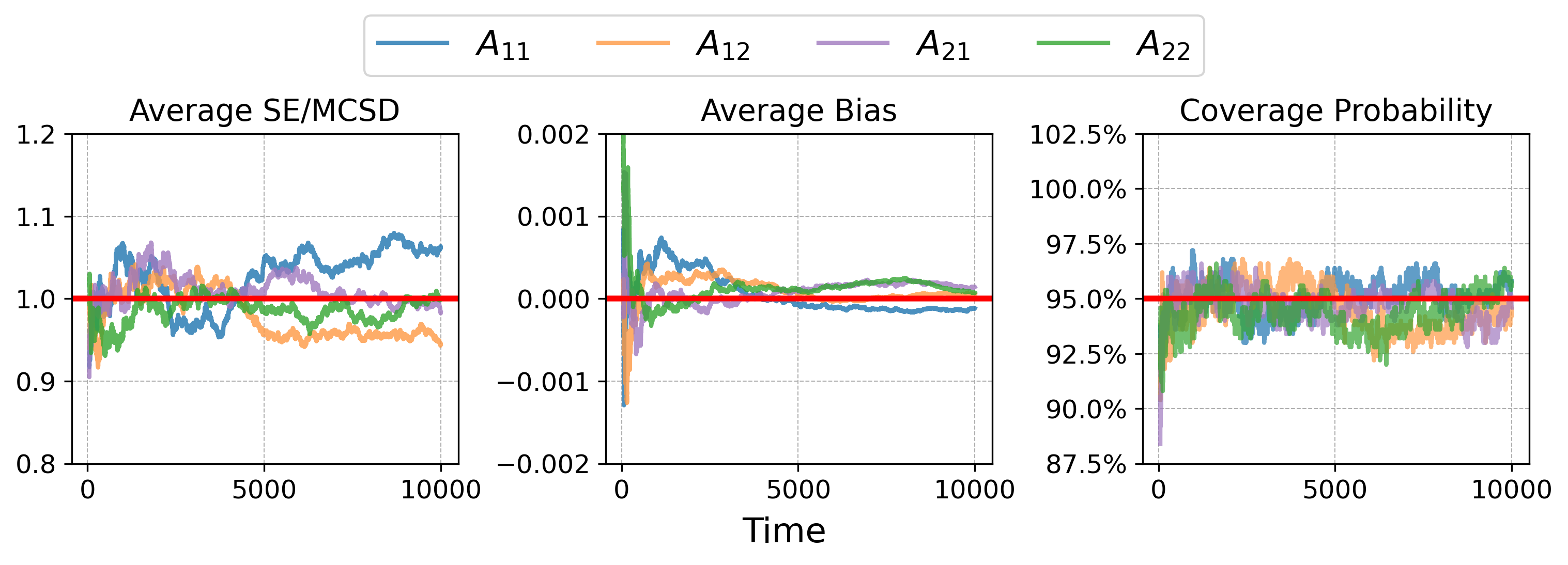}
    \caption{Consistency and asymptotic normality of payoff estimator $\widehat{A}_t$ under UCB exploration.}
    \label{fig:multi_est_ucb}
\end{figure}

\subsection{Additional Results for Contextual Matrix Games}
\label{app:context_simu}

Similarly, Figure~\ref{fig:con_est_ucb} presents the statistical inference results for OnGameLearn under the UCB exploration strategy in the contextual matrix game setting. The results show that OnGameLearn continues to perform well under UCB across all evaluation metrics, further supporting that the proposed inference procedure applies beyond the original $\epsilon$-greedy exploration strategy. These findings are consistent with Theorem~\ref{conthm:para} in Section~\ref{sec:theory_con}.

\begin{figure}[ht]
    \centering
    \includegraphics[width=1.0\linewidth]{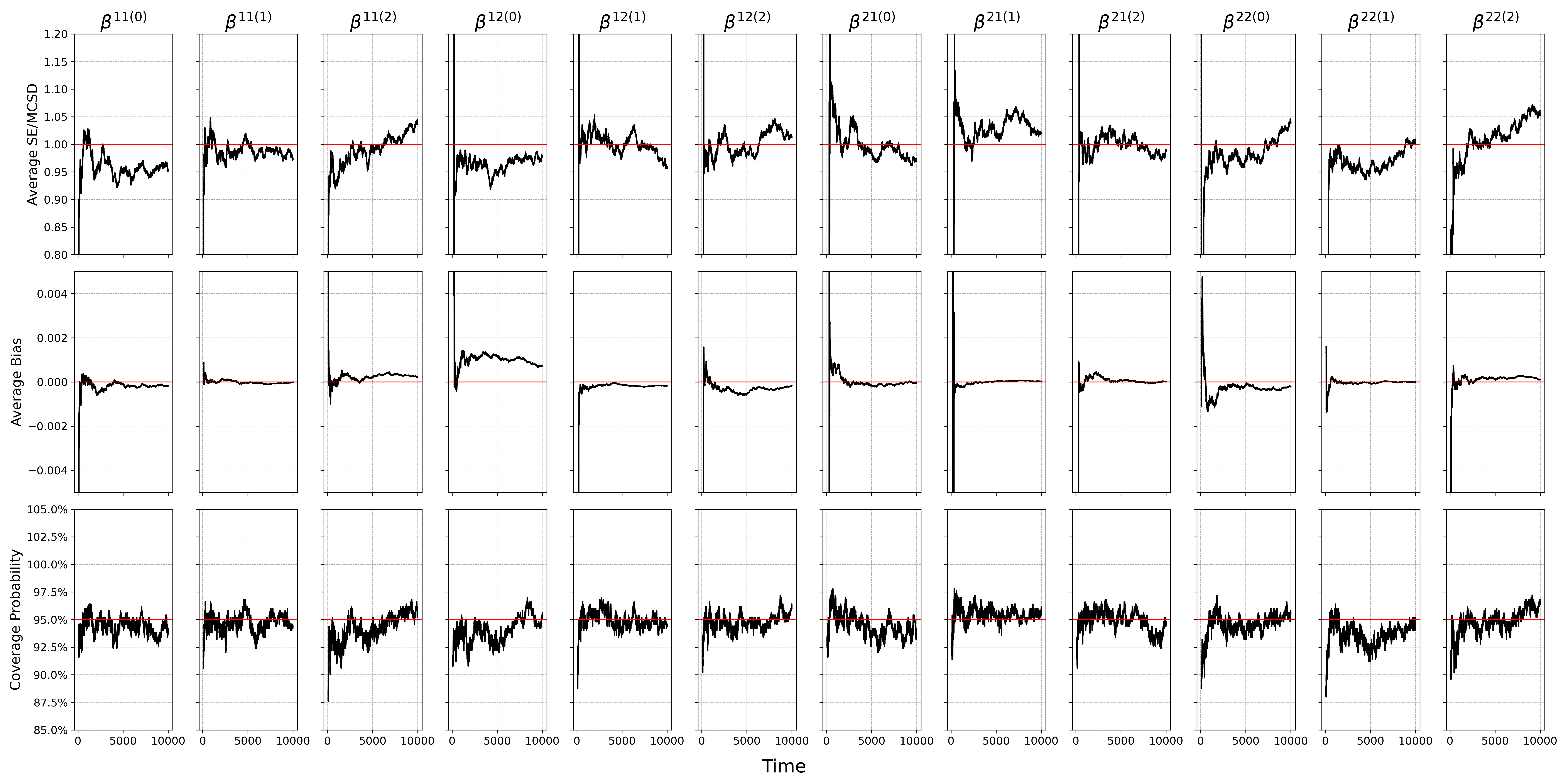}
    \caption{Consistency and asymptotic normality of estimated parameters $\widehat{\bm{\beta}}^{11}_t$, $\widehat{\bm{\beta}}^{12}_t$, $\widehat{\bm{\beta}}^{21}_t$, $\widehat{\bm{\beta}}^{22}_t$ under UCB exploration.}
    \label{fig:con_est_ucb}
\end{figure}

\subsection{Simulation Results under Model Misspecification}
\label{app:mis_linear}

The true model for the misspecified case in Section~\ref{app:mis_model} is set as logistic model
$$\varphi^{ij}(\bm z)=\frac{\exp\{\bm{z}^{\top}\bm{\beta}^{ij}\}}{1+\exp\{\bm{z}^{\top}\bm{\beta}^{ij}\}},$$ 
where $\bm{\beta}^{11} = (1,2,2)^\top$, $\bm{\beta}^{12} = (-2,-1,1)^\top$, $\bm{\beta}^{21} = (-1,1,-2)^\top$, $\bm{\beta}^{22} = (3,1,2)^\top$, as in the simulation setting of Section~\ref{sec:simu_context}. The generation of $\bm z_t$ and the noise terms $e_t$ also follows the setting in Section~\ref{sec:simu_context}. Note that we set $\epsilon_t=0.01$ as a constant exploration probability to ensure that the positivity condition in Theorem~\ref{conthm:wls_normality_misspecified} of Section~\ref{app:sec:mis_model} holds. Specifically, this guarantees that each action pair is selected with probability at least $c_\pi>0$.
For the remaining hyperparameters, we use the same choices as in Section~\ref{sec:simu_context}. Each experiment is repeated 500 times under $\epsilon$-greedy exploration with termination time $T=50{,}000$. The least-false parameter $\bm \beta_\dagger^{ij}$ and the true value under model misspecification, $V^\ast_\dagger$, are computed via Monte Carlo approximation.

Figure~\ref{fig:mis_con_NE} presents the convergence of the estimated Nash equilibrium and the consistency of the doubly robust value estimator $\widehat V_{\dagger,T}^{\mathrm{conDR}}$ under model misspecification. The results show that the estimated Nash equilibrium converges to its target and that $\widehat V_{\dagger,T}^{\mathrm{conDR}}$ accurately estimates the least-false game value, which is consistent with Theorems \ref{conthm:behavior_misspecified} and~\ref{conthm:value_rootT_misspecified}.

\begin{figure}[ht]
    \centering
    \includegraphics[width=1.0\linewidth]{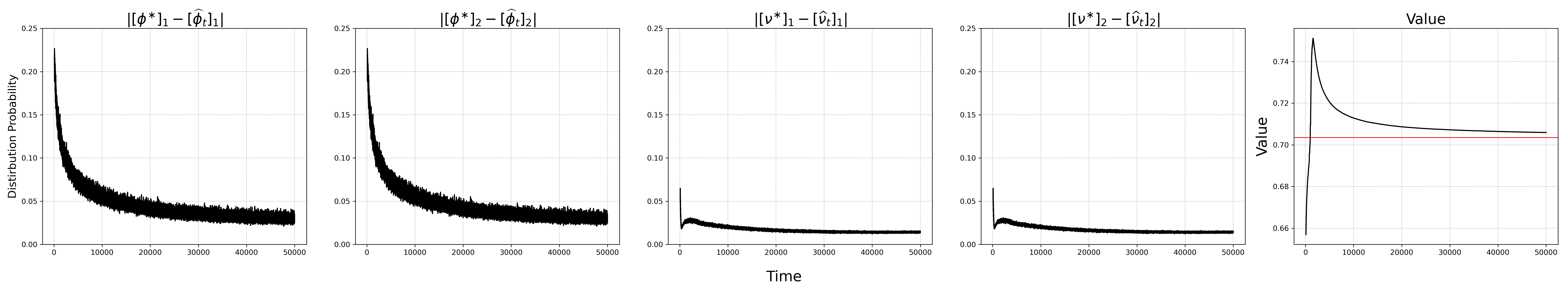}
    \caption{Convergence of the estimated Nash equilibrium and consistency of the doubly robust value estimator $\widehat V_{\dagger,T}^{\mathrm{conDR}}$ under model misspecification.}
    \label{fig:mis_con_NE}
\end{figure}

Figure~\ref{fig:mis_con_est} provides empirical evidence for the asymptotic normality of the WLS estimator for the least-false parameters, as established in Theorem~\ref{conthm:wls_normality_misspecified}.

\begin{figure}[ht]
    \centering
    \includegraphics[width=1.0\linewidth]{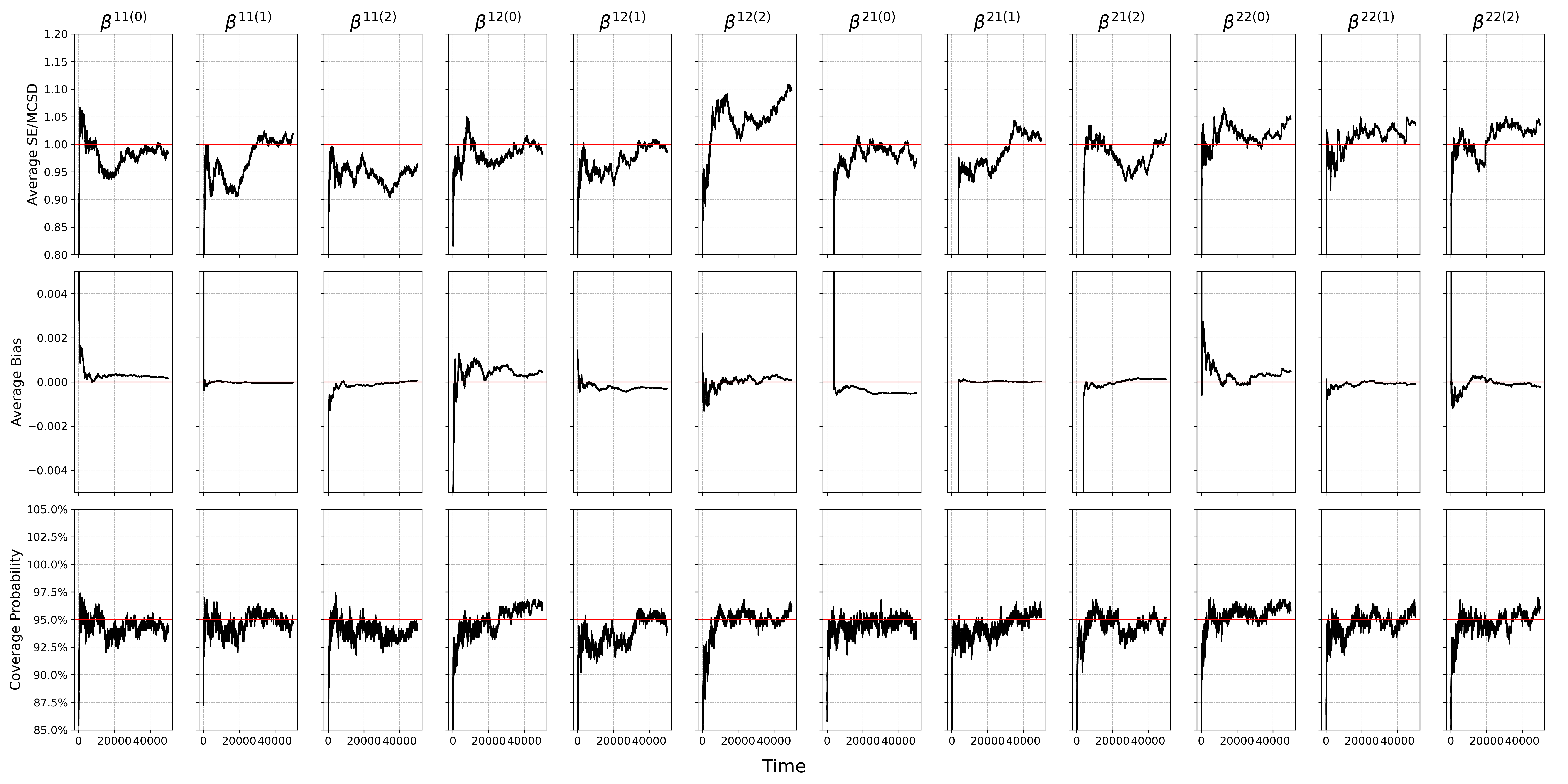}
    \caption{Consistency and asymptotic normality of estimated parameters $\widehat{\bm{\beta}}^{11}_{\dagger,t}$, $\widehat{\bm{\beta}}^{12}_{\dagger,t}$, $\widehat{\bm{\beta}}^{21}_{\dagger,t}$, $\widehat{\bm{\beta}}^{22}_{\dagger,t}$ under $\epsilon$-Greedy exploration and model misspecification.}
    \label{fig:mis_con_est}
\end{figure}

\subsection{Simulation Results for Larger Action Spaces}
\label{app:mk}

\subsubsection{Online Matrix Games}

To further assess the performance of OnGameLearn in larger action spaces, we conduct additional experiments with $m=k=4$. We consider a $4\times4$ payoff matrix $A =((2.00,  1.68,  2.02,  2.53), (1.28,  1.60,  2.96,  1.79), (1.67,  3.00,  1.54,  1.18), (3.53,  1.70,  1.43,  2.23))$.
The Nash equilibrium of matrix $A$ is $\bm{x}^\ast = (0.40, 0.20, 0.25, 0.15)^\top$ and 
$\bm{y}^\ast = (0.15, 0.35, 0.30, 0.20)^\top$. Then the value of matrix game A can be derived as $V_A^\ast=2.0$.
We use the same hyperparameter choices as in Section~\ref{sec:simu_multiarm}: $\epsilon_t = 0.1t^{-1/4}$, $\eta_t = t^{-1/2}$, and $c_t=1$. For Tsallis-INF, we set $\eta_t = 1.0/\sqrt{t+1}$. Each experiment is repeated 500 times.

\begin{figure}[ht]
    \centering
    \includegraphics[width=1.0\linewidth]{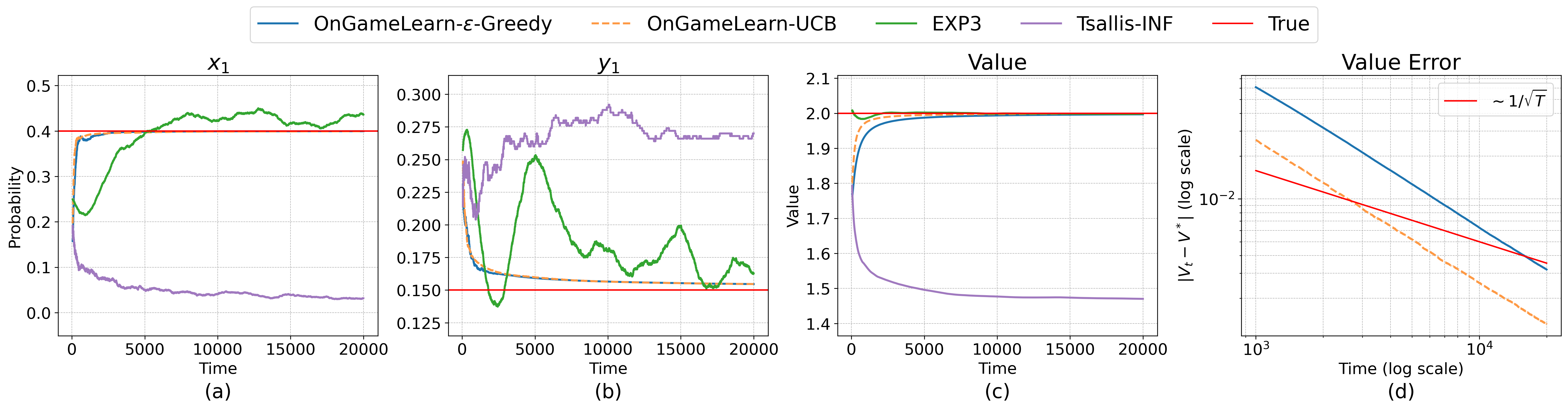}
    \caption{Convergence of the estimated Nash equilibrium and $\sqrt{T}$-consistency of the doubly robust value estimator $\widehat{V}_T^{\mathrm{DR}}$ when $m=k=4$.}
    \label{fig:est_mk}
\end{figure}

\begin{figure}[ht]
    \centering
    \includegraphics[width=1.0\linewidth]{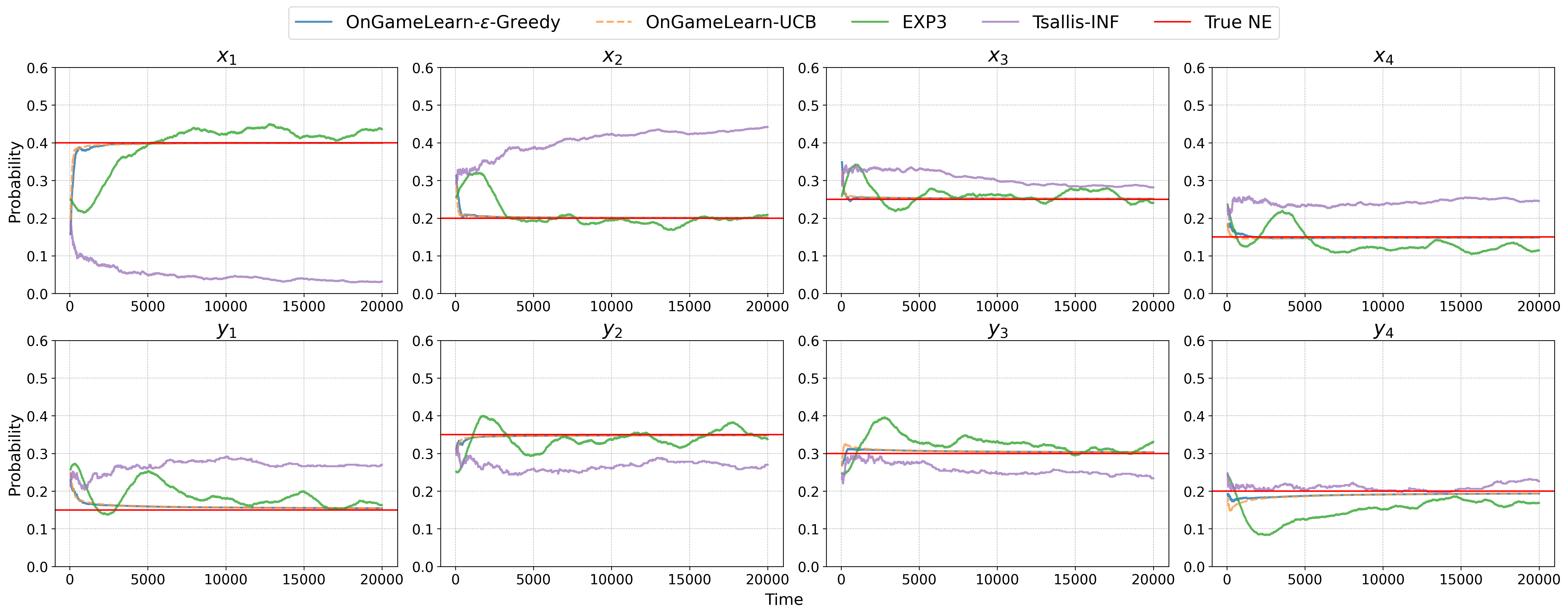}
    \caption{Convergence of all the coordinates of the estimated Nash equilibrium when $m=k=4$. The curves for OnGameLearn under $\epsilon$-greedy and UCB exploration almost completely overlap.}
    \label{fig:est_NEpoint_mk}
\end{figure}

\begin{figure}[ht]
    \centering
    \includegraphics[width=1.0\linewidth]{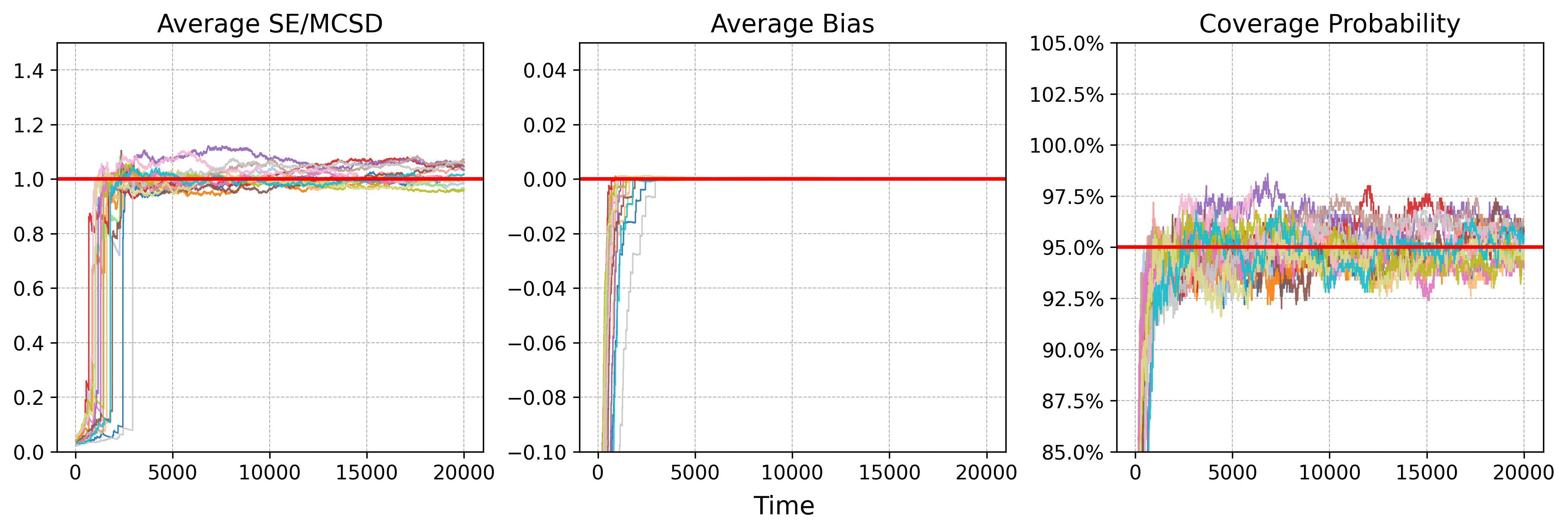}
    \caption{Consistency and asymptotic normality of the payoff matrix estimator $\widehat{A}_t$ under $\epsilon$-Greedy exploration with $m=k=4$. Each color corresponds to one element of the $4\times4$ matrix $\widehat{A}_t$. The 16 curves approach the nominal reference line across all three evaluation metrics.}
    \label{fig:est_inf_eps_mk}
\end{figure}

\begin{figure}[ht]
    \centering
    \includegraphics[width=1.0\linewidth]{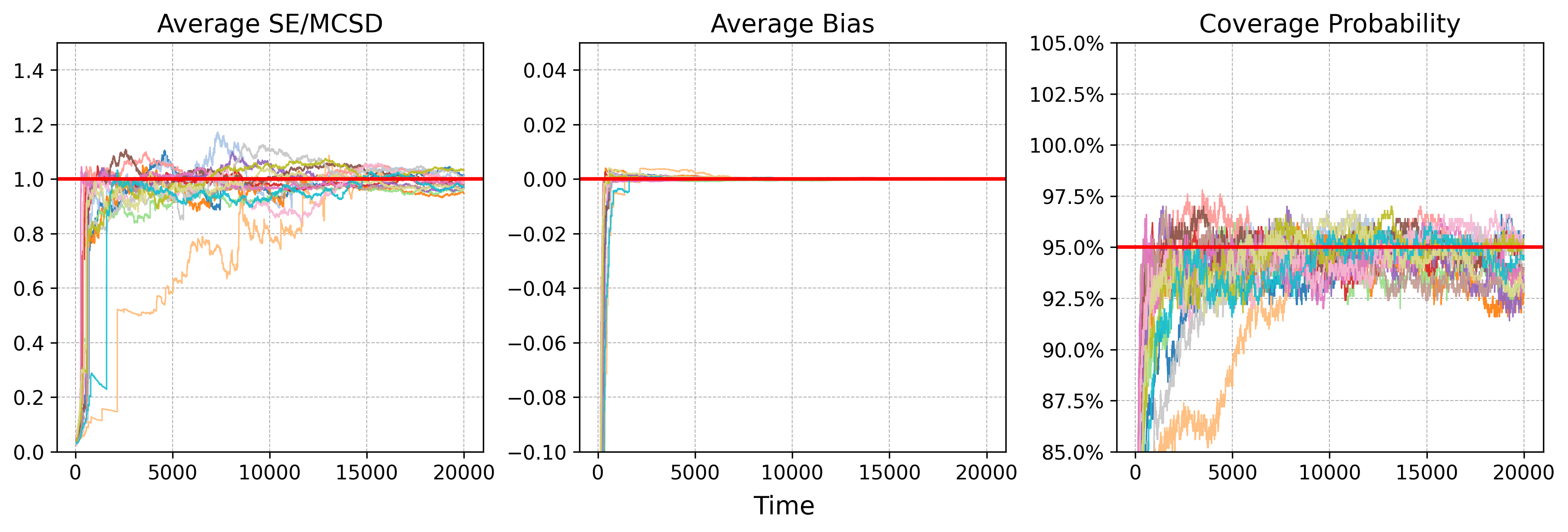}
    \caption{Consistency and asymptotic normality of the payoff matrix estimator $\widehat{A}_t$ under UCB exploration with $m=k=4$. Each color corresponds to one element of the $4\times4$ matrix $\widehat{A}_t$. The 16 curves approach the nominal reference line across all three evaluation metrics.}
    \label{fig:est_inf_ucb_mk}
\end{figure}

Figure~\ref{fig:est_mk} follows the same design as Figures~\ref{fig:multi_NE_sens_2} and~\ref{fig:multi_NE_sens_4}. It reports the convergence of the first coordinate of the estimated Nash equilibrium for both two players, as well as the convergence of the value estimator. In Figures~\ref{fig:est_mk}(a) and~\ref{fig:est_mk}(b), the curves under UCB and $\epsilon$-greedy nearly overlap, indicating that OnGameLearn performs similarly under the two exploration strategies.

Since we consider the larger action-space setting with $m=k=4$, Figure~\ref{fig:est_NEpoint_mk} further reports the convergence of all coordinates of the estimated Nash equilibrium under OnGameLearn and the baseline methods. The results show that OnGameLearn, under both $\epsilon$-greedy and UCB, provides consistent estimates of the Nash equilibrium. Moreover, the two OnGameLearn curves are very close and nearly overlap across all coordinates. In contrast, the baseline methods, EXP3 and Tsallis-INF, do not yield stable or consistent estimates of the Nash equilibrium.

Figures~\ref{fig:est_inf_eps_mk} and~\ref{fig:est_inf_ucb_mk} demonstrate the asymptotic normality of the payoff matrix estimator $\widehat{A}_t$ under both $\epsilon$-greedy and UCB exploration strategies, thereby verifying Theorem~\ref{thm:norm_matrix} in the larger action-space setting with $m=k=4$.

\subsubsection{Contextual Matrix Games}

We also conduct simulations for a larger action-space case in the contextual matrix game setting. We follow the general setup in Section~\ref{sec:simu_context}. The context vector $\bm z=(1,z_2)\in\mathbb R^2$ is generated as follows: $z_1\stackrel{\mathrm{i.i.d.}}{\sim}\mathrm{Uniform}(0,5)$.
The game involves $m=4$ actions for the row player and $k=4$ actions for the column player. The true parameters are specified as
\begin{equation*}
\begin{aligned}
\bm{\beta}^{11} &= (-1.694, -0.459), &
\bm{\beta}^{12} &= (1.120, -0.077), &
\bm{\beta}^{13} &= (-0.246, 0.690), &
\bm{\beta}^{14} &= (0.894, -0.760), & \\
\bm{\beta}^{21} &= (1.912, 0.161), &
\bm{\beta}^{22} &= (0.153, 0.720), &
\bm{\beta}^{23} &= (0.004, -0.431), &
\bm{\beta}^{24} &= (-1.712, 0.078), & \\
\bm{\beta}^{31} &= (-0.925, 0.654), &
\bm{\beta}^{32} &= (-0.002, -0.586), &
\bm{\beta}^{33} &= (0.716, 0.038), &
\bm{\beta}^{34} &= (1.215, 0.401), & \\
\bm{\beta}^{41} &= (-0.476, 0.270), &
\bm{\beta}^{42} &= (-1.736, -0.052), &
\bm{\beta}^{43} &= (-0.848, -0.472), &
\bm{\beta}^{44} &= (1.638, -0.015).
\end{aligned}
\end{equation*}
We use the same choices of $\epsilon_t$, $\eta_t$, and $c_t$ as in Section~\ref{sec:simu_context}. The termination time is set to $T=20{,}000$.

\begin{figure}[ht]
    \centering
    \includegraphics[width=1.0\linewidth]{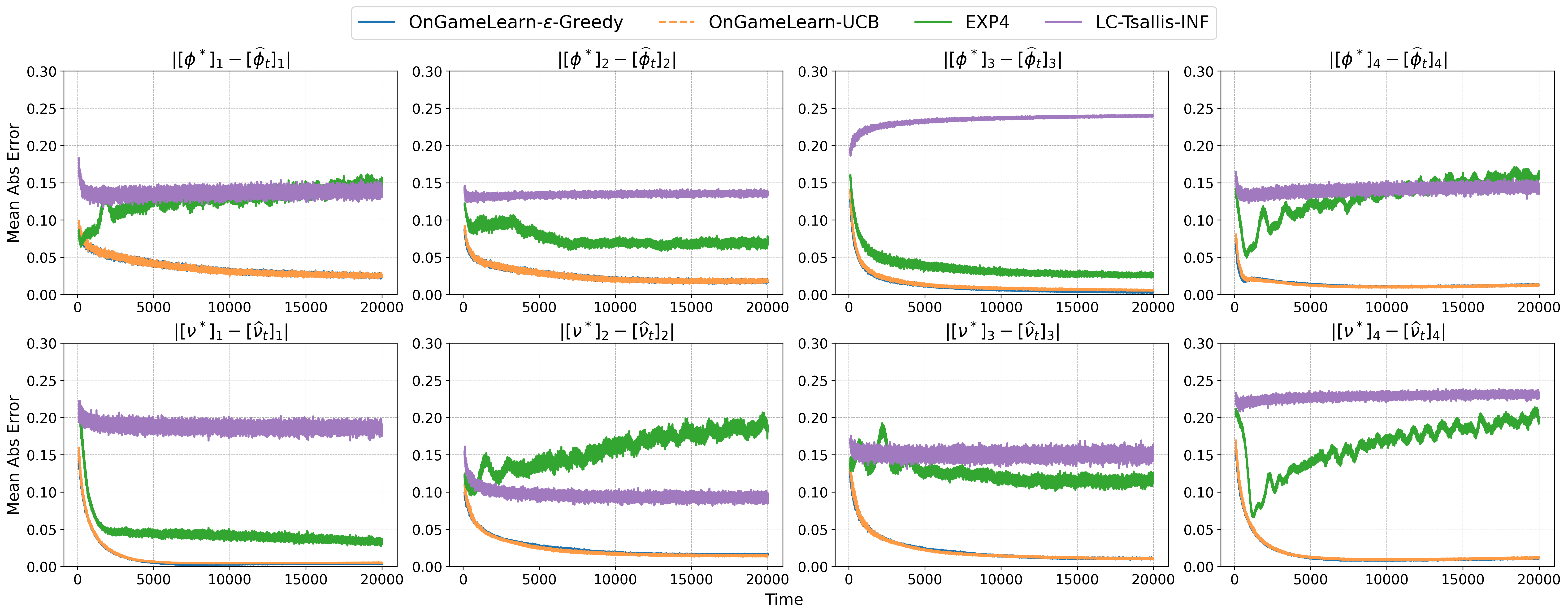}
    \caption{Convergence of the estimated Nash equilibrium when $m=k=4$. The curves under $\epsilon$-greedy and UCB exploration almost completely overlap. The curves under $\epsilon$-greedy and UCB exploration almost completely overlap.}
    \label{fig:est_NE_context_mk}
\end{figure}

\begin{figure}[ht]
    \centering
    \includegraphics[width=0.8\linewidth]{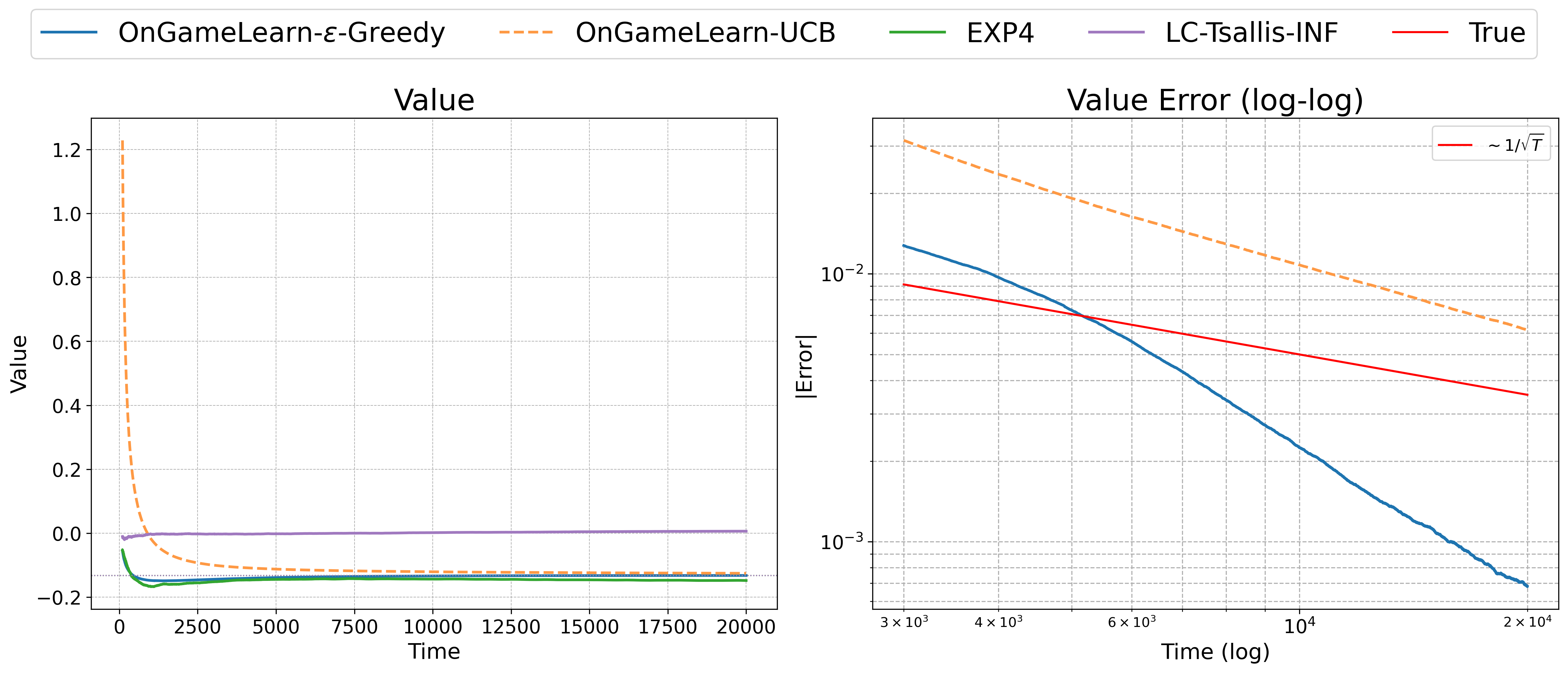}
    \caption{$\sqrt{T}$-consistency of the doubly robust value estimator $\widehat{V}_T^\text{conDR}$ when $m=k=4$.}
    \label{fig:est_value_context_mk}
\end{figure}

\begin{figure}[ht]
    \centering
    \includegraphics[width=1.0\linewidth]{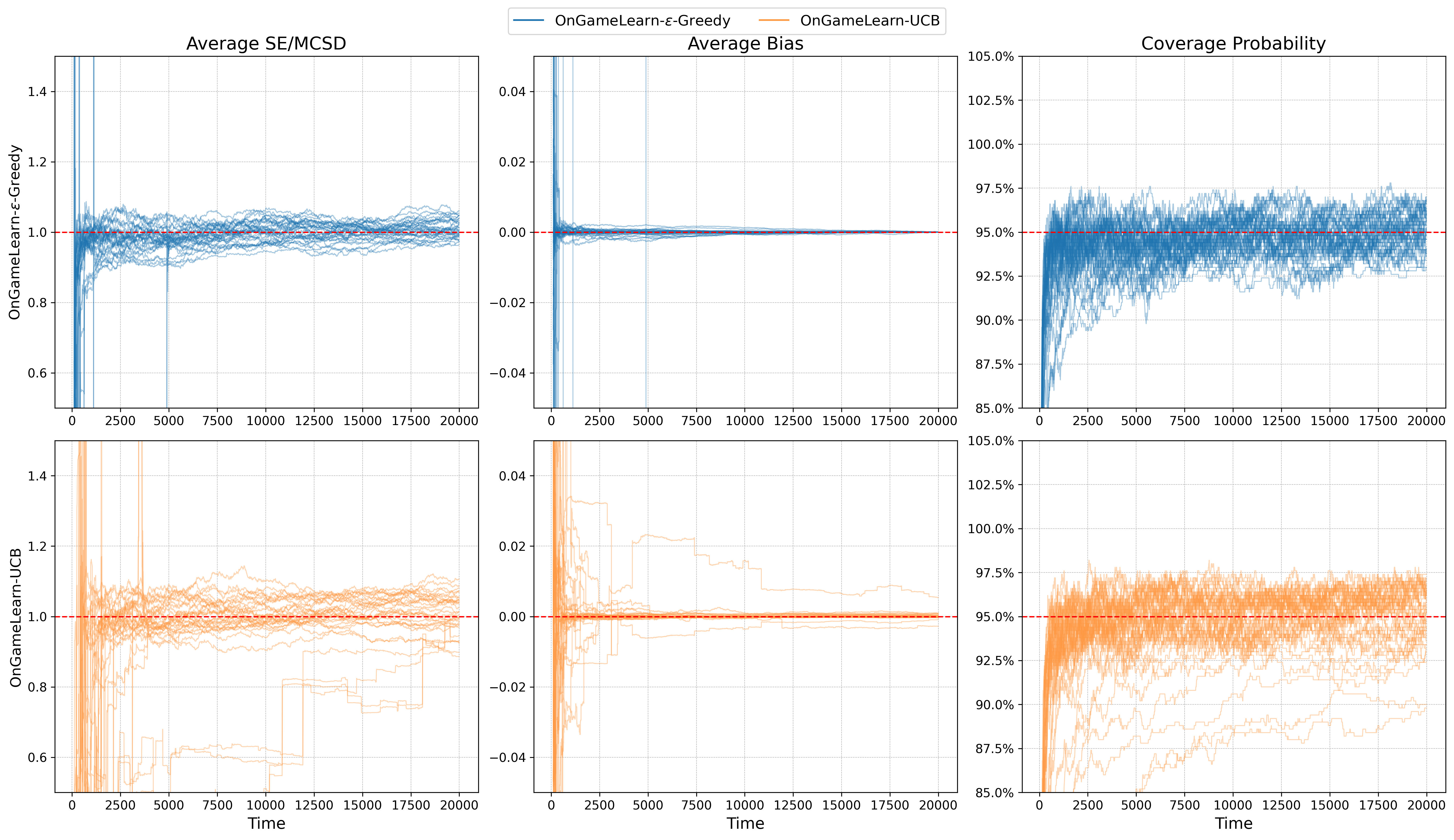}
    \caption{Consistency and asymptotic normality of the estimated parameters $\widehat{\bm{\beta}}_t^{11}, \widehat{\bm{\beta}}_t^{12}, \dots, \widehat{\bm{\beta}}_t^{44}$ under $\epsilon$-Greedy and UCB exploration with $m = k = 4$. Each curve corresponds to one parameter component, and all $4 \times 4 \times 2 = 32$ curves converge to the nominal reference line across the three evaluation metrics. \textbf{Upper panel:} OnGameLearn-$\epsilon$-Greedy. \textbf{Lower panel:} OnGameLearn-UCB.}
    \label{fig:est_inf_context_mk}
\end{figure}

Figures~\ref{fig:est_NE_context_mk},~\ref{fig:est_value_context_mk} and~\ref{fig:est_inf_context_mk} demonstrate that OnGameLearn continues to perform well in contextual settings with a larger number of actions, significantly outperforming the baselines. Furthermore, OnGameLearn-UCB exhibits a slightly slower convergence rate compared to OnGameLearn-$\epsilon$-Greedy.

\section{Proof of Theoretical Results for Online Matrix Games}
\label{app:proof_bandit}

\subsection{Tail Bound of Payoff Matrix Estimator (Theorem \ref{thm:tail_bandit})}
\label{app:tail_bandit_proof}

\begin{proof}
Consider an action pair $(i,j)$ and define $\mathcal N_t^{ij}=\sum_{s=1}^t \mathbb I(i_s=i,j_s=j)$.
Recall that $\pi_s(i,j) = \mathbb P(i_s=i,j_s=j\mid \mathcal H_{s-1}) \ge \epsilon_s^2/mk$ for all $s\le t$, where $\epsilon_s$ is non-increasing, i.e., \(\epsilon_s\ge \epsilon_t\) for all \(s\le t\). Therefore,
\[
\sum_{s=1}^t \pi_s(i,j)
\ge
\sum_{s=1}^t \epsilon_s^2/mk
\ge t \epsilon_t^2/mk.
\]
Since
\[
\mathbb E\left[
\mathbb I(i_s=i,j_s=j)-\pi_s(i,j)
\mid
\mathcal H_{s-1}
\right]
=
\mathbb E(\mathbb I(i_s=i,j_s=j)\mid \mathcal H_{s-1})
- \pi_s(i,j)=  0
\]
and $\mathbb I(i_s=i,j_s=j)-\pi_s(i,j)\in[-1,1]$.
Hence, $\left\{ \mathbb I(i_s=i,j_s=j)-\pi_s(i,j) \right\}_{s=1}^t$ is a bounded martingale difference sequence.
By Lemma~\ref{lem:adaptive_chernoff}, for any $0<c<1$,
\[
\mathbb P\left(
\mathcal N_t^{ij}
\le
(1-c)\sum_{s=1}^t \pi_s(i,j)
\right)
\le
\exp\left\{
-\frac{c^2}{2}\sum_{s=1}^t \pi_s(i,j)
\right\}.
\]
Taking $c=1/2$, we obtain
\[
\mathbb P\left(
\mathcal N_t^{ij}
\le
\frac12 \sum_{s=1}^t \pi_s(i,j)
\right)
\le
\exp\left\{
-\frac18\sum_{s=1}^t \pi_s(i,j)
\right\}.
\]
Since $\sum_{s=1}^t \pi_s(i,j)\ge t\epsilon_t^2/mk$, it follows that
\begin{equation}
\label{equ1}
\mathbb P\left(
\mathcal N_t^{ij}
\le
\frac{t \epsilon_t^2}{2mk}
\right)
\le
\exp\left\{
-\frac{t \epsilon_t^2}{8mk}
\right\}.
\end{equation}
Define the event
\[
E_t^{ij}
=
\left\{
\mathcal N_t^{ij}
\ge
\frac{t \epsilon_t^2}{2mk}
\right\}.
\]
Then
\[
\mathbb P(E_t^{ij})
\ge
1-\exp\left\{
-\frac{t \epsilon_t^2}{8mk}
\right\}.
\]
Note that
\[
[\widehat A_t]_{ij}-A_{ij}
=
\frac{1}{\mathcal N_t^{ij}}
\sum_{s=1}^t e_s \mathbb I(i_s=i,j_s=j).
\]
On the event $E_t^{ij}$, we have $\mathcal N_t^{ij} \ge \frac{t \epsilon_t^2}{2mk}$.
Hence, $\left| [\widehat A_t]_{ij}-A_{ij} \right|>h$ and $E_t^{ij}$
imply that
\[
\left|
\sum_{s=1}^t e_s \mathbb I(i_s=i,j_s=j)
\right|
>
\frac{t \epsilon_t^2 h}{2mk}.
\]
Conditional on the selected action pair $(i,j)$, $e_s$ is mean-zero sub-Gaussian with variance $\sigma_{ij}^2$.
By Lemma~\ref{lem:adaptive-subgaussian}, it follows that
\[
\mathbb P\left(
\left|
[\widehat A_t]_{ij}-A_{ij}
\right|>h,
E_t^{ij}
\right)
\le
2\exp\left\{
-\frac{
t\epsilon_t^2h^2
}{
4mk\sigma_{ij}^2
}
\right\}.
\]
Combining the two bounds, we get
\begin{align*}
\mathbb P\left(
\left|
[\widehat A_t]_{ij}-A_{ij}
\right|>h
\right)
&\le
\mathbb P\left((E_t^{ij})^c\right) + \mathbb P\left(
\left|
[\widehat A_t]_{ij}-A_{ij}
\right|>h,
E_t^{ij}
\right) \\
& \le
\exp\left\{
-\frac{t \epsilon_t^2}{8mk}
\right\}
+
2\exp\left\{
-\frac{t \epsilon_t^2 h^2}{4mk\sigma_{ij}^2}
\right\}.
\end{align*}
Equivalently,
\[
\mathbb P\left(
\left|
[\widehat A_t]_{ij}-A_{ij}
\right|
\le h
\right)
\ge
1
-
\exp\left\{
-\frac{t \epsilon_t^2}{8mk}
\right\}
-
2\exp\left\{
-\frac{t \epsilon_t^2 h^2}{4mk\sigma_{ij}^2}
\right\}.
\]
\end{proof}

\subsection{Convergence of the Estimated Nash Equilibrium (Theorem \ref{thm:nash_con})}
\label{app:nash_con_proof}
\begin{proof}
We will Lemma~\ref{lem:adapted-continuity-equilibria} to prove the convergence of estimated Nash equilibrium. Under our setting, let $\widetilde{f}: \Delta_m \times \Delta_k \times \mathbb{R}^{m \times k} \times \mathbb{R}\rightarrow \mathbb{R}$ and $\widetilde{f}(\bm{x},\bm{y},A,\eta) = \bm{x}^\top A \bm{y} - \eta R_X(\bm{x})+\eta R_Y(\bm{y})$, then $\widetilde{f}(\bm{x},\bm{y},\widehat{A}_{t-1},\eta_t) = {\mathcal{L}}_t(\bm{x},\bm{y})$. 
${\mathcal{L}}_t(\bm{x},\bm{y})$ has a unique Nash equilibrium because it's strongly convex-strongly concave. 
To be more specific, we have $\widehat{\bm{x}}_t = X^\ast(\widehat{A}_{t-1},\eta_t)=\arg \min_{x \in \Delta_m} \max_{y \in \Delta_k}{\mathcal{L}}_t(\bm{x},\bm{y}) $ and $\widehat{\bm{y}}_t=Y^\ast(\widehat{A}_{t-1},\eta_t)=\arg  \max_{y \in \Delta_k}\min_{x \in \Delta_m}{\mathcal{L}}_t(\bm{x},\bm{y})$. Check Assumptions (A1)-(A4) in Lemma~\ref{lem:adapted-continuity-equilibria}.\\
(A1) We just need to show the two conditions in Remark 22 in \citet{feinberg2022continuity} holds. $\widetilde{f}$ is a continuous function and $\Delta_{k}$ is a simplex so the limit in condition(ii) must exist. \\
(A2) Similar arguments by using Remark 23. \\
(A3) For any parameter value \((A,\eta)\), the feasible strategy set of the row player is always \(\Delta_m\). 
Therefore, the row player's feasible-set correspondence constant with respect to \((A,\eta)\). 
Hence, it is lower semi-continuous. 
Indeed, for any open set \(G\subseteq \mathbb R^m\) satisfying 
\(G\cap \Delta_m\neq\emptyset\) for all \((A,\eta)\). 
Thus, Assumption (A3) holds. \\
(A4) Similar arguments as (A3).

Given that assumptions (A1)-(A4) are satisfied and that the strategy spaces $\Delta_{m}$ and $\Delta_{k}$ are convex, Lemma~\ref{lem:adapted-continuity-equilibria} establishes the upper semi-continuity for set-valued functions $X^\ast(\cdot,\cdot)$ and $Y^\ast(\cdot,\cdot)$.
Define the product norm on $\mathbb{R}^{m \times k} \times \mathbb{R}$ as
$$\Vert (A,\eta) \Vert_2\coloneqq (\Vert A \Vert_F^2+\eta^2)^{\frac{1}{2}},$$ 
where $\Vert A \Vert_F = \sqrt{\sum_{i=1} ^m \sum_{j=1}^k A_{ij}^2}$ denotes the Frobenius norm. The well-definedness of this norm is proved in Section \ref{lemma:norm}. We adopt the Euclidean norm $\Vert \cdot \Vert_2$ for vectors in $\mathbb{R}^m$ and $\mathbb{R}^n$.
By Corollary~\ref{coro:consistent_bandit}, for every $i\in[m]$ and
$j\in[k]$, we have $\left|[\widehat A_{t-1}]_{ij}-A_{ij}\right| \stackrel{p}{\to}0$.
Moreover, with $\eta_t\to0$, we get
\[
\left\Vert
(\widehat A_{t-1},\eta_t)-(A,0)
\right\Vert_2
\stackrel{p}{\to}0.
\]
Therefore,
\begin{equation*}
\operatorname{dist}
\left(X^\ast(\widehat A_{t-1},\eta_t),X^\ast(A,0) \right)
\stackrel{p}{\to}0, \quad 
\operatorname{dist} 
\left(Y^\ast(\widehat A_{t-1},\eta_t),Y^\ast(A,0) \right)
\stackrel{p}{\to}0.
\end{equation*}
Equivalently,
\[
\operatorname{dist}
\left(
(\widehat{\bm x}_t,\widehat{\bm y}_t),
\mathcal E(A)
\right)
\stackrel{p}{\to}
0.
\]
\end{proof}

\subsection{Convergence of the Behavior Strategies (Theorem~\ref{thm:behave_con})}
\label{app:nash_behave_proof}

The proof of Theorem~\ref{thm:behave_con} consists three parts to show the probability of exploration under $\epsilon$-Greedy, UCB and TS, respectively. Their explicit forms under different exploration strategies are given as follows:
\begin{enumerate}
    \item[\textnormal{(1)}] UCB and TS: $\widetilde{\bm x}^\ast = \operatorname{Clip}_{m}(\bm x^\ast,\epsilon_\infty/m)$, $\widetilde{\bm y}^\ast = \operatorname{Clip}_{k}(\bm y^\ast,\epsilon_\infty/k)$.
    \item[\textnormal{(2)}] $\epsilon$-Greedy: $\widetilde{\bm x}^\ast = (1-\epsilon_\infty)\bm x^\ast + \frac{\epsilon_\infty}{m}\bm 1_m$, $\widetilde{\bm y}^\ast =(1-\epsilon_\infty)\bm y^\ast + \frac{\epsilon_\infty}{k}\bm 1_k$.
\end{enumerate}
In particular, if $\epsilon_\infty=0$, then the exploration-induced
perturbation vanishes asymptotically, i.e., $\widetilde{\bm x}^\ast=\bm x^\ast$, and $\widetilde{\bm y}^\ast=\bm y^\ast$ for all three bandit algorithms.

\subsubsection{Proof of $\epsilon$-Greedy}
\begin{proof}
By Theorem~\ref{thm:nash_con}, we have $\widehat{\bm x}_t \stackrel{p}{\to} \bm x^\ast$, and $\widehat{\bm y}_t \stackrel{p}{\to} \bm y^\ast$. Under the $\epsilon$-Greedy exploration strategy, the behavior strategies are
\[
\widetilde{\bm x}_t =
(1-\epsilon_t)\widehat{\bm x}_t
+ \frac{\epsilon_t}{m}\mathbf 1_m,
\qquad
\widetilde{\bm y}_t = (1-\epsilon_t)\widehat{\bm y}_t
+ \frac{\epsilon_t}{k}\mathbf 1_k.
\]
With $\epsilon_t\to\epsilon_\infty$, the Continuous Mapping Theorem implies
\begin{align*}
\widetilde{\bm x}_t
\stackrel{p}{\to}
(1-\epsilon_\infty)\bm x^\ast
+ \frac{\epsilon_\infty}{m}\mathbf 1_m, 
\qquad
\widetilde{\bm y}_t
\stackrel{p}{\to}
(1-\epsilon_\infty)\bm y^\ast
+ \frac{\epsilon_\infty}{k}\mathbf 1_k,
\end{align*}
which are the limiting behavior policies under the $\epsilon$-Greedy.
\end{proof}

\subsubsection{Proof of UCB}
\label{app:nash_behave_proof_ucb}
\begin{proof}
The upper bound on each $A_{ij}$ at time $t$ is given by
\begin{equation*}
[\widehat{A}_{t-1}]_{i j}+c_t \widehat{\sigma}_{t-1}^{ij} = [\widehat{A}_{t-1}]_{i j}+c_t \frac{1}{\sqrt{\max\{1,\mathcal{N}_{t-1}^{ij}\}}},
\end{equation*}
where $c_t$ is a constant controlling the exploration level. By Corollary~\ref{coro:consistent_bandit}, $[\widetilde{A}_{t-1}]_{i j}$ is consistent when $t\epsilon_t^2 \to \infty$. Therefore, it remains to show that the
exploration bonus vanishes in probability. By \eqref{equ1}, for every action pair $(i,j)$,
\begin{equation*}
\mathbb P\left(
\mathcal N_{t-1}^{ij}\ge
\frac{(t-1) \epsilon_{t-1}^2}{2mk}
\right) \ge
1- \exp\left\{ -\frac{(t-1) \epsilon_{t-1}^2}{8mk} \right\}.
\end{equation*}
Since $(t-1) \epsilon_{t-1}^2\to\infty$, we also have $\exp\left\{ -{(t-1) \epsilon_{t-1}^2}/{8mk} \right\} \to 0$. When $\mathcal N_{t-1}^{ij} > \frac{(t-1) \epsilon_{t-1}^2}{2mk}$, we have
\[
c_t
\frac{1}{\sqrt{\mathcal N_{t-1}^{ij}}}
\le
c_t \sqrt{\frac{2mk}{(t-1) \epsilon_{t-1}^2}}.
\]
Consequently, since $c_t/\sqrt{(t-1) \epsilon_{t-1}^2}\to 0$ as $t \to\infty$, for any fixed small $h>0$ and all sufficiently large $t$ satisfying $c_t \sqrt{\frac{2mk}{(t-1) \epsilon_{t-1}^2} } \le h$, there is
\[
\mathbb P\left( c_t \frac{1}{\sqrt{\mathcal N_{t-1}^{ij}}} \le h \right)
\ge
\mathbb P\left( \mathcal N_{t-1}^{ij} \ge \frac{(t-1) \epsilon_{t-1}^2}{2mk} \right)
\ge
1 - \exp\left\{ -\frac{(t-1) \epsilon_{t-1}^2}{8mk} \right\}
\to 1.
\]
Therefore,
\begin{equation}
\label{equ3}
c_t \frac{1}{\sqrt{\mathcal N_{t-1}^{ij}}} \stackrel{p}{\to} 0.    
\end{equation}
Combining this with $[\widehat A_t]_{ij} \stackrel{p}{\to} A_{ij}$, we obtain $[\widetilde A_t]_{ij} \stackrel{p}{\to} A_{ij}$. 
The Nash equilibrium induced by the optimistic payoff matrix $\widetilde A_t$ and $\eta_t$ is given by $X^\ast(\widetilde A_t,\eta_t)$ and $Y^\ast(\widetilde A_t,\eta_t)$, as defined in Section~\ref{app:nash_con_proof}. 

Since Euclidean projection onto a nonempty closed convex set is non-expansive,
the clipping mapping $(\bm u,\epsilon) \mapsto \operatorname{Clip}_d(\bm u,\epsilon)$
is jointly continuous for $\epsilon<1/d$. The clipped behavior strategies are 
\[
\widetilde{\bm x}_t =
\operatorname{Clip}_m
\left(
X^\ast(\widetilde A_t,\eta_t),
\frac{\epsilon_t}{m}
\right),
\qquad
\widetilde{\bm y}_t =
\operatorname{Clip}_k
\left(
Y^\ast(\widetilde A_t,\eta_t),
\frac{\epsilon_t}{k}
\right).
\]
Under Assumption~\ref{ass:unique} $\epsilon_t\to\epsilon_\infty$, the Continuous Mapping Theorem implies
\[
\widetilde{\bm x}_t
\stackrel{p}{\to}
\operatorname{Clip}_m
\left(
\bm x^\ast,
\epsilon_\infty
\right),
\qquad
\widetilde{\bm y}_t
\stackrel{p}{\to}
\operatorname{Clip}_k
\left(
\bm y^\ast,
\epsilon_\infty
\right).
\]
Therefore, the clipped behavior strategies converge to the projections of the
true Nash equilibrium strategies onto the corresponding truncated simplices.
\end{proof}

\subsubsection{Proof of TS}
\label{app:nash_behave_proof_ts}
\begin{proof}
Under Thompson sampling in the non-contextual matrix game setting, we draw samples
\[
\widetilde A_{ij}
\mid \mathcal H_{t-1}
\sim
\mathcal N\left(
[\widehat A_{t-1}]_{ij},
\frac{\rho^2}{\mathcal N_{t-1}^{ij}}
\right)
\]
to derive $[\widetilde A_{t}]_{ij}$. Equivalently, we may write
\[
[\widetilde A_{t}]_{ij} =
[\widehat A_{t-1}]_{ij} + \frac{\rho}{\sqrt{\mathcal N_{t-1}^{ij}}}\varsigma_t^{ij},
\]
where $\varsigma_t^{ij}\sim \mathcal N(0,1)$ and $\rho$ is a known parameter. By Corollary~\ref{coro:consistent_bandit}, we have $[\widehat A_{t-1}]_{ij} \stackrel{p}{\to} A_{ij}$. It remains to show that the posterior
sampling error vanishes in probability, i.e.,
\[
\frac{\rho}{\sqrt{\mathcal N_{t-1}^{ij}}}\varsigma_t^{ij}
\stackrel{p}{\to}
0.
\]
Using the same argument as in the proof of \eqref{equ3}, we obtain
\[
\frac{\rho}{\sqrt{\mathcal N_{t-1}^{ij}}}
\stackrel{p}{\to} 0
\]
for every action pair $(i,j)$. Moreover, with $|\varsigma_t^{ij}|=O_p(1)$,
thus we have,
\[
\frac{\rho}{\sqrt{\mathcal N_{t-1}^{ij}}}\varsigma_t^{ij}
= o_p(1)O_p(1) = o_p(1).
\]
Combining this with the consistency of $[\widehat A_{t-1}]_{ij}$, we obtain $\widetilde A_t^{ij} \stackrel{p}{\to} A_{ij}$.
Hence, the Thompson-sampled payoff matrix converges in probability to the true
payoff matrix entrywise. 
By arguments analogous to those in Sections~\ref{app:nash_con_proof} and~\ref{app:nash_behave_proof_ucb}, the behavior strategies induced by the Nash equilibrium of the Thompson-sampled game converge in probability to the projections of the true equilibrium strategies onto the corresponding truncated simplices with lower-bound constraints $\epsilon_\infty/m$ and $\epsilon_\infty/k$, respectively.
\end{proof}

\subsection{Asymptotic Normality of Payoff Matrix Estimator (Theorem \ref{thm:norm_matrix})}
\label{app:norm_matrix_proof}

\begin{proof}
The estimator $[\widehat{A}_t]_{ij}$ is given by,
\begin{equation*}
    [\widehat{A}_t]_{ij} = \left(\sum_{s=1}^{t} \mathbb{I}(i_s=i,j_s=j) \right)^{-1} \left( \sum_{s=1}^{t}r_s \mathbb{I}(i_s=i,j_s=j) \right) = \left(\sum_{s=1}^{t} a_s^{ij} \right)^{-1} \left( \sum_{s=1}^{t} a_s^{ij} r_s \right).
\end{equation*}
Then we have
\begin{equation*}
    \sqrt{t}\left([\widehat{A}_t]_{ij} - A_{ij}\right) = \left(\frac{1}{t}\sum_{s=1}^{t} a_s^{ij} \right)^{-1} \left(\frac{1}{\sqrt{t}}\sum_{s=1}^{t} a_s^{ij} e_s\right).
\end{equation*}
1. First we want to show 
$$\frac{1}{\sqrt{t}} \sum_{s=1}^{t}a_s^{ij} e_s\stackrel{D}{\to} \mathcal{N}\left(0,\widetilde{x}_i^\ast \widetilde{y}_j^\ast \sigma_{ij}^2 \right).$$

For $1 \le k \le t$ and $t \ge 1$, define $\mathcal{H}_{tk} = \mathcal{H}_k$ and
\[
\mathcal M_{tk} = \frac{1}{\sqrt{t}} \sum_{s=1}^k a_s^{ij} e_s.
\]
Note that $\mathbb{E}\left( a_s^{ij} e_s \vert \mathcal{H}_{t,s-1} \right) = \mathbb{E}\left(\mathbb{E}\left(a_s^{ij} e_s \vert i_s,j_s \right)\vert \mathcal{H}_{t,s-1} \right)=0$, thus $\left\{\mathcal M_{tk},\mathcal{H}_{tk}, 1\le k\le t,t\right\}$ is a Martingale array. Then we can get use of Martingale Central Limit Theorem to prove the convergence (see Theorem 2.3 in \citet{hall1980martingale}).

(a) Check the conditional Lindeberg condition. For any $\delta>0$,
\begin{align*}
    & \sum_{s=1}^t \mathbb{E} \left\{\frac{1}{t} a_s^{ij} e_s^2 \mathbb{I}\left\{\vert a_s^{ij} e_s \vert > \delta \sqrt{t} \right\} \Big \vert \mathcal{H}_{t,s-1}\right\} \\
    = & \frac{1}{t} \sum_{s=1}^t \mathbb{E} \left\{a_s^{ij} e_s^2\mathbb{I}\left\{ a_s^{ij} e_s^2 > \delta^2 t \right\} \Big \vert \mathcal{H}_{s-1}\right\} \\
    = & \frac{1}{t} \sum_{s=1}^t \mathbb{P}(i_s=i,j_s=j\vert \mathcal{H}_{s-1}) \mathbb{E} \left\{e_s^2\mathbb{I}\left\{e_s^2 a_s^{ij} > \delta^2 t \right\} \Big \vert \mathcal{H}_{s-1}\right\} \\
    \le & \frac{1}{t} \sum_{s=1}^t \mathbb{E} \left\{e_{s(i,j)}^2 \mathbb{I}\left\{e_{s(i,j)}^2 > \delta^2 t \right\} \Big \vert \mathcal{H}_{s-1}\right\},
\end{align*}
where $e_{s(i,j)}=e_s$ if $i_s=i,j_s=j$ and 0 otherwise. $e_s$ conditioned on $i_s,j_s$ are $i.i.d.$ from sub-Gaussian distribution. Then the last inequality comes to be 
\begin{equation*}
    \mathbb{E} \left\{e_{1(i,j)}^2 \mathbb{I}\left\{e_{1(i,j)}^2 > \delta^2 t \right\}\right\}.
\end{equation*}
Because $e_{1(i,j)}^2 \mathbb{I}\left\{e_{1(i,j)}^2 > \delta^2 t \right\}$ is dominated by $e_{1(i,j)}^2$ with $\mathbb{E}(e_{1(i,j)}^2)\le\sigma_{ij}^2$ and it converges to 0 almost surely when $t\to \infty$, then by Dominated Convergence Theorem, we have
\begin{equation*}
    \sum_{s=1}^t \mathbb{E} \left[\frac{1}{t} a_s^{ij} e_s^2 \mathbb{I}\left\{\vert  a_s^{ij} e_s \vert > \delta \sqrt{t} \right\} \Big \vert \mathcal{H}_{t,s-1}\right] \to 0,
\end{equation*}
as $t \to \infty$.

(b) Find the limit of the conditional variance. The conditional variance is
\begin{align*} 
    \zeta_t^2 &\coloneqq \sum_{s=1}^t \mathbb{E} \left(\frac{1}{t} a_s^{ij} e_s^2 \Big \vert \mathcal{H}_{t,s-1}\right) \\
    &= \frac{1}{t}\sum_{s=1}^t \mathbb{E} \left\{ \mathbb{E} \left(a_s^{ij} e_s^2\vert \mathcal{H}_{s-1}, i_s=i,j_s=j \right) \Big \vert \mathcal{H}_{s-1}\right\} \\
    &= \frac{1}{t}\sum_{s=1}^t \mathbb{E} \left\{a_s^{ij} \vert \mathcal{H}_{s-1}\right\} \mathbb{E} \left(e_s^2 \vert i_s=i,j_s=j \right)\\
    &= \frac{\sigma_{ij}^2}{t} \sum_{s=1}^t \mathbb{P}\left(i_s=i,j_s=j \vert \mathcal{H}_{s-1}\right) \\
    &= \frac{\sigma_{ij}^2}{t} \sum_{s=1}^t \widetilde{x}_{si} \widetilde{y}_{sj}.
\end{align*}
By Theorem~\ref{thm:behave_con} and the Continuous Mapping Theorem, $\widetilde{x}_{si} \widetilde{y}_{sj} \to \widetilde{x}_i^\ast \widetilde{y}_j^\ast$, then we have $\zeta_t^2 \stackrel{p}{\to} \widetilde{x}_i^\ast \widetilde{y}_j^\ast \sigma_{ij}^2$. It follows by Martingale Central Limit Theorem that 
\[
\frac{1}{\sqrt{t}} \sum_{s=1}^{t}a_s^{ij} e_s \stackrel{D}{\to} \mathcal{N}\left(0,\widetilde{x}_i^\ast \widetilde{y}_j^\ast \sigma_{ij}^2 \right).
\]

2. Next, we want to find the limit of $\left(\frac{1}{t}\sum_{s=1}^{t} a_s^{ij} \right)^{-1}$. We first derive the limit of $\xi_t \coloneqq\frac{1}{t}\sum_{s=1}^{t} a_s^{ij}$.

Note that $a_s^{ij}$ is $\mathcal{H}_s$-measurable for each $s$. Let $w$ be a random variable that always take 1. Then $\mathbb{E} (w) = 1 < \infty$. Also, for any $h >0$ and each $s$, $\mathbb{P}(a_s^{ij} > h) \le \mathbb{P}(w > h)$ because $a_s^{ij} \le w$ always holds. 
By Theorem 2.19 in \citet{hall1980martingale},
\begin{equation*}
    \frac{1}{t}\sum_{s=1}^{t} \left\{ a_s^{ij} - \mathbb{E}(a_s^{ij} \Big\vert \mathcal{H}_{s-1}) \right\} = \xi_t - \frac{\zeta_t^2}{\sigma_{ij}^2}  \stackrel{p}{\to} 0.
\end{equation*}
We have proved $\zeta_t^2 \stackrel{p}{\to} \widetilde{x}_i^\ast \widetilde{y}_j^\ast \sigma_{ij}^2$ in the last part. Thus we have $\xi_t \stackrel{p}{\to} \widetilde{x}_i^\ast \widetilde{y}_j^\ast$. By Continuous Mapping Theorem,
\begin{equation*}
    \left(\frac{1}{t}\sum_{s=1}^{t} a_s^{ij} \right)^{-1} \stackrel{p}{\to} \left(\widetilde{x}_i^\ast \widetilde{y}_j^\ast \right)^{-1}.
\end{equation*}
Finally, combine above results using Slutsky Theorem, we have
\begin{equation*}
    \sqrt{t}\left([\widehat{A}_t]_{ij} - A_{ij}\right) \stackrel{D}{\to} \mathcal{N}(0,s_{ij}^2),
\end{equation*}
where $s_{ij}^2 = \left(\widetilde{x}_i^\ast \widetilde{y}_j^\ast \right)^{-1} \sigma_{ij}^2$.

The consistency of the variance estimator requires to show
\begin{equation*}
    \frac{\sum_{s=1}^t a_s^{ij} \widehat{e}_s^2}{\sum_{s=1}^t a_s^{ij}} \left(\frac{1}{t}\sum_{s=1}^{t} a_s^{ij} \right)^{-1} \stackrel{p}{\to} s_{ij}^2,
\end{equation*}
where $\widehat{e}_s = r_s - [\widehat{A}_t]_{ij}$. We just need to show $\frac{\sum_{s=1}^t a_s^{ij} \widehat{e}_s^2}{\sum_{s=1}^t a_s^{ij}} \stackrel{p}{\to} \sigma_{ij}^2$.
We expand the square term of $\widehat{e}_s^2$ and get
\begin{equation*}
    \widehat{e}_s^2 = \left(A_{ij} - [\widehat{A}_t]_{ij}\right)^2 + 2\left(A_{ij} - [\widehat{A}_t]_{ij}\right)e_s+e_s^2.
\end{equation*}
For the first term, by Corollary \ref{coro:consistent_bandit},
\begin{equation*}
    \frac{\sum_{s=1}^t a_s^{ij} \left(A_{ij} - [\widehat{A}_t]_{ij}\right)^2}{\sum_{s=1}^t a_s^{ij}} = \left(A_{ij} - [\widehat{A}_t]_{ij}\right)^2 \stackrel{p}{\to} 0.
\end{equation*}
The second term 
\begin{equation*}
    \frac{\sum_{s=1}^t a_s^{ij} \left(A_{ij} - [\widehat{A}_t]_{ij}\right) e_s}{\sum_{s=1}^t a_s^{ij}} = \left(A_{ij} - [\widehat{A}_t]_{ij}\right)\frac{\sum_{s=1}^t a_s^{ij}  e_s}{\sum_{s=1}^t a_s^{ij}} \stackrel{p}{\to} 0,
\end{equation*}
since $ [\widehat{A}_t]_{i,j} \stackrel{p}{\to} [A]_{i,j}$ and  $\frac{\sum_{s=1}^t a_s^{ij} e_s}{\sum_{s=1}^t a_s^{ij}} \stackrel{p}{\to}0$ by Lemma~\ref{lem:adaptive-subgaussian}. Lastly for the final term, according to Weak Law of Large Numbers, we have
\begin{equation*}
    \frac{\sum_{s=1}^t a_s^{ij} e_s^2}{\sum_{s=1}^t a_s^{ij}} \stackrel{p}{\to} \sigma_{ij}^2.
\end{equation*}
The proof is completed by combing all the results above and using Continuous Mapping Theorem.
\end{proof}

\subsection{Convergence Rate of of Saddle Value (Corollary \ref{thm:NEvalue})}
\label{app:saddlevalue_rate}
\begin{proof}
Following Lemma \ref{lemma:sd_dif}, we define the function $g:\Delta_m \times \Delta_k \times I\subseteq \mathbb{R}^m\times\mathbb{R}^m\times \mathbb{R}^{mk+1}$, where $I_1 = [a_{11},b_{11}] \times \dots \times [a_{1k},b_{1k}] \times \dots \times [a_{mk},b_{mk}] \times [0,\eta_1]$ with consistent $[\widehat{A}_t]_{ij} \in [a_{ij},b_{ij}]$. The function $g$ is expressed as
    \begin{equation*}
        g(\bm{x},\bm{y},\bm{t}) =g(\bm{x},\bm{y},(A,\eta)) = \widetilde{f}(\bm{x},\bm{y},A,\eta) = \bm{x}^\top A \bm{y} -\eta R_X(\bm{x}) + \eta R_Y(\bm{y}),
    \end{equation*}
    where $\widetilde{f}$ is defined in Section \ref{app:nash_con_proof}.
While $g$ and $\widetilde{f}$ are functionally equivalent, the key distinction lies in their domain specification: $g$ explicitly combines the matrix $A$ and and scalar $\eta$ into a single parameter space $\mathbb{R}^{mk+1}$, whereas $\widetilde{f}$ treats them separately. The saddle value function is $V(\bm{t})=V((A,\eta))=\min_{\bm{x}\in X} \max_{\bm{y}\in Y}\widetilde{f}(\bm{x},\bm{y},A,\eta) = \widetilde{f}(X^\ast(A,\eta),Y^\ast(A,\eta),A,\eta)$. 

Let $\bm{t}_1 = (\widehat{A}_{t-1},\eta_t)$ and $\bm{t}_2 = (A,0)$, by Lemma \ref{lemma:sd_dif}, we have 
\begin{equation*}
    \nabla_tg(X^\ast(\bm{t}_1),Y^\ast(\bm{t}_2),\bar{\bar{\bm{t}}})(\bm{t}_1-\bm{t}_2) \le \mathcal{L}_t(\widehat{\bm{x}}_t,\widehat{\bm{y}}_t) - V_A^\ast = V(\bm{t}_1) - V(\bm{t}_2) \le \nabla_tg(X^\ast(\bm{t}_2),Y^\ast(\bm{t}_1),\bar{\bm{t}})(\bm{t}_1-\bm{t}_2),
\end{equation*}
where $\bar{\bm{t}}$ and $\bar{\bar{\bm{t}}}$ satisfying $\bar{\bm{t}} = 
(1-\theta_1)\bm{t}_1+\theta_1\bm{t}_2$ and $\bar{\bar{\bm{t}}} = 
(1-\theta_2)\bm{t}_1+\theta_2\bm{t}_2$, $0<\theta_1,\theta_2<1$.
By the boundedness of the regularization terms $R_X$ and $R_Y$ and the compactness of $\Delta_m$ and $\Delta_k$, there exists a constant $C_g>0$ such that $\Vert \nabla_tg(X^\ast(\bm{t}_2),Y^\ast(\bm{t}_1),\bar{\bm{t}})\Vert_2 \le C_g$ and $\Vert \nabla_tg(X^\ast(\bm{t}_1),Y^\ast(\bm{t}_2),\bar{\bar{\bm{t}}})\Vert_2 \le C_g$ for any $\theta_1$ and $\theta_2$. Then we have $|\mathcal{L}_t(\widehat{\bm{x}}_t,\widehat{\bm{y}}_t) - V_A^\ast| \le C_g \Vert \bm{t}_1-\bm{t}_2\Vert_2$, where the parameter distance is given by
\begin{equation*}
    \Vert \bm{t}_1-\bm{t}_2\Vert_2 = \sqrt{\sum_{i=1}^m\sum_{j=1}^k \left([\widehat{A}_{t-1}]_{ij}-A_{ij}\right)^2 + \eta_t^2} = O_p\left((t^{-1}+\eta_t^2)^{1/2}\right).
\end{equation*}
Therefore, we have proved $\mathcal{L}_t(\widehat{\bm{x}}_t,\widehat{\bm{y}}_t) - V_A^\ast = O_p\left((t^{-1}+\eta_t^2)^{1/2}\right)$.
\end{proof}

\subsection{$\sqrt{T}$-consistency of Value Estimator (Theorem \ref{thm:saddlerate})}
\label{app:dr}

\begin{proof}
Here we first state the formal scenario of Assumption~\ref{ass:mix}.
\begin{assumption}[{\emph{Margin: Formal Statement}}]
\label{ass:margin_formal}
Let $(\bm x^\ast,\bm y^\ast)$ denote the Nash equilibrium of the true payoff
matrix $A$. One of the following conditions holds:
\begin{enumerate}
    \item[\textnormal{(i)}]
    There exists a constant $C_l>0$ such that $x_i^\ast\ge C_l$ and $y_j^\ast\ge C_l$ for all $i\in[m]$ and $j\in[k]$.
    \item[\textnormal{(ii)}]
    Let $(\bar{\bm x}^\ast,\bar{\bm y}^\ast)$ denote the Nash equilibrium of the regularized game associated with any payoff matrix $A'$ and regularization parameter $\eta$. There exists a constant $\delta>0$ such that, whenever $\left\|(A',\eta)-(A,0)\right\|_2<\delta$, we have
    \[
    x_{i_0}^\ast=0
    \ \to\
    \bar x_{i_0}^\ast=0,
    \qquad
    y_{j_0}^\ast=0
    \ \to\
    \bar y_{j_0}^\ast=0,
    \]
    for all $i_0\in[m]$ and $j_0\in[k]$.
\end{enumerate}
\end{assumption}
Here, $\left\| (A',\eta)-(A,0) \right\|_2 = \sqrt{ \left\| A'-A \right\|_F^2 + \eta^2 }$ denotes the product norm on $\mathbb R^{mk}\times\mathbb R$.
Here $\Vert(\cdot\,,\cdot) \Vert_2$ denotes the product norm in $\mathbb{R}^{m k}\times \mathbb{R}$, and the details are provided in Lemma \ref{lemma:norm}. 

We need several steps to prove the square root rate of convergence. We define 
\begin{align*}
    \widetilde{V}_T &\equiv \frac{1}{T}\sum_{t=1}^T \widehat{\bm{x}}_t^\top \widehat{A}_{t-1}\widehat{\bm{y}}_t + \frac{\widehat{x}_{t,i_t}\widehat{y}_{t,j_t}}{{\pi}_t(i_t,j_t)}(r_t - A_{i_t,j_t}).
\end{align*}

\textbf{Step 1:} We first show $\widehat{V}_T^\text{DR} - \widetilde{V}_T = O_p(T^{-1/2})$.

\begin{align*} 
\widehat V_T^{\mathrm{DR}}-\widetilde V_T &= \frac{1}{T} \sum_{t=1}^T \frac{ \widehat{x}_{t,i_t}\widehat{y}_{t,j_t} }{ \pi_t(i_t,j_t) } \left[ A-\widehat A_{t-1} \right]_{i_t,j_t}, 
\end{align*} 
where $\pi_t(i,j) = \widetilde{x}_{ti}\widetilde{y}_{tj}$ denotes the conditional probability of selecting action pair $(i,j)$ under the behavior policy. 

We now consider the three exploration strategies separately. 
\paragraph{$\epsilon$-Greedy.} 
Under the $\epsilon$-Greedy strategy, 
\[ \widetilde{\bm x}_t = (1-\epsilon_t)\widehat{\bm x}_t + \frac{\epsilon_t}{m}\bm 1_m, \qquad \widetilde{\bm y}_t = (1-\epsilon_t)\widehat{\bm y}_t + \frac{\epsilon_t}{k}\bm 1_k. \] 
For any action pair $(i,j)$, 
\begin{align*} 
\frac{ \widehat{x}_{ti}\widehat{y}_{tj} }{ \widetilde{x}_{ti}\widetilde{y}_{tj} } = \frac{ \widehat{x}_{ti}\widehat{y}_{tj} }{\left[(1-\epsilon_t)\widetilde{x}_{ti}+\frac{\epsilon}{m}\right] \left[(1-\epsilon_t)\widetilde{y}_{tj}+\frac{\epsilon}{k}\right]} \le \frac{ \widehat{x}_{ti}\widehat{y}_{tj} }{ (1-\epsilon_t)^2 \widehat{x}_{ti}\widehat{y}_{tj} } = \frac{1}{(1-\epsilon_t)^2} \le C_\epsilon, 
\end{align*} 
whenever $\widehat{x}_{ti}\widehat{y}_{tj}>0$ for all sufficiently large $t$ under condition~\textnormal{(i)} of Assumption~\ref{ass:margin_formal}. Under condition~\textnormal{(ii)}, actions outside the support of the true Nash equilibrium remain outside the support of the estimated equilibrium for all sufficiently large $t$. Hence, whenever $\widehat{x}_{ti}\widehat{y}_{tj}=0$, the corresponding importance ratio is equal to zero.

Therefore, there exists a constant $C_\epsilon<\infty$ such that, for all sufficiently large $t$, \[ 0 \le \frac{ \widehat{x}_{ti}\widehat{y}_{tj} }{ \pi_t(i,j) } \le C_\epsilon \] uniformly over all action pairs $(i,j)$.

\paragraph{UCB and TS.} 
We establish a uniform bound for the importance ratio $\widehat{x}_{ti}\widehat{y}_{tj}/\pi_t(i,j)$ by first proving the almost-sure convergence of the estimated strategies.
Using the entrywise tail bound established in Theorem~\ref{thm:tail_bandit} and the condition $t\epsilon_t^2/{\log t} \to \infty$, since the number of action pairs is finite, it follows that
\[
\sum_{t=1}^{\infty}
\mathbb P\left(
\left\|
\widehat A_t-A
\right\|_F
>
h
\right)
<
\infty.
\]
By the Borel--Cantelli lemma, for every fixed $h>0$,
\[
\mathbb P\left(
\left\|
\widehat A_t-A
\right\|_F
>
h
\text{ infinitely often}
\right)
=
0.
\]
To conclude convergence, consider the countable sequence $h_\ell = \frac{1}{\ell}$, $\ell\in\mathbb N$. For each $\ell\in\mathbb N$, with probability one, there exists an almost-surely finite time $T_\ell$ such that $\Vert \widehat A_t-A \Vert_F \le {1}/{\ell}$
for all $t\ge T_\ell$. Taking the countable intersection of these probability-one events yields
\[
\widehat A_t
\stackrel{\mathrm{a.s.}}{\to}
A.
\]
Since $\eta_t\to0$ and the equilibrium mapping is continuous at $(A,0)$, the almost-sure convergence of $\widehat A_t$ further implies $\widehat{\bm x}_t \stackrel{a.s.}{\to} \bm x^\ast$ and $\widehat{\bm y}_t \stackrel{a.s.}{\to} \bm y^\ast$.

Under the condition $\sum_{t=1}^{\infty} \mathbb P\left( \Vert \widehat A_t-A \Vert_F > h \right) < \infty$, which is sufficient for 
$$\sum_{t=1}^{\infty} \mathbb P\left( \Vert \widetilde A_t-A \Vert_F > h \right) < \infty,$$
we also have
\[
\widetilde{\bm x}_t
\stackrel{a.s.}{\to}
\operatorname{Clip}_m
\left(
\bm x^\ast,\frac{\epsilon_\infty}{m}
\right),
\qquad
\widetilde{\bm y}_t
\stackrel{a.s.}{\to}
\operatorname{Clip}_k
\left(
\bm y^\ast,\frac{\epsilon_\infty}{k}
\right)
\]
for UCB and TS.

Under condition~\textnormal{(i)} of Assumption~\ref{ass:mix}, all entries of
$\bm x^\ast$ and $\bm y^\ast$ are bounded away from zero. Hence, for every
action pair $(i,j)$ and sufficiently large $t$, there exists a constant $\delta$ such that $\widehat{x}_{ti} > C_l-\delta$ and $\widehat{y}_{tj} > C_l-\delta$ for all $i \in[m]$, $j \in [k]$. Thus for all sufficiently large $t>T_0$, 
\[
\frac{
\widehat{x}_{ti}\widehat{y}_{tj}
}{
\widetilde{x}_{ti}\widetilde{y}_{tj}
}
\le
\frac{1}{
\max\{C_l-\delta,\epsilon_\infty/m \}
\max\{C_l-\delta,\epsilon_\infty/k \}
}
<\infty
\]
almost surely.

Under condition~\textnormal{(ii)}, for all sufficiently large $t$, by local support stability, the corresponding numerator satisfies $\widehat{x}_{ti}\widehat{y}_{tj}=0$ for any action pair $(i,j)$ outside the support of the true Nash equilibrium almost surely. Therefore, the importance ratio is eventually equal to zero for such action pairs. 
Therefore, there exists a finite constant $C_\epsilon={1}/{
\max\{C_l-\delta,\epsilon_\infty/m \}
\max\{C_l-\delta,\epsilon_\infty/k \}}< \infty$ and an almost-surely finite time $T_0$ such that
\[
0
\le
\frac{
\widehat{x}_{ti}\widehat{y}_{tj}
}{
\widetilde{x}_{ti}\widetilde{y}_{tj}
}
\le C_\epsilon
\]
for all $i\in[m]$, $j\in[k]$, and $t\ge T_0$.

Combining the above results, under each of the three exploration strategies, there exist a finite constant $c_t>0$ and an almost-surely finite time $T_0$ such that $0 \le \frac{ \widehat{x}_{ti}\widehat{y}_{tj} }{ \pi_t(i,j) } \le C_\epsilon$ uniformly over all action pairs $(i,j)$ and all $t\ge T_0$. Therefore, 
\begin{align*} 
\left| \widehat V_T^{\mathrm{DR}}-\widetilde V_T \right| \le \frac{1}{T} \sum_{t=1}^{T_0-1} \frac{ \widehat{x}_{t,i_t}\widehat{y}_{t,j_t} }{ \pi_t(i_t,j_t) } \left| \left[ A-\widehat A_{t-1} \right]_{i_t,j_t} \right| + \frac{C}{T} \sum_{t=T_0}^{T} \left| \left[ A-\widehat A_{t-1} \right]_{i_t,j_t} \right|. 
\end{align*} 
The first term contains only finitely many summands and is therefore $O_p(T^{-1})$. By Theorem~\ref{thm:norm_matrix}, the second term is $O_p(T^{-1/2})$. Hence, \[ \widehat V_T^{\mathrm{DR}}-\widetilde V_T = O_p(T^{-1/2}). \]

\textbf{Step 2:} Then we will show $\widetilde{V}_T - V^\ast_A = O_p(T^{-1/2})$.

\begin{equation}
\label{tildeV_T - barV_T}
    \widetilde{V}_T - V^\ast_A = \frac{1}{T}\sum_{t=1}^T \widehat{\bm{x}}_t^\top \widehat{A}_{t-1}\widehat{\bm{y}}_t - (\bm{x}^\ast)^\top A \bm{y}^\ast + \frac{\widehat{x}_{t,i_t}\widehat{y}_{t,j_t}}{{\pi}_t(i_t,j_t)}(r_t-A_{i_t,j_t})
\end{equation}

Based on Corollary \ref{thm:NEvalue}, we know that 
\begin{align*}
    &\widehat{\bm{x}}_t^\top \widehat{A}_{t-1}\widehat{\bm{y}}_t - (\bm{x}^\ast)^\top A \bm{y}^\ast\\
    =& \mathcal{L}_t(\widehat{\bm{x}}_t,\widehat{\bm{y}}_t)+ \eta_t R_X(\widehat{\bm{x}}_t)- \eta_t R_Y(\widehat{\bm{y}}_t)-(\bm{x}^\ast)^\top A \bm{y}^\ast\\
    \le &\mathcal{L}_t(\widehat{\bm{x}}_t,\widehat{\bm{y}}_t)- (\bm{x}^\ast)^\top A \bm{y}^\ast + 2\max\left\{\ln m,\ln k\right\}\eta_t\\
    =& O_p(t^{-1/2}),
\end{align*}
when $\eta_t = \mathcal{O}(t^{-1/2})$. Thus we have
\begin{equation}
\label{firstinstep2}
\frac{1}{T} \sum_{t=1}^T \widehat{\bm{x}}_t^\top \widehat{A}_{t-1}\widehat{\bm{y}}_t - (\bm{x}^\ast)^\top A \bm{y}^\ast = O_p(T^{-1/2}).
\end{equation}

Then we focus on the second term in Equation~\eqref{tildeV_T - barV_T}. Define 
\[
Z_t = \frac{ \widehat{x}_{t,i_t}\widehat{y}_{t,j_t} }{ \pi_t(i_t,j_t) } \left( r_t-A_{i_t,j_t} \right). 
\] 
Then We aim to show that 
\[ 
\frac{1}{T}\sum_{t=1}^T Z_t = O_p(T^{-1/2}). 
\] 
By the uniform bound established in Step~1, there exists a constant $C_\epsilon$ such that 
\[ 
\max_{i\in[m],\,j\in[k]} \frac{ \widehat{x}_{t,i}\widehat{y}_{t,j} }{ \pi_t(i,j) }\le C_\epsilon 
\] 
almost surely for all sufficiently large $t$. Moreover, the importance ratios are finite at every time point under all three exploration strategies. For any fixed $N\ge1$, define 
\[ 
\mathcal E_N = \left\{ \sup_{t\ge N}\max_{i\in[m],\,j\in[k]} \frac{ \widehat{x}_{ti}\widehat{y}_{tj} }{ \pi_t(i,j) }\le C_\epsilon \right\},
\] 
and we have $\mathbb P(\mathcal E_N)\to 1$ as $N\to\infty$. Define the localized sequence 
\[ 
Z_t^{(C_\epsilon)} = Z_t \, \mathbb I \left\{\max_{i\in[m],\,j\in[k]} \frac{ \widehat{x}_{t,i}\widehat{y}_{t,j} }{ \pi_t(i,j) }\le C_\epsilon \right\}. 
\] 
Since the estimated target policy and the behavior policy are determined before the reward at time $t$ is observed, we have 
\[
\mathbb E\left( Z_t^{(C_\epsilon)} \mid \mathcal H_{t-1} \right) = \mathbb E\left[ \left. \mathbb I \left\{\max_{i\in[m],\,j\in[k]} \frac{ \widehat{x}_{ti}\widehat{y}_{tj} }{ \pi_t(i,j) }\le C\right\} \frac{ \widehat{x}_{t,i_t}\widehat{y}_{t,j_t} }{ \pi_t(i_t,j_t) } \mathbb E\left( e_t \mid i_t,j_t,\mathcal H_{t-1} \right) \right| \mathcal H_{t-1} \right] =0. 
\]
Therefore, $\{Z_t^{(C_\epsilon)},\mathcal H_t\}_{t\ge1}$ is a martingale difference sequence. 
Since that the conditional noise variances are uniformly bounded: i.e., $\max_{i\in[m],\,j\in[k]} \sigma_{ij}^2 < \infty$. Then 
\begin{align*} 
\mathbb E\left[ \left. \left( Z_t^{(C_\epsilon)} \right)^2 \right| \mathcal H_{t-1} \right] 
&= \mathbb E\left[ \left. \mathbb I \left\{\max_{i\in[m],\,j\in[k]} \frac{ \widehat{x}_{t,i}\widehat{y}_{t,j} }{ \pi_t(i,j) }\le C_\epsilon\right\} \left( \frac{ \widehat{x}_{t,i_t}\widehat{y}_{t,j_t} }{ \pi_t(i_t,j_t) } \right)^2 e_t^2 \right| \mathcal H_{t-1} \right] \\ 
&\le C_\epsilon^2 \max_{i\in[m],\,j\in[k]} \sigma_{ij}^2. 
\end{align*} 
Since the cross terms of a martingale difference sequence vanish, for every fixed $N$, 
\[
\mathbb E\left[ \left( \frac{1}{\sqrt{T}} \sum_{t=N}^T Z_t^{(C_\epsilon)} \right)^2 \right] = \frac{1}{T} \sum_{t=N}^T \mathbb E\left[ \left( Z_t^{(C_\epsilon)} \right)^2 \right] \le C_\epsilon^2 \max_{i\in[m],\,j\in[k]} \sigma_{ij}^2. 
\]
By Chebyshev's inequality, 
\[ 
\frac{1}{\sqrt{T}} \sum_{t=N}^T Z_t^{(C_\epsilon)} = O_p(1). 
\] 
On the event $\mathcal E_N$, we have $Z_t^{(C_\epsilon)}=Z_t$ for all $t\ge N$. Since $\mathbb P(\mathcal E_N)\to 1$ as $N\to\infty$, it follows that 
\[ 
\frac{1}{\sqrt{T}} \sum_{t=N}^T Z_t = O_p(1). 
\] 
The initial segment contains only finitely many terms. Since each importance ratio is upper bounded by $mk/\epsilon_N^2$, \[ \frac{1}{T} \sum_{t=1}^{N-1}Z_t = O_p(T^{-1}) = o_p(T^{-1/2}). \] Therefore, \[ \frac{1}{T} \sum_{t=1}^T \frac{ \widehat{x}_{t,i_t}\widehat{y}_{t,j_t} }{ \pi_t(i_t,j_t) } \left( r_t-A_{i_t,j_t} \right) = O_p(T^{-1/2}). \] Combining this result with \eqref{firstinstep2}, we obtain \[ \widetilde V_T-V_A^\ast = O_p(T^{-1/2}). \]
This completes the proof.

\end{proof}

\subsection{Regret Bound (Theorem~\ref{thm:regret_bandit})}
\label{app:thm:rerget_proof}

\begin{proof}
For each time $t$, let $B_t$ denote the payoff matrix used to compute the
regularized equilibrium strategies and the corresponding strategies
\[
\bm x_t^B
=
X^\ast(B_t,\eta_t),
\qquad
\bm y_t^B
=
Y^\ast(B_t,\eta_t).
\]
Recall that
\[
\operatorname{Gap}_A(\bm x,\bm y)
=
\max_{\bm y'\in\Delta_k}
\bm x^\top A\bm y'
-
\min_{\bm x'\in\Delta_m}
(\bm x')^\top A\bm y.
\]
For the implemented behavior strategies, we insert the intermediate strategies
$(\bm x_t^B,\bm y_t^B)$ and the intermediate payoff matrix $B_t$:
\begin{align*}
\operatorname{Gap}_A
\left(
\widetilde{\bm x}_t,\widetilde{\bm y}_t
\right)
=&
\max_{\bm y\in\Delta_k}
\widetilde{\bm x}_t^\top A\bm y
-
\min_{\bm x\in\Delta_m}
\bm x^\top A\widetilde{\bm y}_t
\\
=&
\left[
\max_{\bm y\in\Delta_k}
\widetilde{\bm x}_t^\top A\bm y
-
\max_{\bm y\in\Delta_k}
(\bm x_t^B)^\top A\bm y
\right]
+
\left[
\max_{\bm y\in\Delta_k}
(\bm x_t^B)^\top A\bm y
-
\max_{\bm y\in\Delta_k}
(\bm x_t^B)^\top B_t\bm y
\right]
\\
&+
\left[
\max_{\bm y\in\Delta_k}
(\bm x_t^B)^\top B_t\bm y
-
\min_{\bm x\in\Delta_m}
\bm x^\top B_t\bm y_t^B
\right]
+
\left[
\min_{\bm x\in\Delta_m}
\bm x^\top B_t\bm y_t^B
-
\min_{\bm x\in\Delta_m}
\bm x^\top A\bm y_t^B
\right]
\\
&+
\left[
\min_{\bm x\in\Delta_m}
\bm x^\top A\bm y_t^B
-
\min_{\bm x\in\Delta_m}
\bm x^\top A\widetilde{\bm y}_t
\right].
\end{align*}

We first control the perturbations caused by the implemented behavior strategies.
Since there exists a constant $L_A$ satisfying
$\|A\|_{\max}\equiv\max_{i\in[m],\,j\in[k]}|A_{ij}|\le L_A$, we have
\begin{equation*}
\max_{\bm y\in\Delta_k}
\widetilde{\bm x}_t^\top A\bm y
-
\max_{\bm y\in\Delta_k}
(\bm x_t^B)^\top A\bm y
\le
\max_{\bm y\in\Delta_k}
(\widetilde{\bm x}_t-\bm x_t^B)^\top A\bm y
\le L_A
\left\|
\widetilde{\bm x}_t-\bm x_t^B
\right\|_1,
\end{equation*}
and similarly,
\begin{equation*}
\min_{\bm x\in\Delta_m}
\bm x^\top A\bm y_t^B
-
\min_{\bm x\in\Delta_m}
\bm x^\top A\widetilde{\bm y}_t
\le L_A
\left\|
\widetilde{\bm y}_t-\bm y_t^B
\right\|_1.
\end{equation*}
For the $\epsilon$-Greedy strategy, the bound follows from the fact that the
$\ell_1$ distance between any two probability vectors is at most two.
For UCB and TS, separate clipping moves at most $\epsilon_t$ total probability
mass for each player. Since the total mass added by clipping equals the total
mass removed, the resulting $\ell_1$ distance is at most $2\epsilon_t$.
Hence,
\[
L_A
\left[
\left\|
\widetilde{\bm x}_t-\bm x_t^B
\right\|_1
+
\left\|
\widetilde{\bm y}_t-\bm y_t^B
\right\|_1
\right]
\le
4L_A\epsilon_t.
\]
We next control the perturbations caused by replacing the true payoff matrix
$A$ with $B_t$. We have
\begin{equation*}
\max_{\bm y\in\Delta_k}
(\bm x_t^B)^\top A\bm y
-
\max_{\bm y\in\Delta_k}
(\bm x_t^B)^\top B_t\bm y
\le
\max_{\bm y\in\Delta_k}
(\bm x_t^B)^\top(A-B_t)\bm y
\le
\left\|A-B_t\right\|_{\max}.
\end{equation*}
Similarly,
\begin{equation*}
\min_{\bm x\in\Delta_m}
\bm x^\top B_t\bm y_t^B
-
\min_{\bm x\in\Delta_m}
\bm x^\top A\bm y_t^B
\le
\left\|A-B_t\right\|_{\max}.
\end{equation*}
Hence, the sum of the two payoff-matrix perturbation terms is bounded by
$2\left\|A-B_t\right\|_{\max}$.
Finally, we control the regularization bias. Since
$(\bm x_t^B,\bm y_t^B)$ is the saddle point of the regularized game
\[
\bm x^\top B_t\bm y
+
\eta_t R_X(\bm x)
-
\eta_t R_Y(\bm y),
\]
where $\bm x$ is the minimizing player and $\bm y$ is the maximizing player,
the saddle-point property implies that, for the row player,
\begin{equation*}
(\bm x_t^B)^\top B_t\bm y_t^B+
\eta_t R_X(\bm x_t^B) -
\eta_t R_Y(\bm y_t^B)
\le
\bm x^\top B_t\bm y_t^B
+ \eta_t R_X(\bm x)
- \eta_t R_Y(\bm y_t^B).
\end{equation*}
After canceling the common term $-\eta_tR_Y(\bm y_t^B)$, we obtain
\[
(\bm x_t^B)^\top B_t\bm y_t^B
-
\bm x^\top B_t\bm y_t^B
\le
\eta_t
\left[
R_X(\bm x)-R_X(\bm x_t^B)
\right].
\]
Therefore,
\begin{equation*}
(\bm x_t^B)^\top B_t\bm y_t^B
-
\min_{\bm x\in\Delta_m}
\bm x^\top B_t\bm y_t^B \le
\eta_t
\left[
\sup_{\bm x\in\Delta_m}R_X(\bm x)
-
\inf_{\bm x\in\Delta_m}R_X(\bm x)
\right]
= \eta_t\log m.
\end{equation*}
Similarly, for the column player, the saddle-point property implies that
\begin{equation*}
(\bm x_t^B)^\top B_t\bm y
+
\eta_t R_X(\bm x_t^B)
-
\eta_t R_Y(\bm y)
\le
(\bm x_t^B)^\top B_t\bm y_t^B
+
\eta_t R_X(\bm x_t^B)
-
\eta_t R_Y(\bm y_t^B).
\end{equation*}
After canceling the common term $\eta_tR_X(\bm x_t^B)$, we obtain
\[
(\bm x_t^B)^\top B_t\bm y
-
(\bm x_t^B)^\top B_t\bm y_t^B
\le
\eta_t
\left[
R_Y(\bm y)-R_Y(\bm y_t^B)
\right].
\]
Therefore,
\begin{equation*}
\max_{\bm y\in\Delta_k}
(\bm x_t^B)^\top B_t\bm y
-
(\bm x_t^B)^\top B_t\bm y_t^B
\le
\eta_t
\left[
\sup_{\bm y\in\Delta_k}R_Y(\bm y)
-
\inf_{\bm y\in\Delta_k}R_Y(\bm y)
\right]
= \eta_t\log k.
\end{equation*}
Combining the two bounds yields
\begin{equation*}
\max_{\bm y\in\Delta_k}
(\bm x_t^B)^\top B_t\bm y
-
\min_{\bm x\in\Delta_m}
\bm x^\top B_t\bm y_t^B
\le \eta_t(\log m+\log k).
\end{equation*}
Therefore, combining the three bounds yields
\begin{equation}
\label{equ:single-step-gap-bound}
\operatorname{Gap}_A
\left(
\widetilde{\bm x}_t,\widetilde{\bm y}_t
\right)
\le
4L_A\epsilon_t
+
2\left\|A-B_t\right\|_{\max}
+
\eta_t(\log m+\log k).
\end{equation}

We next establish a uniform high-probability bound for the payoff-matrix
estimation error. Let
$\sigma_{\max}=\max_{i\in[m],\,j\in[k]}\sigma_{ij}$.
By the entrywise tail bound for the payoff estimator, for every action pair
$(i,j)$, every $t\ge1$, and every $h>0$,
\begin{align}
\label{equ:entrywise-payoff-tail}
\mathbb P\left(
\left|
[\widehat A_t]_{ij} - A_{ij}
\right|
> h
\right)
\le&
\exp\left\{
-
\frac{t\epsilon_t^2}{8mk}
\right\} +
2\exp\left\{
-
\frac{t\epsilon_t^2h^2}{4mk\sigma_{ij}^2}
\right\}.
\end{align}
Since $\sigma_{ij}\le\sigma_{\max}$, the second term can be bounded uniformly
over all action pairs by replacing $\sigma_{ij}$ with $\sigma_{\max}$.
Fix any probability $\delta\in(0,1)$. Since ${t\epsilon_t^2}/{\log t}\to\infty$,
we have
\[
\sum_{t=1}^{\infty}
\exp\left\{
-\frac{t\epsilon_t^2}{8mk}
\right\}
<
\infty.
\]
Therefore, there exists an integer $t_0=t_0(\delta)$ such that
\begin{equation}
\label{equ:count-tail-summability}
mk
\sum_{t=t_0}^{\infty}
\exp\left\{
-\frac{t\epsilon_t^2}{8mk}
\right\}
\le
\frac{\delta}{2}.
\end{equation}
For any horizon $T\ge t_0$, define
\[
h_t(\delta)
=
2\sigma_{\max}
\sqrt{\frac{mk\log\left(4mkT/\delta\right)}{t\epsilon_t^2}},
\qquad
t_0\le t\le T.
\]
Applying a union bound over the $mk$ action pairs and over all time points
$t=t_0,\ldots,T$, we obtain
\begin{align*}
&\mathbb P\left(
\exists\,
t\in\{t_0,\ldots,T\}
:
\left\|
\widehat A_t-A
\right\|_{\max}
>
h_t(\delta)
\right)
\\
\le&
\sum_{t=t_0}^T
\sum_{i=1}^m
\sum_{j=1}^k
\mathbb P\left(
\left|
[\widehat A_t]_{ij}-A_{ij}
\right|
> h_t(\delta)
\right)
\le
mk
\sum_{t=t_0}^T
\exp\left\{
-\frac{t\epsilon_t^2}{8mk}
\right\}
+ 2mk
\sum_{t=t_0}^T
\exp\left\{
-
\frac{t\epsilon_t^2h_t^2(\delta)}{4mk\sigma_{\max}^2}
\right\}.
\end{align*}
By the definition of $h_t(\delta)$,
\[
\frac{
t\epsilon_t^2h_t^2(\delta)
}{
4mk\sigma_{\max}^2
}
=
\log\left(
\frac{4mkT}{\delta}
\right).
\]
Hence,
\begin{equation*}
2mk \sum_{t=t_0}^T
\exp\left\{
-
\frac{t\epsilon_t^2h_t^2(\delta)}{4mk\sigma_{\max}^2}
\right\} =
2mk(T-t_0+1)
\exp\left\{
-
\log\left(
\frac{4mkT}{\delta}
\right)
\right\}
\le \frac{\delta}{2}.
\end{equation*}
Combining this inequality with
\eqref{equ:count-tail-summability} yields
\[
\mathbb P\left(
\left\|
\widehat A_t-A
\right\|_{\max}
\le
2\sigma_{\max}
\sqrt{
\frac{
mk\log\left(4mkT/\delta\right)
}{
t\epsilon_t^2
}
}
\text{ for all }
t=t_0,\ldots,T
\right)
\ge
1-\delta.
\]
Therefore, with probability at least $1-\delta$, we have
\begin{align*}
\sum_{t=t_0}^T
\left\|
\widehat A_t-A
\right\|_{\max}
&\le
2\sigma_{\max}
\sqrt{
mk\log\left(4mkT/\delta\right)
}
\sum_{t=t_0}^T
\frac{1}{\epsilon_t\sqrt t}.
\end{align*}
Since the finite number of initial rounds $t<t_0$ contributes at most an
$O(1)$ term to the cumulative regret, we obtain
\[
\sum_{t=1}^T
\left\|
\widehat A_t-A
\right\|_{\max}
=
\widetilde O_p
\left(
\sqrt{mk}
\sum_{t=1}^T
\frac{1}{\epsilon_t\sqrt t}
\right),
\]
where $\widetilde O_p(\cdot)$ suppresses logarithmic factors in $T$.

We then analyze the three bandit algorithms separately.

Under the $\epsilon$-Greedy strategy, $B_t = \widehat A_{t-1}$. Substituting the estimation-error bound into
\eqref{equ:single-step-gap-bound} and summing over $t$ gives
\[
\operatorname{Reg}_{\epsilon\text{-greedy}}(T)
=
\widetilde O_p
\left(
\sqrt{mk}
\sum_{t=1}^T
\frac{1}{\epsilon_t\sqrt t}
+
\sum_{t=1}^T
\epsilon_t
+
\sum_{t=1}^T
\eta_t
\right).
\]
Under the UCB strategy, $B_t = \widehat A_{t-1} + \mathcal B_t$,
where
\[
[\mathcal B_t]_{ij}
=
\frac{c_t}{\sqrt{\max\left\{1,\mathcal N_{t-1}^{ij}\right\}}}.
\]
The adaptive count concentration bound implies that, uniformly over all action
pairs and all sufficiently large $t$,
\[
\mathcal N_{t-1}^{ij}
\ge
\frac{
(t-1)\epsilon_{t-1}^2
}{
2mk
}
\]
with probability tending to one. Therefore,
\[
\left\|
\mathcal B_t
\right\|_{\max}
=
\max_{i,j}
\frac{c_t}{\sqrt{\mathcal N_{t-1}^{ij}}}
= O_p
\left(
\frac{c_t\sqrt{mk}}{\epsilon_t\sqrt t}
\right).
\]
Using $\left\|
A-B_t
\right\|_{\max}
\le
\left\|
A-\widehat A_{t-1}
\right\|_{\max}
+ \left\|
\mathcal B_t
\right\|_{\max}$,
and substituting into
\eqref{equ:single-step-gap-bound}, we obtain
\[
\operatorname{Reg}_{\mathrm{UCB}}(T)
=
\widetilde O_p
\left(
\sqrt{mk}
\sum_{t=1}^T
\frac{1+c_t}{\epsilon_t\sqrt t}
+
\sum_{t=1}^T
\epsilon_t
+
\sum_{t=1}^T
\eta_t
\right).
\]
Under TS, $B_t = \widehat A_{t-1} + \mathcal S_t$,
where
\[
[\mathcal S_t]_{ij}
= \frac{\rho\,\xi_t^{ij}}{\sqrt{\mathcal N_{t-1}^{ij}}},
\qquad
\xi_t^{ij}
\stackrel{\mathrm{i.i.d.}}{\sim}
\mathcal N(0,1),
\]
and $\rho>0$ controls the scale of the Thompson-sampling perturbation.
Using the same count lower bound and the Gaussian tail bound for
$\xi_t^{ij}$, we obtain
\[
\left\|
\mathcal S_t
\right\|_{\max}
=
\widetilde O_p
\left(
\frac{\rho\sqrt{mk}}{\epsilon_t\sqrt t}
\right).
\]
Since $\left\|
A-B_t
\right\|_{\max}
\le
\left\|
A-\widehat A_{t-1}
\right\|_{\max}
+
\left\|
\mathcal S_t
\right\|_{\max}$,
substituting into
\eqref{equ:single-step-gap-bound} yields
\[
\operatorname{Reg}_{\mathrm{TS}}(T)
=
\widetilde O_p
\left(
(1+\rho)
\sqrt{mk}
\sum_{t=1}^T
\frac{1}{\epsilon_t\sqrt t}
+
\sum_{t=1}^T
\epsilon_t
+
\sum_{t=1}^T
\eta_t
\right).
\]
Finally, if $\epsilon_t=t^{-\alpha}$, then $\sum_{t=1}^T
\epsilon_t
=
O\left(
T^{1-\alpha}
\right)$,
and $\sum_{t=1}^T
{1}/(\epsilon_t\sqrt t) =
O\left(T^{1/2+\alpha} \right)$.
Taking $\alpha=1/4$ to balance the two exponents, we then have $\operatorname{Reg}(T) = \widetilde O_p \left( T^{3/4} \right)$.
The proof of Theorem~\ref{conthm:regret_bandit} is analogous.
\end{proof}

\subsection{Convergence Rate of Estimated Nash Equilibrium (Theorem~\ref{thm:rate_nash_rate})}
\label{app:thm:rate_nash_rate_proof}
\begin{proof}
\textbf{Step 1: Online matrix game case.}
By Theorem~\ref{thm:tail_bandit}, for each action pair $(i,j)$ and any $h>0$,
\[
\mathbb P\left(
\left|
[\widehat A_t]_{ij}-A_{ij}
\right|
> h
\right)
\le
\exp\left\{
-\frac{t\epsilon_t^2}{8mk}
\right\}
+
2\exp\left\{
-\frac{t\epsilon_t^2h^2}
{4mk\sigma_{ij}^2}
\right\}.
\]
Since the number of action pairs is finite, applying the union bound over $(i,j)$ gives
\[
\mathbb P\left(
\|\widehat A_t-A\|_{\max}
> h
\right)
\le
mk
\exp\left\{
-\frac{t\epsilon_t^2}{8mk}
\right\}
+2
\sum_{i=1}^m
\sum_{j=1}^k
\exp\left\{
-\frac{t\epsilon_t^2h^2}
{4mk\sigma_{ij}^2}
\right\},
\]
where $\|\cdot\|_{\max}$ denotes the maximum absolute entry of a matrix.
Therefore, 
\[
\|\widehat A_t-A\|_{\max}
=
\widetilde O_p
\left(
\frac{\sqrt{mk}}{\epsilon_t\sqrt t}
\right).
\]
We next translate the payoff estimation error into a Nash-gap bound. As we show in Section~\ref{app:thm:rerget_proof},
\[
\operatorname{Gap}_{\widehat A_{t-1}}
\left(
\widehat{\bm x}_t,
\widehat{\bm y}_t
\right)
\le
\eta_t(\log m+\log k).
\]
For any two payoff matrices $A$ and $B$, and any $(\bm x,\bm y)\in\Delta_m\times\Delta_k$, $\left|
\bm x^\top(A-B)\bm y
\right|
\le
\|A-B\|_{\max}$.
Hence,  $\left|
\operatorname{Gap}_{A}(\bm x,\bm y)
-
\operatorname{Gap}_{B}(\bm x,\bm y)
\right|
\le 2\|A-B\|_{\max}$.
Thus we obtain
\begin{align*}
\operatorname{Gap}_{A}
\left(
\widehat{\bm x}_t,
\widehat{\bm y}_t
\right)
&\le
\operatorname{Gap}_{\widehat A_{t-1}}
\left(
\widehat{\bm x}_t,
\widehat{\bm y}_t
\right)
+
2\|\widehat A_{t-1}-A\|_{\max} \\
&\le
\eta_t(\log m+\log k)
+
2\|\widehat A_{t-1}-A\|_{\max} \\
&=
\widetilde O_p
\left(
\frac{\sqrt{mk}}{\epsilon_t\sqrt t}
+
\eta_t
\right),
\end{align*}
where replacing $t-1$ by $t$ only changes constants and logarithmic factors.

Finally, we convert the Nash-gap bound into a distance-to-equilibrium bound. Since $\mathcal E(A)$ is a nonempty polyhedral set, by Hoffman's error bound for linear systems (Theorem 4.1 in~\cite{abderrahim2000hoffman}), there exists a constant $C_A<\infty$, depending only on the true payoff matrix $A$, such that $\operatorname{dist}\left((\bm x,\bm y),\mathcal E(A)\right)\le C_A \operatorname{Gap}_{A}(\bm x,\bm y)$
for all $(\bm x,\bm y)\in\Delta_m\times\Delta_k$. Applying this bound to
$(\bm x,\bm y)=(\widehat{\bm x}_t,\widehat{\bm y}_t)$ gives
\[
\operatorname{dist}
\left(
(\widehat{\bm x}_t,\widehat{\bm y}_t),
\mathcal E(A)
\right)
\le
C_A
\operatorname{Gap}_{A}
\left(
\widehat{\bm x}_t,
\widehat{\bm y}_t
\right).
\]
Therefore,
\[
\operatorname{dist}
\left(
(\widehat{\bm x}_t,\widehat{\bm y}_t),
\mathcal E(A)
\right)
=
\widetilde O_p
\left(
\frac{\sqrt{mk}}{\epsilon_t\sqrt t}
+ \eta_t
\right).
\]

\textbf{Step 2: Contextual matrix game case.}
For each context $\bm z$, let $C_{M(\bm z)}$ denote the Hoffman-type constant associated with the contextual zero-sum game whose payoff matrix is $M(\bm z)$. That is, $C_{M(\bm z)}$ is a constant satisfying
\[
\operatorname{dist}
\left(
(\bm x,\bm y),
\mathcal E(M(\bm z))
\right)
\le
C_{M(\bm z)}
\operatorname{Gap}_{M(\bm z)}(\bm x,\bm y),
\qquad
\forall(\bm x,\bm y)\in\Delta_m\times\Delta_k .
\]

We impose the following realized-context error-bound condition.

\begin{assumption}[Uniform realized error bound]
\label{conass:uniform_realized_error_bound}
There exists a set $\mathcal K$ of payoff matrices such that $\mathbb P\left(
\bm M(\bm z_t)\in\mathcal K
\right)
\to 1$,
and $\sup_{M\in\mathcal K} C_{M}<\infty$.
\end{assumption}
\begin{remark}
Assumption~\ref{conass:uniform_realized_error_bound} is closely related to imposed Assumptions~\ref{ass:mix},~\ref{ass:bound},  and~\ref{conass:margin}. Bounded contexts and a continuous payoff map ensure that the realized payoff matrices lie in a bounded set, while the uniform realized error-bound condition further rules out nearly degenerate games on this set.
\end{remark}

By Assumption~\ref{conass:uniform_realized_error_bound}, for every $\delta>0$, choosing any constant $M_\delta>\sup_{M\in\mathcal K} C_{M}$ gives
\[
\limsup_{t\to\infty}
\mathbb P\left(
C(\bm z_t)>M_\delta
\right)
\le \delta.
\]
This proves that $C(\bm z_t)=O_p(1)$. By the definition of $C_{M(\bm z_t)}$,
\[
\operatorname{dist}
\left(
(
\widehat\phi_t(\bm z_t),
\widehat\nu_t(\bm z_t)
),
\mathcal E(M(\bm z_t))
\right)
\le
C_{M(\bm z_t)}
\operatorname{Gap}_{M(\bm z_t)}
\left(
\widehat\phi_t(\bm z_t),
\widehat\nu_t(\bm z_t)
\right).
\]
Since $C(\bm z_t)=O_p(1)$, and the Nash-gap bound follows from the same payoff-perturbation and regularization-bias argument as in the non-contextual
case, the desired rate follows.
\end{proof}

\section{Proof of Theoretical Results for Online Contextual Matrix Games}
\label{app:proof_context}

\subsection{Tail Bound of Parameter Estimator (Theorem \ref{conthm:tail})}
\label{app:conthm:tail_proof}

\begin{proof}
We use the same notation as in Section~\ref{app:tail_bandit_proof} whenever the meaning is unchanged. Let $\mathcal S_t = \{1,\ldots,t\}$. Fix an action pair $(i,j)$ and also define $a_s^{ij} = \mathbb I\{i_s=i,j_s=j\}$.
By Equation~\eqref{equ:para_con}, we have
\[
\widehat{\bm\beta}_t^{ij}
-
\bm\beta^{ij}
=
\left(\sum_{s=1}^t a_s^{ij}\bm z_s\bm z_s^\top \right)^{-1}
\left( \sum_{s=1}^t a_s^{ij}\bm z_s e_s \right).
\]

We first control the minimum eigenvalue of the contextual design matrix.
Define the event
\[
\widetilde E_{ij}^t
=
\left\{
\lambda_{\min}
\left(
\sum_{s=1}^t a_s^{ij}\bm z_s\bm z_s^\top
\right)
>
\frac{
t\epsilon_t^2\lambda
}{
2mk
}
\right\}.
\]
For each $s\le t$, since $\|\bm z_s\|_2\le L_z$, we have $0 \preceq a_s^{ij}\bm z_s\bm z_s^\top \preceq L_z^2 I_d$. Moreover, 
\begin{equation*} 
\mathbb E\left( a_s^{ij}\bm z_s\bm z_s^\top \mid \mathcal H_{s-1}^{\mathrm{con}} \right) = \mathbb E\left[ \mathbb E\left( a_s^{ij}\bm z_s\bm z_s^\top \mid \mathcal H_{s-1}^{\mathrm{con}}, \bm z_s \right) \mid \mathcal H_{s-1}^{\mathrm{con}} \right] = \mathbb E\left[ \pi_s^\mathrm{con}(i,j\mid\bm z_s) \bm z_s\bm z_s^\top \mid \mathcal H_{s-1}^{\mathrm{con}} \right]. 
\end{equation*} 
Since $\pi_s(i,j\mid\bm z_s) \ge {\epsilon_s^2}/{mk}$, we obtain 
\begin{align*} 
\mathbb E\left( a_s^{ij}\bm z_s\bm z_s^\top \mid \mathcal H_{s-1}^{\mathrm{con}} \right) \succeq \frac{\epsilon_s^2}{mk} \mathbb E\left( \bm z_s\bm z_s^\top \mid \mathcal H_{s-1}^{\mathrm{con}} \right) \succeq \frac{\epsilon_s^2\lambda}{mk} I_d, 
\end{align*} 
where the last inequality follows from Assumption~\ref{ass:bound}. Since $\epsilon_s$ is non-increasing, $\epsilon_s\ge\epsilon_t$ for all $s\le t$. Therefore, 
\[ 
\sum_{s=1}^t \mathbb E\left( a_s^{ij}\bm z_s\bm z_s^\top \mid \mathcal H_{s-1}^{\mathrm{con}} \right) \succeq \frac{ t\epsilon_t^2\lambda }{ mk } I_d. 
\] 
Applying Lemma~\ref{lem:adapted-matrix-chernoff} with $R=L_z^2$, $\mu_t = \frac{ t\epsilon_t^2\lambda }{ mk }$ and $\delta=\frac12$, gives 
\begin{equation} 
\label{equ:context-design-event} 
\mathbb P\left( \lambda_{\min} \left(
\sum_{s=1}^t a_s^{ij}\bm z_s\bm z_s^\top
\right) \le \frac{ t\epsilon_t^2\lambda }{ 2mk } \right) \le d \exp\left\{ - \frac{ t\epsilon_t^2\lambda }{ 8mkL_z^2 } \right\}. 
\end{equation}

We next control the estimation error on the event $\widetilde E_{ij}^t$. On this event, every eigenvalue of $\sum_{s=1}^t a_s^{ij}\bm z_s\bm z_s^\top$ is greater than ${t\epsilon_t^2\lambda}/{2mk}$.
Since the function $v\mapsto\frac{1+v}{v^2}$ is decreasing on $(0,\infty)$, we have
\[
\frac{1}{v^2}
\le
\frac{
1+t\epsilon_t^2\lambda/(2mk)
}{
\left[
t\epsilon_t^2\lambda/(2mk)
\right]^2
}
\frac{1}{1+v}.
\]
Applying this inequality to each eigenvalue yields the desired
semidefinite inequality.
\[
\left(
\sum_{s=1}^t a_s^{ij}\bm z_s\bm z_s^\top
\right)^{-2}
\preceq
\frac{
1+t\epsilon_t^2\lambda/(2mk)
}{
\left[
t\epsilon_t^2\lambda/(2mk)
\right]^2
}
\left(
I_d+\sum_{s=1}^t a_s^{ij}\bm z_s\bm z_s^\top
\right)^{-1}.
\]
It follows that
\begin{align*}
\left\|
\widehat{\bm\beta}_t^{ij}
-
\bm\beta^{ij}
\right\|_1
&=
\left\|
\left(
\sum_{s=1}^t a_s^{ij}\bm z_s\bm z_s^\top
\right)^{-1}
\left( \sum_{s=1}^t a_s^{ij}\bm z_s e_s \right)
\right\|_1
\\
&\le
\sqrt d
\left\|
\left(
\sum_{s=1}^t a_s^{ij}\bm z_s\bm z_s^\top
\right)^{-1}
\left( \sum_{s=1}^t a_s^{ij}\bm z_s e_s \right)
\right\|_2
\\
&\le
\sqrt{
\frac{
d\left[
1+t\epsilon_t^2\lambda/(2mk)
\right]
}{
\left[
t\epsilon_t^2\lambda/(2mk)
\right]^2
}
}
\left\|
\left(
I_d+\sum_{s=1}^t a_s^{ij}\bm z_s\bm z_s^\top
\right)^{-1/2}
\left( \sum_{s=1}^t a_s^{ij}\bm z_s e_s \right)
\right\|_2.
\end{align*}
Consequently, $\left\{\left\|\widehat{\bm\beta}_t^{ij}-\bm\beta^{ij}\right\|_1>h\right\}$ and $\widetilde E_{ij}^t$ imply that
\[
\left\|
\left(
I_d+\sum_{s=1}^t a_s^{ij}\bm z_s\bm z_s^\top
\right)^{-1/2}
\sum_{s=1}^t a_s^{ij}\bm z_s e_s
\right\|_2
>
\frac{
h t\epsilon_t^2\lambda
}{
2mk
\sqrt{
d\left[
1+t\epsilon_t^2\lambda/(2mk)
\right]
}
}.
\]
Applying Lemma~\ref{lem:context-self-normalized} with 
$u = {
h t\epsilon_t^2\lambda}/{2mk
\sqrt{
d\left[
1+t\epsilon_t^2\lambda/(2mk)
\right]}}$,
we obtain
\begin{equation}
\label{equ:context-error-on-design-event}
\mathbb P\left(
\left\|
\widehat{\bm\beta}_t^{ij}
-
\bm\beta^{ij}
\right\|_1
>
h,
\,
\widetilde E_{ij}^t
\right)
\nonumber \le
\left(
1+tL_z^2
\right)^{d/2}
\exp\left\{
-
\frac{
t^2\epsilon_t^4\lambda^2h^2
}{
8dm^2k^2\sigma_{ij}^2
\left[
1+t\epsilon_t^2\lambda/(2mk)
\right]
}
\right\}.
\end{equation}
Combining \eqref{equ:context-design-event} and \eqref{equ:context-error-on-design-event}, we obtain
\begin{align*}
\mathbb P\left(
\left\|
\widehat{\bm\beta}_t^{ij}
-
\bm\beta^{ij}
\right\|_1
>
h
\right)
&\le 
\mathbb P\left(
(\widetilde E_{ij}^t)^c
\right)
+
\mathbb P\left(
\left\|
\widehat{\bm\beta}_t^{ij}
-
\bm\beta^{ij}
\right\|_1
>h,
\widetilde E_{ij}^t
\right)
\\
&\le
d\exp\left\{-\frac{t\epsilon_t^2\lambda}{8mkL_z^2}\right\}
+
\left(1+tL_z^2\right)^{d/2}
\exp\left\{
-
\frac{
t^2\epsilon_t^4\lambda^2h^2
}{
8dm^2k^2\sigma_{ij}^2
\left[
1+\frac{t\epsilon_t^2\lambda}{2mk}
\right]
}
\right\}.
\end{align*}
Equivalently,
\begin{align*}
\mathbb P\left(
\left\|
\widehat{\bm\beta}_t^{ij}
-
\bm\beta^{ij}
\right\|_1
\le
h\right) \ge
1-d\exp\left\{
-
\frac{
t\epsilon_t^2\lambda}{
8mkL_z^2}
\right\}
-
\left(
1+tL_z^2
\right)^{d/2}
\exp\left\{
-
\frac{
t^2\epsilon_t^4\lambda^2h^2}{8dm^2k^2\sigma_{ij}^2
\left[
1+\frac{t\epsilon_t^2\lambda}{2mk}
\right]
}
\right\}.
\end{align*}
\end{proof}

\subsection{Consistency of Payoff Matrix Estimator (Corollary \ref{concoro:consistent})}
\label{app:conthm:coro_tail_proof}
\begin{proof} 
For any fixed $h>0$, Theorem~\ref{conthm:tail} gives 
\begin{align*} 
\mathbb P\left( \left\| \widehat{\bm\beta}_t^{ij} - \bm\beta^{ij} \right\|_1 > h \right) \le d\exp\left\{ - \frac{ t\epsilon_t^2\lambda }{ 8mkL_z^2 } \right\} + \left( 1+tL_z^2 \right)^{d/2} \exp\left\{ - \frac{ t^2\epsilon_t^4\lambda^2h^2 }{ 8dm^2k^2\sigma_{ij}^2 \left[ 1+t\epsilon_t^2\lambda/(2mk) \right] } \right\}. 
\end{align*} 
Since ${ t\epsilon_t^2 }/{ \log t } \to \infty$,  we have $t\epsilon_t^2\to\infty$. Hence, for all sufficiently large $t$, 
\[ 
1+ \frac{ t\epsilon_t^2\lambda }{ 2mk } \le \frac{ t\epsilon_t^2\lambda }{ mk }. 
\] 
Therefore, 
\begin{align*} 
\frac{ t^2\epsilon_t^4\lambda^2h^2 }{ 8dm^2k^2\sigma_{ij}^2 \left[ 1+t\epsilon_t^2\lambda/(2mk) \right] } \ge \frac{ t\epsilon_t^2\lambda h^2 }{ 8dmk\sigma_{ij}^2 }. 
\end{align*} 
It follows that 
\begin{align*} 
\left( 1+tL_z^2 \right)^{d/2} \exp\left\{ - \frac{ t^2\epsilon_t^4\lambda^2h^2 }{ 8dm^2k^2\sigma_{ij}^2 \left[ 1+t\epsilon_t^2\lambda/(2mk) \right] } \right\} 
\le \exp\left\{ \frac d2 \log(1+tL_z^2) - \frac{ t\epsilon_t^2\lambda h^2 }{ 8dmk\sigma_{ij}^2 } \right\} \to0. 
\end{align*} 
Since $\log\left(1+tL_z^2 \right)= O(\log t)$, i.e., there exists a constant $C>0$ such that, for all sufficiently large $t$,
\[
\log\left(
1+tL_z^2
\right)
\le
C\log t.
\]
Consequently,
\begin{align*}
\frac d2
\log\left(
1+tL_z^2
\right)
- \frac{
t\epsilon_t^2\lambda h^2}{8dmk\sigma_{ij}^2} \le
\log t
\left[
\frac{dC}{2}
- \frac{\lambda h^2}{8dmk\sigma_{ij}^2}
\frac{t\epsilon_t^2}{\log t}
\right]
\to -\infty.
\end{align*}
Therefore,
\begin{align*}
&
\left(
1+tL_z^2
\right)^{d/2}
\exp\left\{
-
\frac{
t^2\epsilon_t^4\lambda^2h^2
}{
8dm^2k^2\sigma_{ij}^2
\left[
1+t\epsilon_t^2\lambda/(2mk)
\right]
}
\right\}
\to
0.
\end{align*}
The first term also converges to zero. Thus, for every fixed $h>0$, $\mathbb P\left( \left\| \widehat{\bm\beta}_t^{ij} - \bm\beta^{ij} \right\|_1 > h \right) \to 0$, which proves $\widehat{\bm\beta}_t^{ij} \stackrel{p}{\to} \bm\beta^{ij}$.
\end{proof}

\subsection{Convergence of the Estimated Nash Equilibrium (Theorem \ref{conthm:nash_con})}
\label{app:conthm:nash_con_proof}

\begin{proof}
The proof follows arguments similar to those in
Section~\ref{app:nash_con_proof}. We use the same function $\widetilde{f}(\bm x,\bm y,\widehat M_{t-1}(\bm z_t),\eta_t) = \mathcal L_t^{\mathrm{con}}(\bm x,\bm y)$
as defined previously. In the contextual setting, the payoff matrix $A$ is replaced by the context-dependent matrix $M(\bm z_t)$ or its estimator $\widehat M_{t-1}(\bm z_t)$. By Corollary~\ref{concoro:consistent}, for every $i\in[m]$ and
$j\in[k]$, we have $\Vert \widehat{\bm\beta}_{t}^{ij}-\bm\beta^{ij} \Vert_1 \stackrel{p} {\to}0$. Since $\|\bm z\|_\infty\le L_z$ for all $\bm z\in\mathcal Z$, we have
\[
\sup_{\bm z\in\mathcal Z}
\left|
[\widehat M_{t-1}(\bm z)]_{ij}
-
[M(\bm z)]_{ij}
\right|
\le
L_z
\left\Vert
\widehat{\bm\beta}_{t-1}^{ij}-\bm\beta^{ij}
\right\Vert_1
\stackrel{p}{\to}0.
\]
Therefore,
\[
\sup_{\bm z\in\mathcal Z}
\left\Vert
\widehat M_{t-1}(\bm z)-M(\bm z)
\right\Vert_F
\stackrel{p}{\to}0.
\]
Since $\eta_t\to0$, it follows that
\[
\left\Vert
\left(\widehat M_{t-1}(\bm z_t),\eta_t\right)
- \left(M(\bm z_t),0\right)
\right\Vert \le 
\sup_{\bm z\in\mathcal Z}
\left\Vert
\left(\widehat M_{t-1}(\bm z),\eta_t\right)
- \left(M(\bm z),0\right)
\right\Vert
\stackrel{p}{\to}0.
\]
By Lemma~\ref{lem:adapted-continuity-equilibria}, whose assumptions
(A1)--(A4) have been verified, the equilibrium mappings $X^\ast(\cdot,\cdot)$ and
$Y^\ast(\cdot,\cdot)$ are upper semi-continuous.
Consequently, we get 
$$
\operatorname{dist} \left( (\widehat{\phi}_t(\bm{z_t}), \widehat{\nu}_t(\bm{z_t})),\mathcal{E}(M(\bm{z}_t) \right) \stackrel{p}{\to} 0.
$$
\end{proof}

\subsection{Convergence of the Contextual Behavior Policy (Theorem~\ref{conthm:behave_con})}
\label{app:conthm:nash_behave_proof}

Similar to the arguments in Section~\ref{app:nash_behave_proof}, the proof of Theorem~\ref{conthm:behave_con} is divided into three parts, which establish the contextual exploration probabilities under $\epsilon$-Greedy, UCB, and TS, respectively.
The limiting behavior strategies under different exploration
strategies are given by:
\begin{enumerate}
    \item[\textnormal{(1)}] UCB and TS:
    $\widetilde{\phi}^{\ast}(\bm z) =
    \operatorname{Clip}_{m}
    \left(
    \phi^\ast(\bm z),
    \frac{\epsilon_\infty}{m}
    \right)$,
    $\widetilde{\nu}^{\ast}(\bm z) =
    \operatorname{Clip}_{k}
    \left(
    \nu^\ast(\bm z),
    \frac{\epsilon_\infty}{k}
    \right)$.
    \item[\textnormal{(2)}] $\epsilon$-Greedy:
    $\widetilde{\phi}^{\ast}(\bm z) =
    (1-\epsilon_\infty)\phi^\ast(\bm z)
    + \frac{\epsilon_\infty}{m}\bm 1_m$,
    $\widetilde{\nu}^{\ast}(\bm z) =
    (1-\epsilon_\infty)\nu^\ast(\bm z)
    + \frac{\epsilon_\infty}{k}\bm 1_k$.
\end{enumerate}
In particular, if $\epsilon_\infty=0$, then 
$\Vert
\widetilde{\phi}_t(\bm z_t) -
\phi^\ast(\bm z_t)
\Vert_2
\stackrel{p}{\to}
0$,
$\Vert
\widetilde{\nu}_t(\bm z_t) -
\nu^\ast(\bm z_t)
\Vert_2
\stackrel{p}{\to}
0$.

\subsubsection{Proof of $\epsilon$-Greedy}
By Theorem~\ref{conthm:nash_con} and Assumption~\ref{conass:unique}, we have $\Vert \widehat{\phi}_t(\bm z_t) - \phi^\ast(\bm z_t) \Vert_2 \stackrel{p}{\to} 0$, and $\Vert \widehat{\nu}_t(\bm z_t) - \nu^\ast(\bm z_t) \Vert_2 \stackrel{p}{\to} 0$.
Under the $\epsilon$-Greedy exploration strategy, the conditional behavior
policies are given by
\[
\widetilde{\phi}_t(\bm z_t)
=
(1-\epsilon_t)\widehat{\phi}_t(\bm z_t)
+
\frac{\epsilon_t}{m}\mathbf 1_m,
\quad
\widetilde{\nu}_t(\bm z_t)
=
(1-\epsilon_t)\widehat{\nu}_t(\bm z_t)
+
\frac{\epsilon_t}{k}\mathbf 1_k.
\]
With $\epsilon_t\to\epsilon_\infty$, the Continuous Mapping Theorem implies
\[
\Vert
\widetilde{\phi}_t(\bm z_t) -
\widetilde{\phi}^{\ast}(\bm z_t)
\Vert_2
\stackrel{p}{\to} 0,
\quad
\Vert
\widetilde{\nu}_t(\bm z_t) -
\widetilde{\nu}^{\ast}(\bm z_t)
\Vert_2
\stackrel{p}{\to} 0,
\]
where
\[
\widetilde{\phi}^{\ast}(\bm z_t)
=
(1-\epsilon_\infty)\phi^\ast(\bm z_t)
+
\frac{\epsilon_\infty}{m}\mathbf 1_m,
\quad
\widetilde{\nu}^{\ast}(\bm z_t)
=
(1-\epsilon_\infty)\nu^\ast(\bm z_t)
+
\frac{\epsilon_\infty}{k}\mathbf 1_k.
\]

\subsubsection{Proof of UCB}
\label{app:connash_behave_proof_ucb}
\begin{proof}
For the contextual setting, the upper confidence bound for action pair $(i,j)$
at time $t$ is given by
\[
[\widehat M_{t-1}(\bm z_t)]_{ij} + c_t \widehat \sigma_{t-1}^{ij}(\bm z_t),
\]
where $[\widehat M_{t-1}(\bm z_t)]_{ij} = \bm z_t^\top \widehat{\bm\beta}_{t-1}^{ij}$
and
\[
\widehat \sigma_{t-1}^{ij}(\bm z_t) =
\sqrt{\bm z_t^\top \left(\sum_{s=1}^{t-1}
\mathbb I(i_s=i,j_s=j)\bm z_s\bm z_s^\top\right)^{-1} \bm z_t} = \sqrt{\bm z_t^\top \left(\sum_{s=1}^{t-1} a_{s}^{ij}\bm z_s\bm z_s^\top \right)^{-1} \bm z_t}.
\]
$c_t>0$ is a parameter controlling the exploration level. As what we showed in Section~\ref{app:conthm:nash_con_proof}, it remains to show that the exploration bonus $c_t \widehat \sigma_{t-1}^{ij}(\bm z_t)$ vanishes in probability.
By~\eqref{equ:context-design-event}, for every action pair $(i,j)$,
\[
\mathbb P\left( \lambda_{\min} \left(
\sum_{s=1}^t a_s^{ij}\bm z_s\bm z_s^\top
\right) \le \frac{ t\epsilon_t^2\lambda }{ 2mk } \right) \le d \exp\left\{ - \frac{ t\epsilon_t^2\lambda }{ 8mkL_z^2 } \right\}.
\]
When $\lambda_{\min}(\sum_{s=1}^{t-1} a_{s}^{ij}\bm z_s\bm z_s^\top)
\ge (t-1) \epsilon_{t-1}^2\lambda/2mk$, there is
\[
\lambda_{\max}\left\{
\left(\sum_{s=1}^{t-1} a_{s}^{ij}\bm z_s\bm z_s^\top\right)^{-1}
\right\}
=
\frac{1}{\lambda_{\min}(\sum_{s=1}^{t-1} a_{s}^{ij}\bm z_s\bm z_s^\top)}
\le
\frac{2mk}{(t-1) \epsilon_{t-1}^2 \lambda}.
\]
Therefore, we get
\begin{align*}
c_t
\sqrt{
\bm z_t^\top
\left(\sum_{s=1}^{t-1} a_{s}^{ij}\bm z_s\bm z_s^\top\right)^{-1}
\bm z_t
}
&\le
c_t
\sqrt{
\|\bm z_t\|_2^2
\,\lambda_{\max}\left\{
\left(\sum_{s=1}^{t-1} a_{s}^{ij}\bm z_s\bm z_s^\top\right)^{-1}
\right\}
} \\
&\le
c_t
\sqrt{
dL_z^2
\frac{2mk}{(t-1) \epsilon_{t-1}^2\lambda}
} \\
&=
c_t L_z
\sqrt{
\frac{2dmk}{(t-1) \epsilon_{t-1}^2\lambda}} \to 0,
\end{align*}
since $t \epsilon_t^2\to\infty$.
Similar to the arguments in Section~\ref{app:nash_behave_proof_ucb}, we conclude that
\begin{equation}
\label{equ4}
c_t \sqrt{
\bm z_t^\top
\left(\sum_{s=1}^{t-1} a_{s}^{ij}\bm z_s\bm z_s^\top\right)^{-1}
\bm z_t } \stackrel{p}{\to} 0.
\end{equation}
Thus we obtain $[\widehat M_{t-1}(\bm z_t)]_{ij} + c_t \widehat \sigma_{t-1}^{ij}(\bm z_t)-[M(\bm z_t)]_{ij} \stackrel{p}{\to} 0.$
Again, by arguments analogous to those in Section~\ref{app:nash_con_proof} and Section~\ref{app:nash_behave_proof_ucb}, the clipped behavior policy under UCB converges to the projection of the true contextual equilibrium-induced joint action distribution onto the simplices with lower-bound constraints $\epsilon_\infty/m$ and $\epsilon_\infty/k$.
\end{proof}

\subsubsection{Proof of TS}
\begin{proof}
Under TS, we draw
\[
\widetilde{\bm\beta}_{t}^{ij}
\mid \mathcal H_{t-1}
\sim
\mathcal N\left(
\widehat{\bm\beta}_{t-1}^{ij},
\rho^2 \left(\sum_{s=1}^{t-1} \mathbb I(i_s=i,j_s=j)\bm z_s\bm z_s^\top \right)^{-1}
\right),
\]
Equivalently, we may write
\[
\widetilde{\bm\beta}_{t}^{ij}
= \widehat{\bm\beta}_{t-1}^{ij} +
\rho
\left(\sum_{s=1}^{t-1} a_s^{ij}\bm z_s\bm z_s^\top\right)^{-1/2}\bm\varsigma_t^{ij},
\]
where $\bm\varsigma_t^{ij}\sim \mathcal N_d(0,I_d)$. It remains to show that the posterior sampling error vanishes in probability.
Similar to the arguments to show~\eqref{equ4} in Section~\ref{app:nash_behave_proof_ucb}, we conclude that
\begin{equation*}
\rho \sqrt{
\bm z_t^\top
\left(\sum_{s=1}^{t-1} a_{s}^{ij}\bm z_s\bm z_s^\top\right)^{-1}
\bm z_t } \stackrel{p}{\to} 0.    
\end{equation*}
Since $\|\bm\varsigma_t^{ij}\|_2=O_p(1)$,
it follows that
\[
\rho
\left(\sum_{s=1}^{t-1} a_{s}^{ij} \bm z_s\bm z_s^\top \right)^{-1/2}\bm\varsigma_t^{ij}
\stackrel{p}{\to}
0.
\]
Hence, $\widetilde{\bm\beta}_{t}^{ij} \stackrel{p}{\to} \bm\beta^{ij}$.
Thus we obtain $\bm z_t^\top \widetilde{\bm\beta}_{t}^{ij}-[M(\bm z_t)]_{ij} \stackrel{p}{\to} 0.$
By arguments analogous to those in Section~\ref{app:nash_behave_proof_ucb} and Section~\ref{app:nash_behave_proof_ucb}, the clipped behavior policy under TS converges to the projection of the true contextual equilibrium-induced joint action distribution onto the simplices with lower-bound constraints $\epsilon_\infty/m$ and $\epsilon_\infty/k$.
\end{proof}

\subsection{Asymptotic Normality of Online Estimator (Theorem \ref{conthm:para})}
\label{app:conthm:para_proof}

\begin{proof}
The estimator $\widehat{\bm{\beta}}_t^{ij}$ is given by,
\begin{align*}
    \widehat{\bm{\beta}}_t^{ij} &= \left(\frac{1}{t} \sum_{s=1}^{t} \mathbb{I}\left\{i_s=i,j_s=j \right\} \bm{z}_s \bm{z}_s^{\top}\right)^{-1} \left( \frac{1}{t} \sum_{s=1}^t \mathbb{I}\left\{i_s=i,j_s=j \right\} \bm{z}_s r_s \right) \\
    &= \left(\frac{1}{t} \sum_{s=1}^{t} a_s^{ij} \bm{z}_s \bm{z}_s^{\top}\right)^{-1} \left( \frac{1}{t} \sum_{s=1}^t a_s^{ij} \bm{z}_s r_s \right).
\end{align*}
Then we have
\begin{equation*}
    \sqrt{t}\left(\widehat{\bm{\beta}}_t^{ij} - \bm{\beta}^{ij}\right) = \left(\frac{1}{t} \sum_{s=1}^{t} a_s^{ij} \bm{z}_s \bm{z}_s^{\top}\right)^{-1} \left( \frac{1}{\sqrt{t}} \sum_{s=1}^t a_s^{ij} \bm{z}_s e_s \right).
\end{equation*}
1. First we want to show 
\[
\frac{1}{\sqrt{t}} \sum_{s=1}^t a_s^{ij} \bm{z}_s e_s \xrightarrow{d} \mathcal{N}_d(0, G_{ij}).
\]
Using Cramer-Wold device, it suffices to show for any $\bm{v} \in \mathbb{R}^d$,
\[
\frac{1}{\sqrt{t}} \sum_{s=1}^t a_s^{ij} \bm{v}^\top \bm{z}_s e_s \xrightarrow{d} \mathcal{N}(0, \eta^2 = \bm{v}^\top G_{ij} \bm{v}).
\]
For $1 \le k \le t$ and $t \ge 1$, define $\mathcal{H}^{\mathrm{con}}_{tk} = \mathcal{H}^{\mathrm{con}}_k$ and
\[
\mathcal M^{\mathrm{con}}_{tk} = \frac{1}{\sqrt{t}} \sum_{s=1}^k a_s^{ij} \bm{v}^\top \bm{z}_s e_s.
\]
Notice that $\mathbb{E}(a_s^{ij} \bm{v}^\top \bm{z}_s e_s \mid \mathcal{H}^{\mathrm{con}}_{t,s-1}) = 0$, $\{\mathcal M^{\mathrm{con}}_{tk}, \mathcal{H}^{\mathrm{con}}_{tk}, 1 \le k \le t, t \ge 1\}$ is a martingale array. 

(a) Check the conditional Lindeberg condition. For any $\delta > 0$,
\begin{align*}
&\sum_{s=1}^t \mathbb{E} \left[ \frac{1}{t} a_s^{ij} (\bm{v}^\top \bm{z}_s)^2 e_s^2 \mathbb{I} \left\{ |a_s^{ij} \bm{v}^\top \bm{z}_s e_s| > \delta \sqrt{t} \right\} \middle| \mathcal{H}^{\mathrm{con}}_{t, s-1} \right] \\
\le & \frac{\|\bm{v}\|_2^2 L_z^2 d}{t} \sum_{s=1}^t \mathbb{E} \left( a_s^{ij} e_s^2 \mathbb{I} \left\{ a_s^{ij} e_s^2 > \frac{\delta^2 t}{\|\bm{v}\|_2^2 L_z^2 d} \right\} \middle| \mathcal{H}^{\mathrm{con}}_{s-1} \right) \\
= & \frac{\|\bm{v}\|_2^2 L_z^2 d}{t} \sum_{s=1}^t P(a_s^{ij} = 1 \mid \mathcal{H}^{\mathrm{con}}_{s-1}) \mathbb{E} \left[ e_s^2 \mathbb{I} \left\{e_s^2 > \frac{\delta^2 t}{\|\bm{v}\|_2^2 L_z^2 d} \right\} \middle| \mathcal{H}^{\mathrm{con}}_{s-1} \right] \\
\le & \|\bm{v}\|_2^2 L_z^2 d \mathbb{E} \left[ e_{s(i,j)}^2 \mathbb{I} \left\{ e_{s(i,j)}^2 > \frac{\delta^2 t}{\|\bm{v}\|_2^2 L_z^2 d} \right\} \right],
\end{align*}
where $e_{s(i,j)}=e_s$ if $i_s=i,j_s=j$ and 0 otherwise. $e_s$ conditioned on $i_s,j_s$ are $i.i.d.$ from sub-Gaussian distribution. Since $e_{s(i,j)}^2 \mathbb I \left\{ e_{s(i,j)}^2 > \frac{\delta^2 t}{\|\bm{v}\|_2^2 L_z^2 d} \right\}$ is dominated by $e_{s(i,j)}^2$ with $\mathbb{E} e_{s(i,j)}^2 \le \sigma_{ij}^2$ and it converges to 0 almost surely when \( t \to \infty \). Therefore we can apply Dominated Convergence Theorem and get 
\[
\sum_{s=1}^t \mathbb{E} \left[ \frac{1}{t} a_s^{ij} (\bm{v}^\top \bm{z}_s)^2 e_s^2 \, \mathbb I \left\{ |a_s^{ij} \bm{v}^\top \bm{z}_s e_s| > \delta \sqrt{t} \right\} \middle| \mathcal{H}^{\mathrm{con}}_{t,s-1} \right] \to 0,
\]
as $t \to \infty$.

(b) Find the limit of the conditional variance. The conditional variance is
\begin{align*}
    \widehat{\gamma}_t^2 &\coloneqq \sum_{s=1}^t \mathbb{E} \left[ \frac{1}{t} a_s^{ij} (\bm{v}^\top \bm{z}_s)^2 e_s^2 \middle| \mathcal{H}^{\mathrm{con}}_{t,s-1} \right]
    = \frac{\sigma_{ij}^2}{t} \sum_{s=1}^t \mathbb{E} \left[a_s^{ij} \bm{v}^\top \bm{z}_s \bm{z}_s^\top \bm{v}  | \mathcal{H}^{\mathrm{con}}_{s-1} \right].
\end{align*}
By the law of iterated expectations over $\bm{z}_s$, we have
\begin{align*}
    \widehat{\gamma}_t^2 & = \frac{\sigma_{ij}^2}{t} \sum_{s=1}^t \mathbb{E}\left[ \mathbb{P}(i_s=i,j_s=j) \bm{v}^\top \bm{z}_s \bm{z}_s^\top \bm{v}  | \mathcal{H}^{\mathrm{con}}_{s-1} \right] 
    = \frac{\sigma_{ij}^2}{t} \sum_{s=1}^t \mathbb{E}\left[[\widetilde{\phi}_s(\bm{z}_s)]_i [\widetilde{\nu}_s(\bm{z}_s)]_j \bm{v}^\top \bm{z}_s \bm{z}_s^\top \bm{v}  | \mathcal{H}^{\mathrm{con}}_{s-1} \right].
\end{align*}
By Theorem~\ref{conthm:behave_con}, we have $\Vert \widetilde{\phi}_s(\bm{z}_s) - \widetilde \phi^\ast(\bm{z}_s) \Vert_2 \stackrel{p}{\to} 0$, and $\Vert \widetilde{\nu}_s(\bm{z}_s) - \widetilde \nu^\ast(\bm{z}_s) \Vert_2 \stackrel{p}{\to} 0$. Since the last equation is bounded by $d\Vert \bm{v} \Vert_2^2 L_z^2$, by Dominated Convergence Theorem and Continuous Mapping Theorem, we will get
\begin{align*}
    \mathbb{E}\left[[\widetilde{\phi}_s(\bm{z}_s)]_i [\widetilde{\nu}_s(\bm{z}_s)]_j \bm{v}^\top \bm{z}_s \bm{z}_s^\top \bm{v}  | \mathcal{H}^{\mathrm{con}}_{s-1} \right] \stackrel{p}{\to} \, \bm{v}^\top \mathbb{E}\left[ [\widetilde \phi^\ast(\bm{z}_s)]_i [\widetilde \nu^\ast(\bm{z}_s)]_j \bm{z}_s \bm{z}_s^\top | \mathcal{H}^{\mathrm{con}}_{s-1} \right] \bm{v} .
\end{align*}
Hence, we obtain 
\begin{align*}
    \widehat{\gamma}_t^2 & \stackrel{p}{\to} \sigma_{ij}^2 \bm{v}^\top \mathbb{E}\left[ [\widetilde \phi^\ast(\bm{z}_s)]_i [\widetilde \nu^\ast(\bm{z}_s)]_j \right] \bm{v} 
    = \sigma_{ij}^2 \bm{v}^\top \left\{ \int [\widetilde \phi^\ast(\bm{z}_s)]_i [\widetilde \nu^\ast(\bm{z}_s)]_j \bm{z} \bm{z}^\top \,d\mathcal{P}_z \right\}\bm{v} \equiv \gamma^2
\end{align*}
It follows by Martingale Central Limit Theorem that
\[
\frac{1}{\sqrt{t}} \sum_{s=1}^t a_s^{ij} \bm{z}_s e_s \xrightarrow{d} \mathcal{N}_d(0, G_{ij}),
\]
with
\[
G_{ij} = \sigma_{ij}^2 \int [\widetilde \phi^\ast(\bm{z}_s)]_i [\widetilde \nu^\ast(\bm{z}_s)]_j \bm{z} \bm{z}^\top \,d\mathcal{P}_z.
\]
2. Next, we want to find the limit of the second moment term. Using Lemma~\ref{lem:matrix-vector}, we first find the limit of 
\[
\widehat{\psi}_t \coloneqq \frac{1}{t} \sum_{s=1}^{t} a_s^{ij} \bm{v}^\top \bm{z}_s \bm{z}_s^{\top} \bm{v},
\]
for any $\bm{v} \in \mathbb{R}^d$. Let $\bm{z} \sim \mathcal{P}_z$, we have $\mathbb{E}(\bm{v}^T \bm{z}\bm{z}^T \bm{v}) \leq \bm{v}^T 11^T \bm{v} L_z^2 < \infty$. Also, for any $h \geq 0$ and each $s$,  
$\mathbb{P}(a_s^{ij} \bm{v}^T \bm{z}_s \bm{z}_s^T \bm{v} > h) \leq \mathbb{P}(\bm{v}^T \bm{z}\bm{z}^T\bm{v} > h).$ Then by Theorem 2.19 from \citet{hall1980martingale},  
\begin{equation*}
    \frac{1}{t} \sum_{s=1}^t \left\{ a_s^{ij} \bm{v}^T \bm{z}_s \bm{z}_s^T \bm{v} - \mathbb{E} \left( a_s^{ij} \bm{v}^T \bm{z}_s \bm{z}_s^T \bm{v} | \mathcal{H}^{\mathrm{con}}_{s-1} \right) \right\} = \widehat{\psi}_t - \frac{1}{\sigma_{ij}^2} \widehat{\gamma}_t^2 \overset{p}{\to} 0
\end{equation*}
Based on our results in the first part, we know that $\widehat{\psi}_t \overset{p}{\to} \gamma^2 / \sigma_{ij}^2$. By Lemma~\ref{lem:matrix-vector} and Continuous Mapping Theorem,
\[
\left(\frac{1}{t} \sum_{s=1}^{t} a_s^{ij} \bm{z}_s \bm{z}_s^{\top} \right)^{-1} \overset{p}{\to} \sigma_{ij}^2G_{ij}^{-1}.
\]
Finally, combine above results using Slutsky Theorem, we have 
\[
\sqrt{t}\left(\widehat{\bm{\beta}}_t^{ij} - \bm{\beta}^{ij}\right) \overset{D}{\to} \mathcal{N}_d \left( 0, \sigma_{ij}^2 \left\{ \int [\widetilde \phi^\ast(\bm{z})]_i [\widetilde \nu^\ast(\bm{z})]_j \bm{z} \bm{z}^\top \,d\mathcal{P}_z \right\}^{-1} \right).
\]
The consistency of the variance estimator needs to show
\[
\frac{\sum_{s=1}^{t} a_s^{ij}\widehat{e}_s^2}{\sum_{s=1}^{t} a_s^{ij}} \left( \frac{1}{t} \sum_{s=1}^{t} a_s^{ij} \bm{z}_s \bm{z}_s^T \right)^{-1} \overset{P}{\rightarrow} S_{ij}^2.
\]
Using the results we already have, it suffices to prove
\[
\widehat{\sigma}_{ij}^2 := \frac{\sum_{s=1}^{t} a_s^{ij}\widehat{e}_s^2}{\sum_{s=1}^{t} a_s^{ij}} = \frac{\sum_{s=1}^{t} a_s^{ij} \left\{(\bm{\beta}^{ij} - \widehat{\bm{\beta}}_t^{ij})^\top \bm{z}_s + e_s \right\}^2}{\sum_{s=1}^{t} a_s^{ij}} \overset{P}{\rightarrow} \sigma_{ij}^2.
\]
There are three terms in $\widehat{\sigma}_{ij}^2$ after expanding the square term, 
\[
\widehat{\sigma}_{ij}^2 = \frac{\sum_{s=1}^{t} a_s^{ij} \left[\left((\bm{\beta}^{ij} - \widehat{\bm{\beta}}_t^{ij})^\top \bm{z}_s\right)^2+2(\bm{\beta}^{ij} - \widehat{\bm{\beta}}_t^{ij})^\top \bm{z}_se_s+e_s^2\right]}{\sum_{s=1}^{t} a_s^{ij}}.
\]
The first term 
\[
(\bm{\beta}^{ij} - \widehat{\bm{\beta}}_t^{ij})^\top \frac{\sum_{s=1}^t a_s^{ij} \bm{z}_s \bm{z}_s^T}{\sum_{s=1}^t a_s^{ij}} (\bm{\beta}^{ij} - \widehat{\bm{\beta}}_t^{ij}) \leq L_z^2 \| (\bm{\beta}^{ij} - \widehat{\bm{\beta}}_t^{ij})\|_1^2 \overset{P}{\to} 0,
\]
by Corollary \ref{concoro:consistent}. Define $e_{s(i,j)}, s = 1, \dots, t$ to be i.i.d. random variables, then the second term 
\[
2(\bm{\beta}^{ij} - \widehat{\bm{\beta}}_t^{ij})^\top \frac{\sum_{s=1}^t a_s^{ij} \bm{z}_s e_{s(i,j)}}{\sum_{s=1}^t a_s^{ij}} \overset{P}{\to} 0,
\]
since $\widehat{\bm{\beta}}_t^{ij} \stackrel{p}{\to}\bm{\beta}^{ij}$ by Corollary \ref{concoro:consistent}, and $(\sum_{s=1}^t a_s^{ij})^{-1} \sum_{s=1}^t a_s^{ij} \bm{z}_s e_{s(i,j)} \overset{P}{\to} 0$ by Lemma~\ref{lemma1_in_chen}. Finally, the last term
\[
\frac{\sum_{s=1}^t a_s^{ij}(e_{s(i,j)})^2}{\sum_{s=1}^t a_s^{ij}} \overset{P}{\to} \mathbb{E}(e_{s(i,j)})^2 = \sigma_{ij}^2
\]
is by Weak Law of Large Numbers. Combining all three terms with Continuous Mapping Theorem gives the result.
\end{proof}

\subsection{Convergence Rate of Saddle Value (Corollary~\ref{conthm:NEvalue})}
\label{app:conthm:NEvalue_proof}

\begin{proof}
Following the proof in Section \ref{app:saddlevalue_rate}, we adopt the same function 
\[
g(\bm{x},\bm{y},\bm{t})=g(\bm{x},\bm{y},(\widehat{M}_{t-1}(\bm{z}_t),\eta_t)) = {\mathcal{L}}_t^{\mathrm{con}}(\bm{x},\bm{y})
\]
as defined previously, where in the contextual setting we substitute $A$ with $M(\bm{z}_t)$ or $\widehat{M}_{t-1}(\bm{z}_t)$. Let $\bm{t}_1 = (\widehat{M}_{t-1}(\bm{z}_t),\eta_t)$ and $\bm{t}_2 = (M(\bm{z}_t),0)$.

Applying analogous arguments yields $\vert {\mathcal{L}}_t^{\mathrm{con}}(\widehat{\phi}_t(\bm{z}_t),\widehat{\nu}_t(\bm{z}_t)) - \phi^\ast(\bm{z}_t)^\top M(\bm{z}_t) \nu^\ast(\bm{z}_t) \vert \le C_g^\prime \Vert \bm{t}_1 - \bm{t}_2 \Vert_2$ for some constant $C_g^\prime >0$, where
\begin{align*}
    \Vert \bm{t}_1 - \bm{t}_2 \Vert_2 &= \sqrt{\sum_{i=1}^m\sum_{j=1}^k \left([\widehat{M}_{t-1}(\bm{z_t})]_{ij}-[M(\bm{z}_t)]_{ij}\right)^2 + \eta_t^2} \\
    &= \sqrt{\sum_{i=1}^m\sum_{j=1}^k \left(\bm{z}_t^\top(\widehat{\bm{\beta}}_{t-1}^{ij}-\bm{\beta}^{ij}) \right)^2 + \eta_t^2}\\
    & \le \sqrt{L_z^2\sum_{i=1}^m\sum_{j=1}^k \Vert\widehat{\bm{\beta}}_{t-1}^{ij}-\bm{\beta}^{ij}\Vert_2^2 + \eta_t^2}\\
    &= O_p\left((t^{-1}+\eta_t^2)^{1/2}\right).
\end{align*}
Here, the final rate follows from Theorem \ref{conthm:para}, which establishes $\Vert \widehat{\bm{\beta}}_{t-1}^{ij}-\bm{\beta}^{ij}\Vert_2 = O_p(t^{-1/2})$ through the asymptotic distribution of the estimators. Consequently, we obtain the desired convergence rate for the saddle value
\[
\vert {\mathcal{L}}_t^{\mathrm{con}}(\widehat{\phi}_t(\bm{z}_t),\widehat{\nu}_t(\bm{z}_t)) - \phi^\ast(\bm{z}_t)^\top M(\bm{z}_t) \nu^\ast(\bm{z}_t) \vert = O_p\left((t^{-1}+\eta_t^2)^{1/2}\right).
\]
\end{proof}

\subsection{$\sqrt{T}$-consistency of Contextual Saddle Value Estimator (Theorem \ref{conthm:drrate})}
\label{app:conthm:drrate_proof}

Here we first state the formal scenario of Assumption~\ref{conass:margin}.
\begin{assumption}[{\emph{Uniform Margin: Formal Statement}}]
\label{conass:margin_formal}
For any $\bm{z} \sim \mathcal{P}_z$, one of the following conditions holds:
\begin{enumerate}
\item[\textnormal{(i)}]
There exist a constant $C_l^\prime>0$ such that for all $\bm{z}$, $i \in [m]$ and $j \in [k]$, $\phi^\ast(\bm{z}) > C_l^\prime$ and $\nu^\ast(\bm{z})>C_l^\prime$.
\item[\textnormal{(ii)}]
Let $(\bar{\bm x}^\ast,\bar{\bm y}^\ast)$ denote the Nash equilibrium of the regularized contextual game associated with any payoff matrix $M^\prime$ and regularization parameter $\eta\ge 0$. There exists a constant $\delta>0$ such that, whenever
\[
\sup_{\bm z\in\mathcal Z}
\Vert
\left(M^\prime,\eta\right) -
\left(M(\bm z),0\right)
\Vert_2
< \delta,
\]
we have
\[
[\phi^\ast(\bm z)]_{i_0}=0
\ \to\
\bar{\bm x}^\ast_{i_0}=0,
\qquad
[\nu^\ast(\bm z)]_{j_0}=0
\ \to\
\bar{\bm y}^\ast_{j_0}=0,
\]
for all $\bm z\in\mathcal Z$, $i_0\in[m]$, and $j_0\in[k]$.
\end{enumerate}
\end{assumption}

\begin{proof}

The proof is similar as those in Section \ref{app:dr}. We need several steps to prove the square root rate of convergence. We define 
\begin{align*}
    \bar{V}_T &\equiv \frac{1}{T}\sum_{t=1}^T \widehat{\phi}_t(\bm{z}_t)^\top \widehat{M}_{t-1}(\bm{z}_t)\widehat{\nu}_t(\bm{z}_t) + \frac{[\widehat{\phi}_t(\bm{z}_t)]_{i_t}[\widehat{\nu}_t(\bm{z}_t)]_{j_t}}{{\pi}_t^{\mathrm{con}}(i_t,j_t \vert \bm{z}_t)}(r_t - [M(\bm{z}_t)]_{i_t,j_t}), \\
\end{align*}

\textbf{Step 1:} We first show $\widehat{V}_T^\text{conDR} - \bar{V}_T = O_p(T^{-1/2})$.

\begin{align*} 
       \widehat{V}_T^\text{conDR} - \bar{V}_T &= \frac{1}{T}\sum_{t=1}^T \frac{[\widehat{\phi}_t(\bm{z}_t)]_{i_t}[\widehat{\nu}_t(\bm{z}_t)]_{j_t}}{{\pi}_t^{\mathrm{con}}(i_t,j_t \vert \bm{z}_t)} \left[M(\bm{z}_t) - [\widehat{M}_{t-1}(\bm{z}_t)\right]_{i_t,j_t} \\
    &= \frac{1}{T}\sum_{t=1}^T \frac{[\widehat{\phi}_t(\bm{z}_t)]_{i_t}[\widehat{\nu}_t(\bm{z}_t)]_{j_t}}{[\widetilde{\phi}_t(\bm{z}_t)]_{i_t}[\widetilde{\nu}_t(\bm{z}_t)]_{j_t}} \bm{z}_t^\top \left(\widehat{\bm{\beta}}_{t-1}^{ij}-\bm{\beta}^{ij}  \right) \\
    &\le L_z \frac{1}{T}\sum_{t=1}^T \frac{[\widehat{\phi}_t(\bm{z}_t)]_{i_t}[\widehat{\nu}_t(\bm{z}_t)]_{j_t}}{[\widetilde{\phi}_t(\bm{z}_t)]_{i_t}[\widetilde{\nu}_t(\bm{z}_t)]_{j_t}} \left\Vert \left(\widehat{\bm{\beta}}_{t-1}^{ij}-\bm{\beta}^{ij} \right) \right\Vert_2.
\end{align*}

We now consider the three exploration strategies separately. 
\paragraph{$\epsilon$-Greedy.}
Similar as the arguments what we used in Section~\ref{app:dr}, since $\epsilon_t\to\epsilon_\infty<1$, there exist constants $C_\epsilon$ and $T_0<\infty$ such that for all $t\ge T_0$, there is
\[
\frac{
[\widehat{\phi}(\bm z_t)]_i
[\widehat{\nu}(\bm z_t)]_j
}{\pi_t^\mathrm{con}(i,j | \bm z_t)
}
\le C_\epsilon^\prime
\]
uniformly over all action pairs $(i,j)$ and all sufficiently large $t$.

\paragraph{UCB and TS.}
We next consider UCB and TS. Analogous to the arguments in Section~\ref{conthm:behave_con}, the corresponding contextual payoff matrices satisfy
\[
\sup_{\bm z \sim P_z}
\Vert
\bm z ^\top \widetilde{\bm{\beta}}_t^{ij}-M(\bm z) 
\Vert_F \le L_z 
\Vert
\widetilde{\bm{\beta}}_t^{ij}- {\bm{\beta}}^{ij}
\Vert_2
\stackrel{\mathrm{a.s.}}{\to}
0,
\]
where $\widetilde{\bm{\beta}}_t^{ij}$ is the estimated parameter  for the behavior strategies under UCB or TS. Since $\eta_t\to0$, the equilibrium mapping is continuous, and the clipping operator is jointly continuous,
\[
\sup_{\bm z\in\mathcal Z}
\left\|
\widetilde{\phi}_t(\bm z)
-
\operatorname{Clip}_m
\left(
\phi^\ast(\bm z),
\frac{\epsilon_\infty}{m}
\right)
\right\|_2
\stackrel{\mathrm{a.s.}}{\to}
0,
\]
and
\[
\sup_{\bm z\in\mathcal Z}
\left\|
\widetilde{\nu}_t(\bm z)
-
\operatorname{Clip}_k
\left(
\nu^\ast(\bm z),
\frac{\epsilon_\infty}{k}
\right)
\right\|_2
\stackrel{\mathrm{a.s.}}{\to}
0.
\]

Consequently, under either condition of Assumption~\ref{conass:margin_formal}, there exist a finite constant $C_\epsilon^\prime$ and an almost-surely finite time $T_0$ such that
\[
\frac{
[\widehat{\phi}(\bm z_t)]_i
[\widehat{\nu}(\bm z_t)]_j
}{
\pi_t^\mathrm{con}(i,j\mid\bm z_t)
}
\le C_\epsilon^\prime
\]
uniformly over all $i\in[m]$, $j\in[k]$, and $t\ge T_0$.
Thus we get $$\widehat{V}_T^\text{conDR} - \bar{V}_T =  O_p\left(\sum_{t=1}^T \left(\widehat{\bm{\beta}}_{t-1}^{ij}-\bm{\beta}^{ij} \right)\right)=O_p(T^{-1/2})$$ by Theorem \ref{conthm:para}.

\textbf{Step 2:} Then we will show $\bar{V}_T - V^\ast = O_p(T^{-1/2})$.

\begin{equation}
\label{barV_T - V_star}
    \bar{V}_T - V^\ast = \frac{1}{T}\sum_{t=1}^T \widehat{\phi}_t(\bm{z}_t)^\top \widehat{M}_{t-1}(\bm{z}_t)\widehat{\nu}_t(\bm{z}_t) - \mathbb{E} \left\{\phi^\ast(\bm{z})^\top M(\bm{z})\nu^\ast(\bm{z})\right\} + \frac{[\widehat{\phi}_t(\bm{z}_t)]_{i_t}[\widehat{\nu}_t(\bm{z}_t)]_{j_t}}{{\pi}_t^{\mathrm{con}}(i_t,j_t \vert \bm{z}_t)}(r_t - [M(\bm{z}_t)]_{i_t,j_t}).
\end{equation}
We will decompose the Equation \eqref{barV_T - V_star} into two parts, where
\begin{equation*}
    J_1 \coloneqq \frac{1}{T}\sum_{t=1}^T \widehat{\phi}_t(\bm{z}_t)^\top \widehat{M}_{t-1}(\bm{z}_t)\widehat{\nu}_t(\bm{z}_t) - \mathbb{E} \left\{\phi^\ast(\bm{z})^\top M(\bm{z})\nu^\ast(\bm{z})\right\},
\end{equation*}
and
\begin{equation*}
    J_2 \coloneqq \frac{1}{T}\sum_{t=1}^T \frac{[\widehat{\phi}_t(\bm{z}_t)]_{i_t}[\widehat{\nu}_t(\bm{z}_t)]_{j_t}}{{\pi}_t^{\mathrm{con}}(i_t,j_t \vert \bm{z}_t)}(r_t - [M(\bm{z}_t)]_{i_t,j_t}).
\end{equation*}
We analyze $J_1$ first. Add and minus the same term, we have
\begin{align}
    J_1 =& \frac{1}{T}\sum_{t=1}^T \widehat{\phi}_t(\bm{z}_t)^\top \widehat{M}_{t-1}(\bm{z}_t)\widehat{\nu}_t(\bm{z}_t) - \phi^\ast(\bm{z}_t)^\top M(\bm{z}_t)\nu^\ast(\bm{z}_t) + \phi^\ast(\bm{z}_t)^\top M(\bm{z}_t)\nu^\ast(\bm{z}_t)-
    \mathbb{E} \left\{\phi^\ast(\bm{z})^\top M(\bm{z})\nu^\ast(\bm{z})\right\} \\
    =& \frac{1}{T}\sum_{t=1}^T \widehat{\phi}_t(\bm{z}_t)^\top \widehat{M}_{t-1}(\bm{z}_t)\widehat{\nu}_t(\bm{z}_t) - \phi^\ast(\bm{z}_t)^\top M(\bm{z}_t)\nu^\ast(\bm{z}_t) \label{equ:J3}\\
    +& \frac{1}{T}\sum_{t=1}^T \phi^\ast(\bm{z}_t)^\top M(\bm{z}_t)\nu^\ast(\bm{z}_t)-
    \mathbb{E} \left\{\phi^\ast(\bm{z})^\top M(\bm{z})\nu^\ast(\bm{z})\right\}\label{equ:J4}
\end{align}
By Corollary \ref{conthm:NEvalue} and $\eta_t = O_p(t^{-1/2})$, Equation \eqref{equ:J3} is bounded by $\frac{1}{T}\sum_{t=1}^TO_p(t^{-1/2})$, which becomes $O_p(T^{-1/2})$. Because $\phi^\ast(\bm{z}_1)^\top M(\bm{z}_1)\nu^\ast(\bm{z}_1), \dots, \phi^\ast(\bm{z}_T)^\top M(\bm{z}_T)\nu^\ast(\bm{z}_T)$ can be considered as $i.i.d$ samples, by Central Limit Theorem, Equation \eqref{equ:J4} is $O_p(T^{-1/2})$. Hence, we know that $J_1 = O_p(T^{-1/2})$.

Next, we consider $J_2$. The argument is analogous to the non-contextual case in Step 2 of Section~\ref{app:dr}. Here we define
\[
Z_t = \frac{
[\widehat{\phi}_t(\bm z_t)]_{i_t}
[\widehat{\nu}_t(\bm z_t)]_{j_t}
}{
\pi_t^\mathrm{con}(i_t,j_t\mid\bm z_t)
}
\left(
r_t-\bm z_t^\top\bm\beta^{i_t,j_t}
\right).
\]
Since
\[
\mathbb E\left(
r_t-\bm z_t^\top\bm\beta^{i_t,j_t}
\mid
\bm z_t,i_t,j_t,\mathcal H_{t-1}^{\mathrm{con}}
\right)
=
0,
\]
we have
\[
\mathbb E\left(
Z_t
\mid
\mathcal H_{t-1}^{\mathrm{con}}
\right)
=
0.
\]
Thus, $\{Z_t,\mathcal H_t^{\mathrm{con}}\}_{t\ge1}$ is a martingale
difference sequence. By the eventual uniform bound on the importance ratios and
the boundedness of the conditional noise variances, there exists a constant
$C_\epsilon^\prime<\infty$ such that
\[
\mathbb E\left(
Z_t^2
\mid
\mathcal H_{t-1}^{\mathrm{con}}
\right)
\le C_\epsilon^\prime
\]
almost surely for all sufficiently large $t$. Therefore,
\[
\frac{1}{T}
\sum_{t=1}^T Z_t
=
O_p(T^{-1/2}).
\]
Then we get $J_2 = O_p(T^{-1}) = o_p(T^{-1/2})$. With $J_1=O_p(T^{-1/2})$ shown before, Equation \eqref{barV_T - V_star} can be bounded by $O_p(T^{-1/2})$. Hence we have proved
\[
\bar{V}_T - V^\ast= O_p(T^{-1/2}).
\]
Therefore, $\widehat{V}_T^\text{conDR} - V^\ast = O_p(T^{-1/2})$. The proof of Theorem \ref{conthm:drrate} is completed.
\end{proof}

\section{Analysis of Contextual Matrix games under Model Misspecification}
\label{app:sec:mis_model}
Under model misspecification, we use the inverse-probability-weighted least-squares estimator
\begin{equation}
\label{equ:wls_context_misspecified}
\widehat{\bm\beta}_{\dagger,t}^{ij}
=
\left\{
\frac{1}{t}
\sum_{s=1}^t
\frac{
a_{\dagger,s}^{ij}
}{
\pi_{\dagger,s}^{ij}
}
\bm z_s\bm z_s^\top
\right\}^{-1}
\left\{
\frac{1}{t}
\sum_{s=1}^t
\frac{
a_{\dagger,s}^{ij}
}{
\pi_{\dagger,s}^{ij}
}
\bm z_sr_s
\right\}.
\end{equation}
We impose the following regularity conditions.

\begin{assumption}[\emph{Bounded linear approximation error}]
\label{ass:context_misspecified}
The approximation error is uniformly bounded:
\[
\left|
\varphi^{ij}(\bm z) - \bm z^\top\bm\beta_\dagger^{ij}
\right|
\le
L_\varphi
\]
for any $\bm z \sim P_z$.
\end{assumption}

\begin{theorem}[\textbf{\emph{Tail bound for the contextual WLS estimator under model misspecification}}]
\label{conthm:wls_tail_misspecified}
Suppose that Assumption~\ref{ass:bound} and Assumption~\ref{ass:context_misspecified} hold, we have
\begin{align*}
\mathbb P
\left(
\left\|
\widehat{\bm\beta}_{\dagger,t}^{ij}
-
\bm\beta_\dagger^{ij}
\right\|_1
>h
\right)
&\le
d\exp
\left\{
-
\frac{t\epsilon_t^2\lambda}{8mkL_z^2}
\right\}
+
2d\exp
\left\{
-
\frac{t\epsilon_t^2\lambda^2h^2}{32mkd^2
\left(L_z^2L_\varphi^2
+ {L_zL_\varphi\lambda h}/{12d}
\right)
}
\right\}
\\
&\quad+
2d
\exp
\left[
-\frac{t\epsilon_t^2}{4mk}
\min
\left\{
\frac{
\lambda^2h^2}{16d^2L_z^2\sigma_{ij}^2},
\frac{
\lambda h}{4dL_z\sigma_{ij}}
\right\}
\right].
\end{align*}
\end{theorem}

\begin{proof}
Define the weighted design matrix
\[
\widetilde\Sigma_t^{ij}
=
\frac1t
\sum_{s=1}^t
\frac{a_{\dagger,s}^{ij}}{\pi_{\dagger,s}^{ij}
}
\bm z_s\bm z_s^\top.
\]
By the definition of the contextual WLS estimator,
\begin{align*}
\widehat{\bm\beta}_{\dagger,t}^{ij}-\beta_\dagger^{ij} =
\left(
\widetilde\Sigma_t^{ij}
\right)^{-1}
\left[
\frac1t
\sum_{s=1}^t
\frac{a_{\dagger,s}^{ij}}{\pi_{\dagger,s}^{ij}}
\bm z_s
\left(\varphi^{ij}(\bm z_s) - \bm z_s^\top\bm\beta_\dagger^{ij} \right)
+
\frac1t
\sum_{s=1}^t
\frac{a_{\dagger,s}^{ij}}{\pi_{\dagger,s}^{ij}}
\bm z_s e_s
\right].
\end{align*}
We first control the minimum eigenvalue of the weighted design matrix.
Since $\|\bm z_s\|_2\le L_z$, we have
\[
0
\preceq
\frac{ a_{\dagger,s}^{ij} }{ \pi_{\dagger,s}^{ij} } \bm z_s\bm z_s^\top
\preceq
\frac{L_z^2}{\epsilon_t^2/mk}I_d.
\]
Moreover,
\begin{align*}
\mathbb E
\left[
\left.
\frac{ a_{\dagger,s}^{ij} }{ \pi_{\dagger,s}^{ij} } \bm z_s\bm z_s^\top
\right|
\mathcal H_{s-1}^{\mathrm{con}}
\right]
=\mathbb E
\left[
\left.
\mathbb E
\left(
\left.
\frac{a_{\dagger,s}^{ij}}{\pi_{\dagger,s}^{ij}}
\bm z_s\bm z_s^\top
\right|
\mathcal H_{s-1}^{\mathrm{con}},
\bm z_s
\right)
\right|
\mathcal H_{s-1}^{\mathrm{con}}
\right]
=\mathbb E
\left(\bm z_s\bm z_s^\top\right)
=\Sigma.
\end{align*}
Therefore,
\[
\lambda_{\min}
\left\{
\sum_{s=1}^t
\mathbb E
\left[
\left.
\frac{ a_{\dagger,s}^{ij} }{ \pi_{\dagger,s}^{ij} } \bm z_s\bm z_s^\top
\right|
\mathcal H_{s-1}^{\mathrm{con}}
\right]
\right\}
\ge
t\lambda.
\]
Applying Lemma~\ref{lem:adapted-matrix-chernoff} with $R={L_z^2}/{\epsilon_t^2/mk}$, $\mu_t=t\lambda$, and $\delta=1/2$, yields
\begin{equation}
\label{equ:wls_weighted_design_tail}
\mathbb P
\left(
\lambda_{\min}
\left(
\widetilde\Sigma_t^{ij}
\right)
\le
\frac{\lambda}{2}\right)
\le
d\exp
\left\{
- \frac{t\epsilon_t^2\lambda}{8mkL_z^2}
\right\}.
\end{equation}
Define the event
\[
\widetilde E_t^{ij}
=
\left\{
\lambda_{\min}
\left(
\widetilde\Sigma_t^{ij}
\right)
>
\frac{\lambda}{2}
\right\}.
\]
On this event,
\[
\left\|
\left(
\widetilde\Sigma_t^{ij}
\right)^{-1}
\right\|_2
= \lambda_{\max}
\left\{
\left(\widetilde\Sigma_t^{ij} \right)^{-1}
\right\}
\le
\frac{2}{\lambda}.
\]
Hence,
\begin{align*}
\left\|
\widehat{\bm\beta}_{\dagger,t}^{ij}
-
\beta_\dagger^{ij}
\right\|_1
\le
\frac{2\sqrt d}{\lambda}
\left\|
\frac1t
\sum_{s=1}^t
\frac{a_{\dagger,s}^{ij}}{\pi_{\dagger,s}^{ij}}
\bm z_s
\left(\varphi^{ij}(\bm z_s) - \bm z_s^\top\bm\beta_\dagger^{ij} \right)
+
\frac1t
\sum_{s=1}^t
\frac{a_{\dagger,s}^{ij}}{\pi_{\dagger,s}^{ij}}
\bm z_se_s
\right\|_2.
\end{align*}
Therefore, the event
\[
\left\{
\left\|
\widehat{\bm\beta}_{\dagger,t}^{ij}
-
\beta_\dagger^{ij}
\right\|_1
> h
\right\}
\cap
\widetilde E_t^{ij}
\]
implies that, for at least one coordinate $\ell\in[d]$, either
\[
\left|
\frac1t
\sum_{s=1}^t
\frac{a_{\dagger,s}^{ij}}{\pi_{\dagger,s}^{ij}}
z_{s,\ell}
\left(\varphi^{ij}(\bm z_s) - \bm z_s^\top\bm\beta_\dagger^{ij} \right)
\right|
>q_h
\]
or
\[
\left|
\frac1t
\sum_{s=1}^t
\frac{a_{\dagger,s}^{ij}}{\pi_{\dagger,s}^{ij}}
z_{s,\ell} e_{s(i,j)}
\right|
>q_h,
\]
where $q_h={\lambda h}/{4d}$.
Applying Lemma~\ref{lem:context_ipw_score_concentration} gives
\begin{align*}
\mathbb P
\left(
\left|
\frac1t
\sum_{s=1}^t
\frac{
a_{\dagger,s}^{ij}}{\pi_{\dagger,s}^{ij}}z_{s,\ell}
\left(\varphi^{ij}(\bm z_s) - \bm z_s^\top\bm\beta_\dagger^{ij} \right)
\right| >q_h
\right)
\le
2\exp
\left\{
-\frac{t\epsilon_t^2q_h^2}{2mk\left(L_z^2L_\varphi^2
+L_zL_\varphi q_h/3
\right)
}
\right\}.
\end{align*}
and
\begin{align*}
\mathbb P
\left(
\left|
\frac1t
\sum_{s=1}^t
\frac{a_{\dagger,s}^{ij}}{\pi_{\dagger,s}^{ij}}
z_{s,\ell} e_{s(i,j)}
\right| >q_h
\right)
\le
2\exp
\left[-\frac{t\epsilon_t^2}{4mk}
\min
\left\{
\frac{q_h^2}{L_z^2\sigma_{ij}^2},
\frac{q_h}{L_z\sigma_{ij}}
\right\}
\right].
\end{align*}

Applying a union bound over the $d$ coordinates and combining the result with
\eqref{equ:wls_weighted_design_tail} completes the proof.
\end{proof}

\begin{corollary}[\textbf{\emph{Consistency of the contextual WLS estimator under model misspecification}}]
\label{concoro:wls_consistency_misspecified}
Suppose that the conditions of
Theorem~\ref{conthm:wls_tail_misspecified} hold. If $t\epsilon_t^2\to\infty$ as $t \to \infty$ then $\widehat{\bm\beta}_{\dagger,t}^{ij}$ is a consistent estimator for  $\beta_\dagger^{ij}$ for every action pair $(i,j)$.
\end{corollary}

For every fixed $h>0$, all three exponential terms in Theorem~\ref{conthm:wls_tail_misspecified} converge to zero because $t\epsilon_t^2 \to \infty$ as $t \to \infty$.

\begin{theorem}[\textbf{\emph{Convergence of the contextual behavior strategies under model misspecification}}]
\label{conthm:behavior_misspecified}
Suppose that the conditions of Corollary~\ref{concoro:wls_consistency_misspecified} hold. Define the least-false contextual payoff matrix as $[M_\dagger(\bm z)]_{ij}=\bm z^\top\bm\beta_\dagger^{ij}$. If $\eta_t\to 0$ as $t \to \infty$, we have
\begin{equation*}
    \operatorname{dist} \left( (\widehat{\phi}_{\dagger,t}(\bm{z_t}), \widehat{\nu}_{\dagger,t}(\bm{z_t})),\mathcal{E}(M_{\dagger}(\bm{z}_t)) \right) \stackrel{p}{\to} 0.
    \end{equation*}
\end{theorem}
Additionally suppose that, for every context $\bm z$, the least-false contextual game $M_\dagger(\bm z)$ admits a unique Nash equilibrium $\left(\phi_\dagger^\ast(\bm z),\nu_\dagger^\ast(\bm z)\right)$.
If $\epsilon_t\to\epsilon_\infty$ and $c_t/\sqrt{(t-1)\epsilon_{t-1}^2} \to 0$ as $t \to \infty$, then
\[
\Vert
\widetilde{\phi}_{\dagger,t}(\bm z_t) - \widetilde{\phi}_\dagger^{\ast}(\bm z_t)
\Vert_2
\stackrel{p}{\to} 0,
\qquad
\Vert
\widetilde{\nu}_{\dagger,t}(\bm z_t) - \widetilde{\nu}_\dagger^{\ast}(\bm z_t)
\Vert_2
\stackrel{p}{\to} 0,
\]
where the limiting behavior strategies under different exploration
strategies are given by:
\begin{enumerate}
    \item[\textnormal{(1)}] UCB and TS:
    $\widetilde{\phi}_\dagger^{\ast}(\bm z) =
    \operatorname{Clip}_{m}
    \left(
    \phi_\dagger^\ast(\bm z),
    \frac{\epsilon_\infty}{m}
    \right)$,
    $\widetilde{\nu}_\dagger^{\ast}(\bm z) =
    \operatorname{Clip}_{k}
    \left(
    \nu_\dagger^\ast(\bm z),
    \frac{\epsilon_\infty}{k}
    \right)$.
    \item[\textnormal{(2)}] $\epsilon$-Greedy:
    $\widetilde{\phi}_\dagger^{\ast}(\bm z) =
    (1-\epsilon_\infty)\phi_\dagger^\ast(\bm z)
    + \frac{\epsilon_\infty}{m}\bm 1_m$,
    $\widetilde{\nu}_\dagger^{\ast}(\bm z) =
    (1-\epsilon_\infty)\nu_\dagger^\ast(\bm z)
    + \frac{\epsilon_\infty}{k}\bm 1_k$.
\end{enumerate}
We can use similar arguments in Section~\ref{app:conthm:nash_con_proof} and Section~\ref{app:conthm:nash_behave_proof} to prove Theorem~\ref{conthm:behavior_misspecified}.

\begin{theorem}[\textbf{\emph{Asymptotic normality of the contextual WLS estimator under model misspecification}}]
\label{conthm:wls_normality_misspecified}
Suppose that conditions in Theorem~\ref{conthm:behavior_misspecified} hold. Then for every action pair $(i,j)$, we have $\pi_t^{ij} - [\widetilde{\phi}_\dagger^{\ast}(\bm z_t)]_i [\widetilde{\nu}_\dagger^{\ast}(\bm z_t)]_j \stackrel{p}{\to} 0$.
Suppose further that there exists a constant $c_\pi>0$ such that $\min_{i\in[m],\,j\in[k]} [\widetilde{\phi}_\dagger^\ast(\bm z)]_i [\widetilde{\nu}_\dagger^\ast(\bm z)]_j \ge c_\pi$ for any $\bm z\sim P_z$. Then,
\[
\sqrt t
\left(
\widehat{\bm\beta}_{\dagger,t}^{ij}
-
\beta_\dagger^{ij}
\right)
\stackrel{d}{\to}
\mathcal N_d
\left(
\bm 0,
\Sigma^{-1}H_{ij}\Sigma^{-1}
\right),
\]
where \[
H_{ij} =
\mathbb E
\left[
\frac{
\bm z\bm z^\top
\left[
\left\{
\varphi^{ij}(\bm z)-\bm z^\top\bm\beta_\dagger^{ij}
\right\}^2
+
\sigma_{ij}^2
\right]
}{
[\widetilde{\phi}_\dagger^\ast(\bm z)]_i
[\widetilde{\nu}_\dagger^\ast(\bm z)]_j
}
\right].
\]
A consistent estimator for $\Sigma^{-1} H_{ij} \Sigma^{-1}$ is given by $(\widehat{\Sigma}_{\dagger,t}^{ij})^{-1} \widehat H_{\dagger,t}^{ij} (\widehat{\Sigma}_{\dagger,t}^{ij})^{-1}$ with
\[
\widehat{\Sigma}_{\dagger,t}^{ij}
=
\frac{1}{t}
\sum_{s=1}^t
\frac{a_{\dagger,s}^{ij}}{\pi_{\dagger,s}^{ij}}
\bm z_s\bm z_s^\top
\]
and
\[
\widehat H_{\dagger,t}^{ij}
= \frac{1}{t}
\sum_{s=1}^t
\frac{a_{\dagger,s}^{ij}}{(\pi_{\dagger,s}^{ij})^2}
\bm z_s\bm z_s^\top
\left(
r_s -
\bm z_s^\top
\widehat{\bm\beta}_{\dagger,t}^{ij}
\right)^2.
\]
\end{theorem}
\begin{remark}
The condition $\min_{i\in[m],\,j\in[k]} [\widetilde{\phi}_\dagger^\ast(\bm z)]_i [\widetilde{\nu}_\dagger^\ast(\bm z)]_j \ge c_\pi$ is automatically satisfied when $\epsilon_t\to\epsilon_\infty>0$.
\end{remark}

\begin{proof} 
\label{app:conthm:wls_normality_misspecified_proof}
By Equation~\eqref{equ:wls_context_misspecified}, 
\begin{align} 
\label{equ:wls_normality_decomposition} 
\sqrt t \left( \widehat{\bm\beta}_{\dagger,t}^{ij} - \bm\beta_\dagger^{ij} \right) = \left\{ \frac{1}{t} \sum_{s=1}^t \frac{ a_{\dagger,s}^{ij} }{ \pi_{\dagger,s}^{ij} } \bm z_s\bm z_s^\top \right\}^{-1} \left\{ \frac{1}{\sqrt t} \sum_{s=1}^t \frac{ a_{\dagger,s}^{ij} }{ \pi_{\dagger,s}^{ij} } \bm z_s \left[ \varphi^{ij}(\bm z_s) - \bm z_s^\top\bm\beta_\dagger^{ij} + e_{s(i,j)} \right] \right\}. 
\end{align} 

\paragraph{Step 1: asymptotic normality of the weighted score.}
We first show that 
\begin{align} 
\label{equ:wls_weighted_score_clt} 
\frac{1}{\sqrt t} \sum_{s=1}^t \frac{ a_{\dagger,s}^{ij} }{ \pi_{\dagger,s}^{ij} } \bm z_s \left[ \varphi^{ij}(\bm z_s) - \bm z_s^\top\bm\beta_\dagger^{ij} + e_{s(i,j)} \right] \stackrel{d}{\to} \mathcal N_d \left( \bm 0, H_{ij} \right). 
\end{align} 
By Cramer-Wold device, it suffices to show that, for every fixed $\bm v\in\mathbb R^d$, 
\[ 
\frac{1}{\sqrt t} \sum_{s=1}^t D_{s}^{ij}(\bm v) \stackrel{d}{\to} \mathcal N \left( 0, \bm v^\top H_{ij}\bm v \right), 
\] 
where 
\[ 
D_s^{ij}(\bm v) = \frac{ a_{\dagger,s}^{ij} }{ \pi_{\dagger,s}^{ij} } \bm v^\top\bm z_s \left[ \varphi^{ij}(\bm z_s) - \bm z_s^\top\bm\beta_\dagger^{ij} + e_{s(i,j)} \right]. 
\] 
We first verify that $\{D_s^{ij}(\bm v)\}_{s\ge1}$ is a martingale-difference sequence. Using iterated conditional expectation, we obtain 
\begin{align*} 
\mathbb E \left[ D_s^{ij}(\bm v) \mid \mathcal H_{s-1}^{\mathrm{con}} \right] = \mathbb E \Bigg[ \left. \mathbb E \Bigg[ \left. \frac{ a_{\dagger,s}^{ij} }{ \pi_{\dagger,s}^{ij} } \bm v^\top\bm z_s \left[ \varphi^{ij}(\bm z_s) - \bm z_s^\top\bm\beta_\dagger^{ij} + e_{s(i,j)} \right] \right| \mathcal H_{s-1}^{\mathrm{con}}, \bm z_s \Bigg] \right| \mathcal H_{s-1}^{\mathrm{con}} \Bigg]. 
\end{align*} 
For the approximation-error term, 
\begin{align*} 
\mathbb E \left[ \left. \frac{ a_{\dagger,s}^{ij} }{ \pi_{\dagger,s}^{ij} } \bm v^\top\bm z_s \left[ \varphi^{ij}(\bm z_s) - \bm z_s^\top\bm\beta_\dagger^{ij} \right] \right| \mathcal H_{s-1}^{\mathrm{con}}, \bm z_s \right] &= \bm v^\top\bm z_s \left[ \varphi^{ij}(\bm z_s) - \bm z_s^\top\bm\beta_\dagger^{ij} \right] \frac{ \mathbb E \left[ a_{\dagger,s}^{ij} \mid \mathcal H_{s-1}^{\mathrm{con}}, \bm z_s \right] }{ \pi_{\dagger,s}^{ij} } \\ &= \bm v^\top\bm z_s \left[ \varphi^{ij}(\bm z_s) - \bm z_s^\top\bm\beta_\dagger^{ij} \right]. 
\end{align*} 
For the reward-noise term, 
\begin{align*} 
\mathbb E \left[ \left. \frac{ a_{\dagger,s}^{ij} }{ \pi_{\dagger,s}^{ij} } \bm v^\top\bm z_s e_{s(i,j)} \right| \mathcal H_{s-1}^{\mathrm{con}}, \bm z_s \right] = \frac{ \bm v^\top\bm z_s }{ \pi_{\dagger,s}^{ij} } \mathbb E \left[ \left. a_{\dagger,s}^{ij}e_{s(i,j)} \right| \mathcal H_{s-1}^{\mathrm{con}}, \bm z_s \right] &= 0. 
\end{align*} 
Therefore, 
\begin{align*} 
\mathbb E \left[ D_s^{ij}(\bm v) \mid \mathcal H_{s-1}^{\mathrm{con}} \right] = \bm v^\top \mathbb E \left[ \bm z \left\{ \varphi^{ij}(\bm z) - \bm z^\top\bm\beta_\dagger^{ij} \right\} \right] = \bm 0, 
\end{align*} 
which comes from the first-order optimality condition associated with the least-false parameter. Thus, $\mathbb E \left[ D_s^{ij}(\bm v) \mid \mathcal H_{s-1}^{\mathrm{con}} \right] = 0$.  \\
We next verify the conditional Lindeberg condition. Denote $u_s^{ij}=\varphi^{ij}(\bm z_s)-\bm z_s^\top\bm\beta_\dagger^{ij}+e_{s(i,j)}$.
For any fixed $\bm v\in\mathbb R^d$ and any $\delta>0$, it suffices to show that
\begin{align}
\label{equ:wls_conditional_lindeberg}
\frac{1}{t}
\sum_{s=1}^t
\mathbb E
\Bigg[
\left.
\frac{(a_{\dagger,s}^{ij})^2}{(\pi_{\dagger,s}^{ij})^2}
(\bm v^\top\bm z_s)^2 (u_s^{ij})^2
\right.
\left.
\times
\mathbb I
\left\{
\left|
\frac{a_{\dagger,s}^{ij}}{\pi_{\dagger,s}^{ij}}
\bm v^\top\bm z_s
u_s^{ij}
\right|
>
\delta\sqrt t
\right\}
\right|
\mathcal H_{s-1}^{\mathrm{con}}
\Bigg]
\stackrel{p}{\to} 0.
\end{align}
By the asymptotic uniform positivity condition, with probability tending to one, there exists an integer $s_0$ such that $\pi_{\dagger,s}^{ij}\ge {c_\pi}/{2}$ for every $s\ge s_0$ and every $\bm z\in\mathcal Z$. Hence, for every $s\ge s_0$, the left-hand side of
\eqref{equ:wls_conditional_lindeberg} is bounded above by
\begin{align*}
&\frac{4\|\bm v\|_2^2L_z^2}{tc_\pi^2}
\sum_{s=s_0}^t
\mathbb E
\Bigg[
\left.
a_{\dagger,s}^{ij}
(u_s^{ij})^2
\right.
\left.
\times
\mathbb I
\left\{
a_{\dagger,s}^{ij}
(u_s^{ij})^2 >
\frac{\delta^2tc_\pi^2}{4\|\bm v\|_2^2L_z^2}
\right\}
\right|
\mathcal H_{s-1}^{\mathrm{con}}
\Bigg] \\
\le &
\frac{4\|\bm v\|_2^2L_z^2}{c_\pi^2}
\mathbb E
\Bigg[
(u_s^{ij})^2
\times
\mathbb I
\left\{
(u_s^{ij})^2 >
\frac{\delta^2tc_\pi^2}{4\|\bm v\|_2^2L_z^2}
\right\}
\Bigg].
\end{align*}
The term inside the expectation brackets is dominated by $\left[\varphi^{ij}(\bm z_s)-\bm z_s^\top\bm\beta_\dagger^{ij}+e_{s(i,j)}
\right]^2$,
with $$\mathbb E \left[ \left[ \varphi^{ij}(\bm z) - \bm z^\top\bm\beta_\dagger^{ij} + e_{s(i,j)} \right]^2 \right] \le L_\varphi^2 +\sigma_{ij}^2<\infty.$$
Together with
\[
\frac{
\delta^2tc_\pi^2
}{
4\|\bm v\|_2^2L_z^2
}
\to
\infty,
\]
we have
\begin{align*}
\left[
\varphi^{ij}(\bm z)
-
\bm z^\top\bm\beta_\dagger^{ij}
+
e^{ij}
\right]^2
\times
\mathbb I
\left\{
\left[
\varphi^{ij}(\bm z)
-
\bm z^\top\bm\beta_\dagger^{ij}
+
e^{ij}
\right]^2
>
\frac{
\delta^2tc_\pi^2
}{
4\|\bm v\|_2^2L_z^2
}
\right\}
\to
0
\end{align*}
by applying Dominated Convergence Theorem. The finitely many terms corresponding to $s<s_0$ also vanish after division by $t$. 

We next calculate the limiting conditional covariance matrix. Since
\[
\mathbb E
\left[
\left.
\frac{
a_{\dagger,s}^{ij}
}{
\pi_{\dagger,s}^{ij}
}
\bm z_s
\left[
\varphi^{ij}(\bm z_s)
-
\bm z_s^\top\bm\beta_\dagger^{ij}
+ e_{s(i,j)}
\right]
\right|
\mathcal H_{s-1}^{\mathrm{con}}
\right]
= \bm 0,
\]
the conditional covariance matrix is equal to the conditional second-moment matrix. Then
\begin{align*}
&\frac1t
\sum_{s=1}^t
\mathbb E
\left[
\left.
\left\{
D_s^{ij}(\bm v)
\right\}^2
\right|
\mathcal H_{s-1}^{\mathrm{con}}
\right] \\
=&
\frac1t
\sum_{s=1}^t
\mathbb E
\left[
\left.
\frac{(a_{\dagger,s}^{ij})^2}{(\pi_{\dagger,s}^{ij})^2}
(\bm v^\top\bm z_s)^2
\left[
\varphi^{ij}(\bm z_s)
-
\bm z_s^\top\bm\beta_\dagger^{ij}
+
e_{s(i,j)}
\right]^2
\right|
\mathcal H_{s-1}^{\mathrm{con}}
\right] \\
=& \frac1t
\sum_{s=1}^t
\mathbb E
\left[
\left.
\frac{
(\bm v^\top\bm z_s)^2}{\pi_{\dagger,s}^{ij}(\bm z_s)}
\mathbb E
\left[
\left.
\left[
\varphi^{ij}(\bm z_s)
-
\bm z_s^\top\bm\beta_\dagger^{ij}
+
e_{s(i,j)}
\right]^2
\right|
\mathcal H_{s-1}^{\mathrm{con}},
\bm z_s,
i_s=i,
j_s=j
\right]
\right|
\mathcal H_{s-1}^{\mathrm{con}}
\right] \\
=& \frac1t
\sum_{s=1}^t
\mathbb E
\left[
\left.
\frac{
(\bm v^\top\bm z_s)^2}{\pi_{\dagger,s}^{ij}(\bm z_s)}
\left[
\left\{
\varphi^{ij}(\bm z_s)
-
\bm z_s^\top\bm\beta_\dagger^{ij}
\right\}^2
+
\sigma_{ij}^2
\right]
\right|
\mathcal H_{s-1}^{\mathrm{con}}
\right] \\
=& \frac1t
\sum_{s=1}^t
\int_{\mathcal Z}
\frac{
(\bm v^\top\bm z)^2
}{
\pi_{\dagger,s}^{ij}(\bm z)
}
\left[
\left\{
\varphi^{ij}(\bm z)
-
\bm z^\top\bm\beta_\dagger^{ij}
\right\}^2
+
\sigma_{ij}^2
\right]
dP_{\bm z}(\bm z).
\end{align*}
By the convergence of the contextual behavior strategies, $\sup_{\bm z\in\mathcal Z}\left|\pi_{\dagger,s}^{ij}(\bm z)-\pi_\dagger^{ij}(\bm z)\right|\stackrel{p}{\to}0$,
where $\pi_\dagger^{ij}(\bm z)=[\widetilde{\phi}_\dagger^\ast(\bm z)]_i[\widetilde{\nu}_\dagger^\ast(\bm z)]_j$.
Moreover, since there exists $c_\pi>0$ such that $\inf_{\bm z\in\mathcal Z} \pi_\dagger^{ij}(\bm z) \ge c_\pi$.
Hence, with probability tending to one, $\inf_{\bm z\in\mathcal Z}\pi_{\dagger,s}^{ij}(\bm z)\ge {c_\pi}/{2}$ for all sufficiently large $s$, and consequently,
\[
\sup_{\bm z\in\mathcal Z}
\left|
\frac1{\pi_{\dagger,s}^{ij}(\bm z)}
-
\frac1{\pi_\dagger^{ij}(\bm z)}
\right|
\stackrel{p}{\to} 0.
\]
Since $\|\bm z\|_2\le L_z$ and $\left| \varphi^{ij}(\bm z) - \bm z^\top\bm\beta_\dagger^{ij} \right| \le L_\varphi$,
the integrand is uniformly bounded by an integrable envelope. Therefore,
\[
\int_{\mathcal Z}
\frac{(\bm v^\top\bm z)^2}{\pi_{\dagger,s}^{ij}(\bm z)}
\left[
\left\{
\varphi^{ij}(\bm z)
-
\bm z^\top\bm\beta_\dagger^{ij}
\right\}^2
+
\sigma_{ij}^2
\right]
dP_{\bm z}(\bm z)
\stackrel{p}{\to}
\int_{\mathcal Z}
\frac{(\bm v^\top\bm z)^2}{\pi_\dagger^{ij}(\bm z)}
\left[
\left\{
\varphi^{ij}(\bm z)
-
\bm z^\top\bm\beta_\dagger^{ij}
\right\}^2
+
\sigma_{ij}^2
\right]
dP_{\bm z}(\bm z).
\]
By Martingale Central Limit Theorem, it follows that
\[
\frac1{\sqrt t}
\sum_{s=1}^t
\frac{a_{\dagger,s}^{ij}}{\pi_{\dagger,s}^{ij}}
\bm z_s
\left[
\varphi^{ij}(\bm z_s)
-
\bm z_s^\top\bm\beta_\dagger^{ij}
+e_{s(i,j)}
\right]
\stackrel{d}{\to}
\mathcal N_d
\left(
\bm 0,
H_{ij}
\right)
\]
where 
$$H_{ij}
=
\mathbb E
\left[
\frac{\bm z\bm z^\top}{\pi_\dagger^{ij}(\bm z)}
\left[
\left\{
\varphi^{ij}(\bm z)
-
\bm z^\top\bm\beta_\dagger^{ij}
\right\}^2
+
\sigma_{ij}^2
\right]
\right].$$

\paragraph{Step 2: convergence of the weighted design matrix.}
Next, we find the limit of the second moment term. Using Lemma~\ref{lem:matrix-vector}, we find the limit of
\[
\widehat{\xi}_t
=
\frac{1}{t}
\sum_{s=1}^t
\frac{a_{\dagger,s}^{ij}}{\pi_{\dagger,s}^{ij}}
\bm v^\top \bm z_s \bm z_s^\top \bm v,
\]
for any $\bm v\in\mathbb R^d$. Let $\bm z\sim \mathcal P_{\bm z}$. We have
\[
\mathbb E
\left(
\frac{2}{c_\pi}
\bm v^\top \bm z\bm z^\top \bm v
\right)
\le
\frac{
2\bm v^\top \bm 1\bm 1^\top \bm v L_z^2
}{
c_\pi
}
<\infty.
\]
Moreover, for any $\kappa \ge 0$ and each $s$,
\[
\mathbb P
\left(
\frac{a_{\dagger,s}^{ij}}{\pi_{\dagger,s}^{ij}}
\bm v^\top \bm z_s\bm z_s^\top \bm v
>
\kappa
\right)
\le
\mathbb P
\left(
\frac{2}{c_\pi}
\bm v^\top \bm z\bm z^\top \bm v
>
\kappa
\right).
\]
Then all conditions of Theorem~2.19 from \citet{hall1980martingale} are met,
so
\[
\frac{1}{t}
\sum_{s=1}^t
\left\{
\frac{a_{\dagger,s}^{ij}}{\pi_{\dagger,s}^{ij}}
\bm v^\top \bm z_s\bm z_s^\top \bm v
-
\mathbb E
\left(
\left.
\frac{a_{\dagger,s}^{ij}}{\pi_{\dagger,s}^{ij}}
\bm v^\top \bm z_s\bm z_s^\top \bm v
\right|
\mathcal F_{s-1}
\right)
\right\}
=
\widehat{\xi}_t
-
\mathbb E
\left(
\bm v^\top \bm z\bm z^\top \bm v
\right)
\stackrel{p}{\to}
0.
\]
By Lemma~\ref{lem:matrix-vector} and the Continuous Mapping Theorem,
\[
\left(
\frac{1}{t}
\sum_{s=1}^t
\frac{a_{\dagger,s}^{ij}}{\pi_{\dagger,s}^{ij}}
\bm z_s\bm z_s^\top
\right)^{-1}
\stackrel{p}{\to}
\left(
\int \bm z\bm z^\top dP_{\bm z}
\right)^{-1}
=
\Sigma^{-1}.
\]
Finally, combining the above results using Slutsky's theorem, we have
\[
\sqrt t
\left(
\widehat{\bm\beta}_{\dagger,t}^{ij}
-
\beta_\dagger^{ij}
\right)
\stackrel{d}{\to}
\mathcal N_d
\left(
\bm 0,
\Sigma^{-1}H_{ij}\Sigma^{-1}
\right).
\]

\paragraph{Step 3: consistency of the asymptotic variance estimator.} 
Since the consistency of $\widehat{\Sigma}_{\dagger,t}^{ij}$ follows from the Law of Large Numbers easily, it suffices to show $\widehat H_{\dagger,t}^{ij}\stackrel{p}{\to} H_{ij}$.
Expanding the squared term gives 

\begin{equation}
\label{equ:H_estimator_expansion} 
\widehat H_{\dagger,t}^{ij} = \frac{1}{t} \sum_{s=1}^t \frac{ a_{\dagger,s}^{ij} }{ (\pi_{\dagger,s}^{ij})^2 } \bm z_s\bm z_s^\top \left[\left[ \bm z_s^\top \left( \bm\beta_\dagger^{ij} - \widehat{\bm\beta}_{\dagger,t}^{ij} \right) \right]^2 +2u_s^{ij} \bm z_s^\top \left( \bm\beta_\dagger^{ij} - \widehat{\bm\beta}_{\dagger,t}^{ij} \right) + (u_s^{ij}) ^2\right],
\end{equation} 
where $u_s=r_s - \bm z_s^\top \bm\beta_\dagger^{ij} = \varphi^{ij}(\bm z_s) - \bm z_s^\top \bm\beta_\dagger^{ij} + e_{s(i,j)}$.
Now, the first term satisfies
\begin{align*}
&
\frac{1}{t}
\sum_{s=1}^t
\frac{a_{\dagger,s}^{ij}}{(\pi_{\dagger,s}^{ij})^2}
\bm z_s
\left[
\left(
\bm\beta_\dagger^{ij}
-
\widehat{\bm\beta}_{\dagger,t}^{ij}
\right)^\top
\bm z_s
\right]^2
\bm z_s^\top
=
\frac{
\left\|
\bm\beta_\dagger^{ij}
-
\widehat{\bm\beta}_{\dagger,t}^{ij}
\right\|_2^2}{t}
\sum_{s=1}^t
\frac{a_{\dagger,s}^{ij}\|\bm z_s\|_2^2\bm z_s\bm z_s^\top}{(\pi_{\dagger,s}^{ij})^2}
\stackrel{p}{\to} 0
\end{align*}
by Theorem~\ref{conthm:wls_tail_misspecified}, since the summand is bounded. The second term satisfies
\begin{align*}
&
\frac{1}{t}
\sum_{s=1}^t
\frac{2a_{\dagger,s}^{ij}}{(\pi_{\dagger,s}^{ij})^2}
\left(
\bm\beta_\dagger^{ij}
-
\widehat{\bm\beta}_{\dagger,t}^{ij}
\right)^\top
\bm z_s
\left\{
\varphi^{ij}(\bm z_s) 
-
\bm z_s^\top\bm\beta_\dagger^{ij}+e_{s(i,j)}
\right\}
\bm z_s\bm z_s^\top
\stackrel{p}{\to}
0,
\end{align*}
by Lemma~\ref{lem:matrix-vector} because, for any $\bm v\in\mathbb R^d$,
\[
\left(
\bm\beta_\dagger^{ij}
-
\widehat{\bm\beta}_{\dagger,t}^{ij}
\right)^\top
\left[
\frac{1}{t}
\sum_{s=1}^t
\frac{2a_{\dagger,s}^{ij}}{(\pi_{\dagger,s}^{ij})^2}
\bm z_s
\left\{
\phi_i(\bm z_s)
-
\bm z_s^\top\bm\beta_\dagger^{ij}
+e_{s(i,j)}
\right\}
\bm v^\top
\bm z_s\bm z_s^\top
\bm v
\right]
\stackrel{p}{\to} 0,
\]
by Theorem~\ref{conthm:wls_tail_misspecified} and Lemma~\ref{lem:adaptive-subgaussian}.
Finally, for any $\bm v\in\mathbb R^d$,
\[
\frac{1}{t}
\sum_{s=1}^t
\left(
\frac{a_{\dagger,s}^{ij}}{(\pi_{\dagger,s}^{ij})^2}
\bm v^\top\bm z_s\bm z_s^\top\bm v
(u_s^{ij}) ^2
-
\mathbb E
\left[
\left.
\frac{a_{\dagger,s}^{ij}}{(\pi_{\dagger,s}^{ij})^2}
\bm v^\top\bm z_s\bm z_s^\top\bm v
(u_s^{ij}) ^2
\right|
\mathcal H_{s-1}^{\mathrm{con}}
\right]
\right)
\stackrel{p}{\to} 0
\]
by Theorem~2.19 in \citet{hall1980martingale}, because for $\bm z\sim\mathcal P_z$ and a $\sigma$-sub-Gaussian $e\perp \bm z$ with $\sigma \ge \sigma_{ij}$,
\[
\mathbb E
\left[
\frac{4}{c_\pi^2}
\bm v^\top\bm z\bm z^\top\bm v
\left(
L_\varphi+e
\right)^2
\right]
\le
\frac{4}{c_\pi^2}
\bm v^\top\bm 1\bm 1^\top\bm v
L_z^2
\left(
L_\varphi^2+\sigma^2
\right)
<\infty,
\]
and for any $\kappa\ge0$ and sufficiently large $s$,
\[
\mathbb P
\left[
\frac{a_{\dagger,s}^{ij}}{(\pi_{\dagger,s}^{ij})^2}
\bm v^\top\bm z_s\bm z_s^\top\bm v
(u_s^{ij}) ^2
>
\kappa
\right]
\le
\mathbb P
\left[
\frac{4}{c_\pi^2}
\bm v^\top\bm z\bm z^\top\bm v
\left(
L_\varphi+e
\right)^2
>
\kappa
\right].
\]
Combining the above convergence with
\[
\frac{1}{t}
\sum_{s=1}^t
\mathbb E
\left[
\left.
\frac{
a_{\dagger,s}^{ij}
}{
(\pi_{\dagger,s}^{ij})^2
}
\bm v^\top\bm z_s\bm z_s^\top\bm v
(u_s^{ij}) ^2
\right|
\mathcal H_{s-1}^{\mathrm{con}}
\right]
\stackrel{p}{\to}
\bm v^\top H_{ij}\bm v,
\]
we have the last term converges in probability to $H_{ij}$ by Lemma~\ref{lem:matrix-vector}. Combining all three terms with the Continuous Mapping Theorem gives the consistency of the variance estimator.
\end{proof}

\begin{theorem}[\textbf{\emph{Consistency of the contextual value estimator under model misspecification}}]
\label{conthm:value_rootT_misspecified}
Suppose that the conditions of Theorem~\ref{conthm:wls_normality_misspecified} hold for every action pair
$(i,j)$. Let $\Phi(\bm z)=\left[\varphi^{ij}(\bm z)\right]_{i\in[m],j\in[k]}$
denote the true contextual payoff matrix. Define $$V_\dagger^\ast
= \mathbb E
\left[
\phi_\dagger^\ast(\bm z)^\top
\Phi(\bm z)
\nu_\dagger^\ast(\bm z)
\right].$$
The value estimator is given by
\[
\widehat V_{\dagger,T}^{\mathrm{conDR}}
=
\frac{1}{T} \sum_{t=1}^T \Bigg[ \widehat{\phi}_{\dagger,t}(\bm z_t)^\top \widehat M_{\dagger,t-1}(\bm z_t) \widehat{\nu}_{\dagger,t}(\bm z_t) + \frac{ [\widehat{\phi}_{\dagger,t}(\bm z_t)]_{i_t} [\widehat{\nu}_{\dagger,t}(\bm z_t)]_{j_t} }{ \pi_{\dagger,t}^{i_t,j_t} } \left\{ r_t-[\widehat M_{\dagger,t-1}(\bm z_t)]_{i_t,j_t} \right\} \Bigg].
\]
Assume further that $\eta_t=O(t^{-1/2})$. Then
\[
\widehat V_{\dagger,T}^{\mathrm{conDR}}
- V_\dagger^\ast = o_p(1).
\]
\end{theorem}

\begin{proof}
The proof is similar to that of Theorem~\ref{conthm:drrate}. 
We define 
\begin{align*} 
\bar V_{\dagger,T} = \frac{1}{T} \sum_{t=1}^T \Bigg[ \widehat{\phi}_{\dagger,t}(\bm z_t)^\top \Phi(\bm z_t) \widehat{\nu}_{\dagger,t}(\bm z_t) + \frac{ [\widehat{\phi}_{\dagger,t}(\bm z_t)]_{i_t} [\widehat{\nu}_{\dagger,t}(\bm z_t)]_{j_t} }{ \pi_{\dagger,t}^{i_t,j_t} } \left\{ r_t-\varphi^{i_t,j_t}(\bm z_t) \right\} \Bigg]. 
\end{align*} 
\textbf{Step 1:} We first show that $\widehat V_{\dagger,T}^{\mathrm{conDR}} - \bar V_{\dagger,T} = O_p(T^{-1/2})$.
By definition, 
\begin{align*} 
\widehat V_{\dagger,T}^{\mathrm{conDR}} - \bar V_{\dagger,T} =\,& \frac{1}{T} \sum_{t=1}^T \Bigg[ \widehat{\phi}_{\dagger,t}(\bm z_t)^\top \left\{ \widehat M_{\dagger,t-1}(\bm z_t) - \Phi(\bm z_t) \right\} \widehat{\nu}_{\dagger,t}(\bm z_t) \\ 
&\qquad+ \frac{ [\widehat{\phi}_{\dagger,t}(\bm z_t)]_{i_t} [\widehat{\nu}_{\dagger,t}(\bm z_t)]_{j_t} }{ \pi_{\dagger,t}^{i_t,j_t}} \left\{ \varphi^{i_t,j_t}(\bm z_t) - [\widehat M_{\dagger,t-1}(\bm z_t)]_{i_t,j_t} \right\} \Bigg]. 
\end{align*} 
Denote the summand above by $\Delta_t$. Conditional on $\mathcal H_{t-1}^{\mathrm{con}}$ and $\bm z_t$, we have 
\begin{align*} 
& \mathbb E \left[ \left. \frac{ [\widehat{\phi}_{\dagger,t}(\bm z_t)]_{i_t} [\widehat{\nu}_{\dagger,t}(\bm z_t)]_{j_t} }{ \pi_{\dagger,t}^{i_t,j_t}} \left\{ \varphi^{i_t,j_t}(\bm z_t) - [\widehat M_{\dagger,t-1}(\bm z_t)]_{i_t,j_t} \right\} \right| \mathcal H_{t-1}^{\mathrm{con}}, \bm z_t \right] \\ =& \sum_{i=1}^m \sum_{j=1}^k [\widehat{\phi}_{\dagger,t}(\bm z_t)]_i [\widehat{\nu}_{\dagger,t}(\bm z_t)]_j \left\{ \varphi^{ij}(\bm z_t) - [\widehat M_{\dagger,t-1}(\bm z_t)]_{ij} \right\} \\ =&\widehat{\phi}_{\dagger,t}(\bm z_t)^\top \left\{ \Phi(\bm z_t) - \widehat M_{\dagger,t-1}(\bm z_t) \right\} \widehat{\nu}_{\dagger,t}(\bm z_t). 
\end{align*} 
Therefore, $\mathbb E \left[ \Delta_t \mid \mathcal H_{t-1}^{\mathrm{con}}, \bm z_t \right] = 0$. Hence, $\left\{ \Delta_t, \mathcal H_t^{\mathrm{con}} \right\}_{t\ge1}$ is a martingale-difference sequence. By the asymptotic uniform boundedness of the importance ratios, the boundedness of $\bm z_t$, and the boundedness of the approximation error $\left| \varphi^{ij}(\bm z_t) - \bm z_t^\top\bm\beta_\dagger^{ij} \right| \le L_\varphi$, there exists a constant $C<\infty$ such that $\mathbb E \left[ \Delta_t^2 \mid \mathcal H_{t-1}^{\mathrm{con}} \right] \le C$ almost surely for all sufficiently large $t$. Therefore, using the same localization argument as Step 2 in Section~\ref{app:conthm:drrate_proof}, we have 
\[ \frac{1}{T} \sum_{t=1}^T \Delta_t = O_p(T^{-1/2}). \] 
Thus, $\widehat V_{\dagger,T}^{\mathrm{conDR}} - \bar V_{\dagger,T} = O_p(T^{-1/2})$.

\textbf{Step 2:} We next show that $\bar V_{\dagger,T}-V_\dagger^\ast=O_p(T^{-1/2})$.
Decompose $\bar V_{\dagger,T}-V_\dagger^\ast=J_1+J_2$,
where
\begin{align*}
J_1
&=
\frac{1}{T}
\sum_{t=1}^T
\widehat{\phi}_{\dagger,t}(\bm z_t)^\top
\Phi(\bm z_t)
\widehat{\nu}_{\dagger,t}(\bm z_t)
-
V_\dagger^\ast,
\\
J_2
&=
\frac{1}{T}
\sum_{t=1}^T
\frac{
[\widehat{\phi}_{\dagger,t}(\bm z_t)]_{i_t}
[\widehat{\nu}_{\dagger,t}(\bm z_t)]_{j_t}}{\pi_{\dagger,t}^{i_tj_t}}
\left\{
r_t-\varphi^{i_tj_t}(\bm z_t)
\right\}.
\end{align*}
For $J_1$, adding and subtracting $\phi_\dagger^\ast(\bm z_t)^\top\Phi(\bm z_t)\nu_\dagger^\ast(\bm z_t)$,
we obtain
\begin{align*}
J_1=
\frac{1}{T}
\sum_{t=1}^T
\left[
\widehat{\phi}_{\dagger,t}(\bm z_t)^\top
\Phi(\bm z_t)
\widehat{\nu}_{\dagger,t}(\bm z_t)
-
\phi_\dagger^\ast(\bm z_t)^\top
\Phi(\bm z_t)
\nu_\dagger^\ast(\bm z_t)
\right]
+
\frac{1}{T}
\sum_{t=1}^T
\left[
\phi_\dagger^\ast(\bm z_t)^\top
\Phi(\bm z_t)
\nu_\dagger^\ast(\bm z_t)
-
V_\dagger^\ast
\right].
\end{align*}
By convergence of the target strategies in Theorem~\ref{conthm:behavior_misspecified} and the boundedness of $\Phi(\bm z)$, the first average converges to 0. By the central limit theorem, the second average is \(O_p(T^{-1/2})\). Hence, $J_1 \stackrel{p}{\to} 0$.

For \(J_2\), by an argument similar to Step 1 of Section~\ref{app:conthm:wls_normality_misspecified_proof}, we have
\[
\mathbb E
\left[
\left.
\frac{
[\widehat{\phi}_{\dagger,t}(\bm z_t)]_{i_t}
[\widehat{\nu}_{\dagger,t}(\bm z_t)]_{j_t}}{\pi_{\dagger,t}^{i_tj_t}(\bm z_t)}
\left\{
r_t-\varphi^{i_tj_t}(\bm z_t)
\right\}
\right\vert
\mathcal H_{t-1}^{\mathrm{con}}
\right]
=
0,
\]
which is therefore a martingale difference sequence. By the asymptotic uniform boundedness of the importance ratios and the boundedness of the conditional noise variances, the conditional variance is bounded for all sufficiently large $t$. Therefore, $J_2 = O_p(T^{-1/2})$.
Combining the bounds for $J_1$ and $J_2$, we obtain $\bar V_{\dagger,T} \stackrel{p}{\to} V_\dagger^\ast$. Together with Step~1, this gives $\widehat V_{\dagger,T}^{\mathrm{conDR}} \stackrel{p}{\to} V_\dagger^\ast$.

Under model misspecification, the pseudo value is defined with respect to the pseudo-optimal strategies. However, the value is still evaluated under the true contextual payoff matrix $\Phi(\bm z)$. Since the pseudo-target strategies are
not necessarily saddle-point optimal for $\Phi(\bm z)$, the plug-in value error contains first-order policy-estimation terms. Thus, policy consistency yields consistency of the value estimator, but a $\sqrt{T}$ rate requires an additional
orthogonality condition or a sufficiently fast policy convergence rate.
\end{proof}

\section{Auxiliary Results}
\label{app:lemmas}

\begin{lemma}[Adaptive Chernoff bound]
\label{lem:adaptive_chernoff}
Let $\{\mathcal H_s\}_{s=0}^t$ be a filtration. Suppose that
$a_s\in\{0,1\}$ and $\mathbb E(a_s| \mathcal H_{s-1})=\pi_s$,
where $\pi_s$ is $\mathcal H_{s-1}$-measurable. Let $S_t=\sum_{s=1}^t a_s$, and $\mu_t=\sum_{s=1}^t \pi_s$.
If $\mu_t\ge \mu$ almost surely for some $\mu>0$, then for any
$c\in(0,1)$,
\[
\mathbb P\left(S_t\le (1-c)\mu_t\right)
\le
\exp\left\{
-\frac{c^2\mu}{2}
\right\}.
\]
\end{lemma}

\begin{proof}
Fix any $c\in(0,1)$ and $p=\log(1-c)<0$.
Since $a_s\in\{0,1\}$ and $\mathbb E(a_s\mid \mathcal H_{s-1})=\pi_s$,
we have
\[
\mathbb E\left[\exp(p a_s)\mid \mathcal H_{s-1}\right]
= 1-\pi_s+\pi_s e^p
= 1+\pi_s(e^p-1)
\le \exp\left\{\pi_s(e^p-1)\right\}.
\]
Therefore, $\mathbb E\left[
\exp\left\{ p a_s-(e^p-1)\pi_s \right\}
\mid \mathcal H_{s-1} \right] \le 1$.
Denote $\sum_{r=1}^s Z_r = \exp\left\{ p a_s-(e^p-1)\pi_s \right\}$.
Then $\mathbb E(\sum_{r=1}^s Z_r\mid \mathcal H_{s-1})\le 1$.
Moreover,
\[
\prod_{s=1}^t \sum_{r=1}^s Z_r =
\exp\left\{
p\sum_{s=1}^t a_s - (e^p-1)\sum_{s=1}^t\pi_s
\right\}
= \exp\left\{ pS_t-(e^p-1)\mu_t \right\},
\]
where $S_t=\sum_{s=1}^t a_s$ and $\mu_t=\sum_{s=1}^t \pi_s$.
By iterating conditional expectations,
\begin{align*}
\mathbb E\left[\prod_{s=1}^t \sum_{r=1}^s Z_r\right]
&=
\mathbb E\left[
\mathbb E\left(
\prod_{s=1}^t \sum_{r=1}^s Z_r
\mid \mathcal H_{t-1}
\right)
\right] \\
&=
\mathbb E\left[
\left(\prod_{s=1}^{t-1}\sum_{r=1}^s Z_r\right)
\mathbb E(Y_t\mid \mathcal H_{t-1})
\right] \\
&\le
\mathbb E\left[
\prod_{s=1}^{t-1}\sum_{r=1}^s Z_r
\right] \\
&\le \cdots \le 1.
\end{align*}
Hence, $\mathbb E\left[\exp\left\{pS_t-(e^p-1)\mu_t\right\}\right] \le 1$.
Now consider the event $\mathcal E=\left\{S_t\le (1-c)\mu_t\right\}$.
Since $p<0$, on $\mathcal E$ we have $pS_t \ge p(1-c)\mu_t$.
Therefore, on $\mathcal E$,
\[
pS_t-(e^p-1)\mu_t
\ge
\left[
p(1-c)-(e^p-1)
\right]\mu_t.
\]
Because $p=\log(1-c)$, we have $e^p=1-c$, and hence
\[
p(1-c)-(e^p-1)
=
(1-c)\log(1-c)+c.
\]
Let $c^\prime=(1-c)\log(1-c)+c$. Then $c^\prime>0$ for $c\in(0,1)$. Therefore, on $\mathcal E$, $pS_t-(e^p-1)\mu_t\ge c^\prime \mu_t$.
If $\mu_t\ge \mu$ almost surely for some $\mu>0$, then on $\mathcal E$, $pS_t-(e^p-1)\mu_t\ge c^\prime \mu$.
Thus,
\[
\mathcal E
\subseteq
\left\{
\exp\left(
pS_t-(e^p-1)\mu_t
\right)
\ge
\exp(c^\prime\mu)
\right\}.
\]
By Markov's inequality,
\begin{align*}
\mathbb P(\mathcal E)
&\le
\mathbb P\left(
\exp\left(
pS_t-(e^p-1)\mu_t
\right)
\ge
\exp(c^\prime\mu)
\right) \\
&\le
\frac{
\mathbb E\left[
\exp\left(
pS_t-(e^p-1)\mu_t
\right)
\right]
}{
\exp(c^\prime\mu)
} \\
&\le
\exp(-c^\prime\mu).
\end{align*}
Therefore,
\[
\mathbb P\left(
S_t\le (1-c)\mu_t
\right)
\le
\exp\left\{
-\mu\left[
c+(1-c)\log(1-c)
\right]
\right\}.
\]
Finally, since
\[
c+(1-c)\log(1-c)
\ge
\frac{c^2}{2},
\]
we obtain
\[
\mathbb P\left(
S_t\le (1-c)\mu_t
\right)
\le
\exp\left\{
-\frac{c^2\mu}{2}
\right\}.
\]
\end{proof}

\begin{lemma}[Adapted matrix Chernoff bound]
\label{lem:adapted-matrix-chernoff}
Let $\{\mathcal H_s\}_{s\ge0}$ be a filtration and let $\{Z_s\}_{s=1}^t$ be an adapted sequence of random positive semidefinite matrices satisfying $0 \preceq Z_s \preceq R I_d$ almost surely for some positive $R$. Define $W_t = \sum_{s=1}^t \mathbb E\left( Z_s \mid \mathcal H_{s-1} \right)$. Suppose that $W_t \succeq \mu_t I_d$ almost surely for some deterministic $\mu_t>0$. Then, for any
\[
\mathbb P\left(
\lambda_{\min}
\left(
\sum_{s=1}^t Z_s
\right)
\le
(1-\delta)\mu_t
\right)
\le
d \exp\left\{- \frac{\delta^2\mu_t}{2R} \right\}.
\]
\end{lemma}

\begin{proof}
Fix $\theta>0$ and define
\[
g(\theta) = \frac{1-e^{-\theta R}}{R}.
\]
Since $e^{-\theta x}$ is convex for $x$ on $[0,R]$, we have $e^{-\theta x} \le 1-g(\theta)x$ for every $x\in[0,R]$. By functional calculus, the same inequality holds in
the semidefinite order:
\[
e^{-\theta Z_s}
\preceq
I_d-g(\theta)Z_s.
\]
Taking conditional expectations gives
\[
\mathbb E\left(
e^{-\theta Z_s}
\mid
\mathcal H_{s-1}
\right)
\preceq
I_d-g(\theta)\mathbb E\left( Z_s \mid \mathcal H_{s-1} \right).
\]
Since $I_d-H \preceq e^{-H}$ for every positive semidefinite matrix $H$, it follows that
\[
\mathbb E\left(
e^{-\theta Z_s}
\mid
\mathcal H_{s-1}
\right)
\preceq
\exp\left\{
-g(\theta)\mathbb E\left( Z_s \mid \mathcal H_{s-1} \right)
\right\}.
\]
Equivalently,
\[
\log
\mathbb E\left(
e^{-\theta Z_s}
\mid
\mathcal H_{s-1}
\right)
\preceq
-g(\theta)\mathbb E\left( Z_s \mid \mathcal H_{s-1} \right).
\]
Consider the trace-exponential process
\[
L_s(\theta) =
\operatorname{Tr} \left(
\exp\left\{
-\theta \sum_{r=1}^s Z_r
+ g(\theta)W_s
\right\} \right).
\]
We now show that $\{L_s(\theta)\}_{s\ge0}$ is a nonnegative supermartingale. By Lieb's theorem and Jensen's inequality,
\begin{align*}
& \mathbb E\left[ L_s(\theta) \mid \mathcal H_{s-1} \right] \\
=&
\mathbb E\left[
\left.
\operatorname{Tr}
\left( \exp\left\{
-\theta \sum_{r=1}^{s-1} Z_r
+ g(\theta)W_s
- \theta Z_s
\right\}
\right|
\mathcal H_{s-1}
\right) \right] \\
\le &
\operatorname{Tr} \left(
\exp\left\{
-\theta \sum_{r=1}^{s-1} Z_r
+ g(\theta)W_s
+ \log \mathbb E\left( e^{-\theta Z_s} \mid \mathcal H_{s-1} \right)
\right\} \right) \\
\le &
\operatorname{Tr} \left(
\exp\left\{
-\theta \sum_{r=1}^{s-1} Z_r
+
g(\theta)W_s
- g(\theta)\mathbb E\left( Z_s \mid \mathcal H_{s-1} \right)
\right\} \right) \\
= &
\operatorname{Tr} \left(
\exp\left\{
-\theta \sum_{r=1}^{s-1} Z_r
+ g(\theta)W_{s-1}
\right\} \right) \\
= & L_{s-1}(\theta).
\end{align*}
Therefore, $\mathbb E\left[ L_t(\theta) \right] \le L_0(\theta) = \operatorname{Tr}(I_d) = d$.
Now consider the lower-tail event 
$$\mathcal E_t=\left\{ \lambda_{\min}\left(\sum_{s=1}^t Z_s\right) \le (1-\delta)\mu_t \right\}.$$
On this event $\mathcal E_t$, let $\bm v \in \mathbb R^d$ be a unit eigenvector corresponding to $\lambda_{\min}(\sum_{s=1}^t Z_s)$. Since $W_t \succeq \mu_t I_d$,
we have
\begin{align*}
\lambda_{\max}
\left(
-\theta \sum_{s=1}^t Z_s
+ g(\theta)W_t
\right)
\ge
\bm v^\top
\left(
-\theta \sum_{s=1}^t Z_s
+ g(\theta)W_t
\right)
\bm v
\ge
-\theta(1-\delta)\mu_t
+ g(\theta)\mu_t.
\end{align*}
Hence, $L_t(\theta) \ge \exp\left\{ \left[ g(\theta) - \theta(1-\delta) \right] \mu_t \right\}$ on $\mathcal E_t$. By Markov's inequality,
\[
\mathbb P(\mathcal E_t)
\le
\mathbb P(L_t(\theta) \ge \exp\left\{ \left[ g(\theta) - \theta(1-\delta) \right] \mu_t \right\})
\le d \exp\left\{
- \left[
g(\theta)
- \theta(1-\delta)
\right]
\mu_t
\right\}.
\]
It remains to optimize over $\theta>0$. Taking $\theta = -\frac{1}{R} \log(1-\delta)$
gives $e^{-\theta R} = 1-\delta$ and $g(\theta) = \frac{\delta}{R}$.
Therefore,
\[
g(\theta) - \theta(1-\delta)
=
\frac{\delta + (1-\delta)\log(1-\delta)}{R}.
\]
Substituting this expression yields
\[
\mathbb P(\mathcal E_t)
\le
d \exp\left\{ -
\frac{\mu_t}{R}
\left[\delta + (1-\delta)\log(1-\delta)\right]
\right\}.
\]
Finally, using
\[
\delta
+
(1-\delta)\log(1-\delta)
\ge
\frac{\delta^2}{2}
\]
we have
\[
\mathbb P\left(
\lambda_{\min}
\left(
\sum_{s=1}^t Z_s
\right)
\le
(1-\delta)\mu_t
\right)
\le
d \exp\left\{
- \frac{\delta^2\mu_t}{2R}
\right\}.
\]
\end{proof}

\begin{lemma}[Self-normalized concentration for adaptive samples]
\label{lem:adaptive-subgaussian}
Fix an action pair $(i,j)$. Define $a_s^{ij} = \mathbb I(i_s=i,j_s=j)$ and $\mathcal N_t^{ij}=\sum_{s=1}^t a_s^{ij}$.
Suppose that, conditional on the event $a_s^{ij}=1$, the noise term $e_s$ is mean-zero and $\sigma_{ij}$-sub-Gaussian, i.e.,
\[
\mathbb E
\left[
\exp(\lambda e_s)|
a_s^{ij}=1
\right]
\le
\exp\left\{
\frac{\lambda^2\sigma_{ij}^2}{2}
\right\}
\]
for every $\lambda\in\mathbb R$. Then, for every $h>0$ and $n>0$,
\[
\mathbb P\left(
\left|
\sum_{s=1}^t a_s^{ij} e_s
\right|
>
h\mathcal N_t^{ij},
\quad
\mathcal N_t^{ij}\ge n
\right)
\le
2
\exp\left\{
-\frac{nh^2}{2\sigma_{ij}^2}
\right\}.
\]
\end{lemma}

\begin{proof}
For any fixed $\lambda\in\mathbb R$, define
\[
L_t^{ij}(\lambda)
=
\exp\left\{
\lambda \sum_{s=1}^t a_s^{ij} e_s
-
\frac{\lambda^2\sigma_{ij}^2}{2}
\mathcal N_t^{ij}
\right\}.
\]
We first show that $\{L_t^{ij}(\lambda),\mathcal H_t\}_{t\ge0}$ is a nonnegative supermartingale. Since
\begin{align*}
L_t^{ij}(\lambda)
&=
L_{t-1}^{ij}(\lambda)
\exp\left\{ \lambda a_t^{ij} e_t
- \frac{\lambda^2\sigma_{ij}^2}{2} a_t^{ij}
\right\},
\end{align*}
then,
\begin{align*}
\mathbb E\left[
L_t^{ij}(\lambda)
\mid
\mathcal H_{t-1}
\right] =
L_{t-1}^{ij}(\lambda)
\,
\mathbb E\left[
\left.
\exp\left\{
\lambda a_t^{ij} e_t 
- \frac{\lambda^2\sigma_{ij}^2}{2} a_t^{ij}
\right\}
\right|
\mathcal H_{t-1}
\right].
\end{align*}
Let $\pi_t(i,j) = \mathbb P\left( a_t^{ij}=1 \mid \mathcal H_{t-1} \right)$.
Conditioning on whether $a_t^{ij}=0$ or $a_t^{ij}=1$, we obtain
\begin{align*}
&
\mathbb E\left[
\left.
\exp\left\{
\lambda e_t a_t^{ij}
-
\frac{\lambda^2\sigma_{ij}^2}{2}
a_t^{ij}
\right\}
\right|
\mathcal H_{t-1}
\right]
\\
=&
1-\pi_t(i,j)+
\pi_t(i,j)
\mathbb E\left[
\left.
\exp\left\{
\lambda e_t
-
\frac{\lambda^2\sigma_{ij}^2}{2}
\right\}
\right|
\mathcal H_{t-1},
a_t^{ij}=1
\right].
\end{align*}
By the conditional sub-Gaussian definition,
\[
\mathbb E\left[
\left.
\exp\left\{
\lambda e_t -
\frac{\lambda^2\sigma_{ij}^2}{2}
\right\}
\right|
\mathcal H_{t-1},
a_t^{ij}=1
\right]
=
\mathbb E\left[
\left.
\exp\left\{
\lambda e_t -
\frac{\lambda^2\sigma_{ij}^2}{2}
\right\}
\right|
a_t^{ij}=1
\right] \le 1.
\]
Hence, $\mathbb E\left[ L_t^{ij}(\lambda) \mid \mathcal H_{t-1} \right] \le L_{t-1}^{ij}(\lambda)$.
Thus, $\left\{ L_t^{ij}(\lambda), \mathcal H_t \right\}_{t\ge 0}$ is a nonnegative supermartingale. Since $L_0^{ij}(\lambda)=1$, it follows that $\mathbb E\left[ L_t^{ij}(\lambda) \right] \le 1$. We next bound the upper tail.
For any $\lambda>0$, on the event $\left\{ \sum_{s=1}^t a_s^{ij} e_s  > h\mathcal N_t^{ij}, \, \mathcal N_t^{ij}\ge n \right\}$, we have
\[
\lambda \sum_{s=1}^t a_s^{ij} e_s
- \frac{\lambda^2\sigma_{ij}^2}{2}
\mathcal N_t^{ij} >
\left(\lambda h - \frac{\lambda^2\sigma_{ij}^2}{2} \right)
\mathcal N_t^{ij}.
\]
Choose $\lambda = {h}/{\sigma_{ij}^2}$, then we get
\[
\lambda h - \frac{\lambda^2\sigma_{ij}^2}{2}
= \frac{h^2}{2\sigma_{ij}^2}.
\]
Since $\mathcal N_t^{ij}\ge n$ on this event, we obtain
\[
\lambda \sum_{s=1}^t e_s a_s^{ij}
-
\frac{\lambda^2\sigma_{ij}^2}{2}
\mathcal N_t^{ij}
>
\frac{nh^2}{2\sigma_{ij}^2}.
\]
Therefore, by Markov's inequality,
\begin{align*}
\mathbb P \left( \sum_{s=1}^t a_s^{ij} e_s  > h\mathcal N_t^{ij}, \, \mathcal N_t^{ij}\ge n \right)
&\le
\mathbb P\left(
L_t^{ij}(\lambda)
>
\exp\left\{
\frac{nh^2}{2\sigma_{ij}^2}
\right\}
\right)
\\
&\le
\exp\left\{
-\frac{nh^2}{2\sigma_{ij}^2}
\right\}
\mathbb E\left[
L_t^{ij}(\lambda)
\right]
\\
&\le
\exp\left\{
-\frac{nh^2}{2\sigma_{ij}^2}
\right\},
\end{align*}
where the last inequality comes from $\mathbb E\left[ L_t^{ij}(\lambda) \right] \le 1$.
Applying the same argument to $-\sum_{s=1}^t e_s a_s^{ij}$ yields
\[
\mathbb P\left(
\sum_{s=1}^t e_s a_s^{ij}
<
-h\mathcal N_t^{ij},
\quad
\mathcal N_t^{ij}\ge n
\right)
\le
\exp\left\{
-\frac{nh^2}{2\sigma_{ij}^2}
\right\}.
\]
Combining the upper and lower tails by the union bound gives
\[
\mathbb P\left(
\left|
\sum_{s=1}^t e_s a_s^{ij}
\right|
>
h\mathcal N_t^{ij},
\quad
\mathcal N_t^{ij}\ge n
\right)
\le
2
\exp\left\{
-\frac{nh^2}{2\sigma_{ij}^2}
\right\}.
\]
\end{proof}

\begin{lemma}[Self-normalized contextual martingales]
\label{lem:context-self-normalized}
Fix an action pair $(i,j)$ and define $a_t^{ij} = \mathbb I\{i_t=i,j_t=j\}$, $V_t^{ij} = \sum_{s=1}^t a_s^{ij}\bm z_s\bm z_s^\top$, and $S_t^{ij}= \sum_{s=1}^t a_s^{ij}\bm z_s e_s$.
Suppose that $\|\bm z\|_2\le L_z$ and that, conditional on the selected action pair, the noise term $e_t$ is mean-zero and $\sigma_{ij}$-sub-Gaussian whenever $(i_t,j_t)=(i,j)$.
Then, for any $u>0$,
\[
\mathbb P\left(
\left\|
\left(
I_d+V_t^{ij}
\right)^{-1/2}
S_t^{ij}
\right\|_2
>
u
\right)
\le
\left(
1+tL_z^2
\right)^{d/2}
\exp\left\{
-\frac{u^2}{2\sigma_{ij}^2}
\right\}.
\]
\end{lemma}

\begin{proof}
Fix an arbitrary vector $\bm q\in\mathbb R^d$. By the conditional
sub-Gaussianity of $e_s$, whenever $(i_s,j_s)=(i,j)$,
\[
\mathbb E\left[
\left.
\exp\left\{
(\bm q^\top\bm z_s)e_s
-
\frac{\sigma_{ij}^2}{2}
(\bm q^\top\bm z_s)^2
\right\}
\right|
\mathcal H_{s-1}^{\mathrm{con}},
\bm z_s,
i_s,
j_s
\right]
\le
1.
\]
If $(i_s,j_s)\neq(i,j)$, then $a_s^{ij}=0$, and the corresponding
exponential term is equal to one. Therefore,
\[
\mathbb E\left[
\left.
\exp\left\{
a_s^{ij}(\bm q^\top\bm z_s)e_s
-
\frac{\sigma_{ij}^2}{2}
a_s^{ij}(\bm q^\top\bm z_s)^2
\right\}
\right|
\mathcal H_{s-1}^{\mathrm{con}},
\bm z_s,
i_s,
j_s
\right]
\le
1.
\]
Applying the law of iterated expectations sequentially over
$s=1,\ldots,t$, we obtain
\begin{align}
\label{equ:fixed-q-supermartingale}
\mathbb E\left[
\exp\left\{
\bm q^\top S_t^{ij}
-
\frac{\sigma_{ij}^2}{2}
\bm q^\top V_t^{ij}\bm q
\right\}
\right]
\le
1.
\end{align}
We next average \eqref{equ:fixed-q-supermartingale} over an auxiliary Gaussian vector $\bm q
\sim
\mathcal N_d
\left(
\bm 0,
\sigma_{ij}^{-2}I_d
\right)$,
which is independent of the observed data. Moreover, because the integrand is nonnegative, Tonelli's theorem allows us to exchange the order of integration. Therefore,
\begin{align}
\label{equ:mixture-expectation}
\mathbb E\left[
\mathbb E_{\bm q}
\left[
\exp\left\{
\bm q^\top S_t^{ij}
-
\frac{\sigma_{ij}^2}{2}
\bm q^\top V_t^{ij}\bm q
\right\}
\right]
\right]
\le
1.
\end{align}
To evaluate the inner expectation, recall that the density of $\bm q\sim\mathcal N_d(\bm 0,\sigma_{ij}^{-2}I_d)$ is
\[
\left(
\frac{\sigma_{ij}^2}{2\pi}
\right)^{d/2}
\exp\left\{
-\frac{\sigma_{ij}^2}{2}
\bm q^\top\bm q
\right\}.
\]
Hence,
\begin{align*}
&
\mathbb E_{\bm q}
\left[
\exp\left\{
\bm q^\top S_t^{ij}
-
\frac{\sigma_{ij}^2}{2}
\bm q^\top V_t^{ij}\bm q
\right\}
\right] =
\left(
\frac{\sigma_{ij}^2}{2\pi}
\right)^{d/2}
\int_{\mathbb R^d}
\exp\left\{
\bm q^\top S_t^{ij}
-
\frac{\sigma_{ij}^2}{2}
\bm q^\top
\left(
I_d+V_t^{ij}
\right)
\bm q
\right\}
d\bm q.
\end{align*}
Completing the square gives
\begin{align*}
&
\bm q^\top S_t^{ij}
-
\frac{\sigma_{ij}^2}{2}
\bm q^\top
\left(
I_d+V_t^{ij}
\right)
\bm q
\\
=&
-\frac{\sigma_{ij}^2}{2}
\left[
\bm q
-
\frac{1}{\sigma_{ij}^2}
\left(
I_d+V_t^{ij}
\right)^{-1}
S_t^{ij}
\right]^\top
\left(
I_d+V_t^{ij}
\right)
\left[
\bm q
-
\frac{1}{\sigma_{ij}^2}
\left(
I_d+V_t^{ij}
\right)^{-1}
S_t^{ij}
\right]
\\
&\quad+
\frac{1}{2\sigma_{ij}^2}
(S_t^{ij})^\top
\left(
I_d+V_t^{ij}
\right)^{-1}
S_t^{ij}.
\end{align*}
Recall that, for any positive-definite matrix $A\succ0$ and any vector $\bm\mu\in\mathbb R^d$, 
\[ \int_{\mathbb R^d} \exp\left\{ -\frac12 (\bm q-\bm\mu)^\top A (\bm q-\bm\mu) \right\} d\bm q = (2\pi)^{d/2} \det(A)^{-1/2}. \] 
Applying this formula with $A = \sigma_{ij}^2 \left( I_d+V_t^{ij} \right)$ and $\bm\mu = \frac{1}{\sigma_{ij}^2} \left( I_d+V_t^{ij} \right)^{-1} S_t^{ij}$, we obtain 
\begin{align*} 
& \int_{\mathbb R^d} \exp\Bigg\{ -\frac{\sigma_{ij}^2}{2} \left[ \bm q - \frac{1}{\sigma_{ij}^2} \left( I_d+V_t^{ij} \right)^{-1} S_t^{ij} \right]^\top \left( I_d+V_t^{ij} \right) \left[ \bm q - \frac{1}{\sigma_{ij}^2} \left( I_d+V_t^{ij} \right)^{-1} S_t^{ij} \right] \Bigg\} d\bm q \\ 
=& (2\pi)^{d/2} \det\left\{ \sigma_{ij}^2 \left( I_d+V_t^{ij} \right) \right\}^{-1/2} 
= \left( \frac{2\pi}{\sigma_{ij}^2} \right)^{d/2} \det\left( I_d+V_t^{ij} \right)^{-1/2}. 
\end{align*} 
Therefore, 
\begin{align} 
\label{equ:gaussian-mixture} 
\mathbb E_{\bm q} \left[ \exp\left\{ \bm q^\top S_t^{ij} - \frac{\sigma_{ij}^2}{2} \bm q^\top V_t^{ij}\bm q \right\} \right] 
= \frac{ \exp\left\{ \frac{1}{2\sigma_{ij}^2} (S_t^{ij})^\top \left( I_d+V_t^{ij} \right)^{-1} S_t^{ij} \right\} }{ \det\left( I_d+V_t^{ij} \right)^{1/2} }. 
\end{align}
Combining \eqref{equ:mixture-expectation} and
\eqref{equ:gaussian-mixture} yields
\begin{align}
\label{equ:mixture-random-variable}
\mathbb E\left[
\frac{
\exp\left\{
\frac{1}{2\sigma_{ij}^2}
(S_t^{ij})^\top
\left(
I_d+V_t^{ij}
\right)^{-1}
S_t^{ij}
\right\}
}{
\det\left(
I_d+V_t^{ij}
\right)^{1/2}
}
\right]
\le
1.
\end{align}
It remains to upper bound the determinant. Since $\operatorname{Tr} \left( V_t^{ij} \right) = \sum_{s=1}^t a_s^{ij} \|\bm z_s\|_2^2 \le tL_z^2$,
the arithmetic--geometric mean inequality implies
\begin{align*}
\det\left(
I_d+V_t^{ij}
\right)
&\le
\left[
1+
\frac{\operatorname{Tr}(V_t^{ij})}{d}
\right]^d \le
\left(
1+tL_z^2
\right)^d.
\end{align*}
If $\left\| \left( I_d+V_t^{ij} \right)^{-1/2} S_t^{ij} \right\|_2 > u$, then $(S_t^{ij})^\top \left( I_d+V_t^{ij} \right)^{-1} S_t^{ij} > u^2$.
Consequently,
\begin{equation*}
\frac{
\exp\left\{
\frac{1}{2\sigma_{ij}^2}
(S_t^{ij})^\top
\left(
I_d+V_t^{ij}
\right)^{-1}
S_t^{ij}
\right\}
}{
\det\left(
I_d+V_t^{ij}
\right)^{1/2}
} >
\left(
1+tL_z^2
\right)^{-d/2}
\exp\left\{
\frac{u^2}{2\sigma_{ij}^2}
\right\}.    
\end{equation*}
Finally, applying Markov's inequality to the nonnegative random variable in
\eqref{equ:mixture-random-variable}, we obtain
\[
\mathbb P\left(
\left\|
\left(
I_d+V_t^{ij}
\right)^{-1/2}
S_t^{ij}
\right\|_2
>
u
\right)
\le
\left(
1+tL_z^2
\right)^{d/2}
\exp\left\{
-\frac{u^2}{2\sigma_{ij}^2}
\right\}.
\]
This completes the proof.
\end{proof}

We restate the implication of Theorem 22 of \citet{feinberg2022continuity} that is used in our proof. 
We also state the sufficient conditions from Remarks 22 and 23 that we use to verify Assumptions (A1) and (A2).
\begin{lemma}[Adapted continuity of equilibrium; adapted from Theorem 22 and Remark 22 of \citet{feinberg2022continuity}]
\label{lem:adapted-continuity-equilibria}
Let \(\Theta=\mathbb R^{m\times k}\times \mathbb R\) be the parameter space, and for 
\((A,\eta)\in\Theta\), define
\[
f(\bm x,\bm y,A,\eta)
=
\bm x^\top A\bm y-\eta R_X(\bm x)+\eta R_Y(\bm y),
\qquad
(\bm x,\bm y)\in \Delta_m\times \Delta_k .
\]
For each \((A,\eta)\in\Theta\), define the saddle-point correspondence
\[
\mathcal S(A,\eta)
=
\left\{
(\bm x,\bm y)\in\Delta_m\times\Delta_k:
f(\bm x,\bm y',A,\eta)
\le
f(\bm x,\bm y,A,\eta)
\le
f(\bm x',\bm y,A,\eta),
\ \forall \bm x'\in\Delta_m,\ \bm y'\in\Delta_k
\right\}.
\]
Suppose the following conditions hold: \\
(A1) For any fixed \((A,\eta)\in\Theta\) and \(\bm x\in\Delta_m\), the function $\bm y\mapsto f(\bm x,\bm y,A,\eta)$ is continuous on \(\Delta_k\), and the maximization problem $\max_{\bm y\in\Delta_k} f(\bm x,\bm y,A,\eta)$ attains its maximum. \\
(A2) For any fixed \((A,\eta)\in\Theta\) and \(\bm y\in\Delta_k\), the function $\bm x\mapsto f(\bm x,\bm y,A,\eta)$ is continuous on \(\Delta_m\), and the minimization problem $\min_{\bm x\in\Delta_m} f(\bm x,\bm y,A,\eta)$ attains its minimum. \\
(A3) The row player's feasible-set correspondence $(A,\eta)\mapsto \Delta_m$ is lower semi-continuous. \\
(A4) The column player's feasible-set correspondence $(A,\eta)\mapsto \Delta_k$ is lower semi-continuous. \\
Then the saddle-point correspondence \(\mathcal S(A,\eta)\) is upper semi-continuous in 
\((A,\eta)\). \\
In particular, if \((A_n,\eta_n)\to(A_0,\eta_0)\) and the limiting game has a unique saddle point, $\mathcal S(A_0,\eta_0)=\{(\bm x_0,\bm y_0)\}$, then for any sequence \((\bm x_n,\bm y_n)\in\mathcal S(A_n,\eta_n)\), we have $\bm x_n\to\bm x_0$, and $\bm y_n\to\bm y_0$. \\
Moreover, Assumption (A1) can be verified by the following two sufficient conditions from Remark 22 of \citet{feinberg2022continuity}: \\
(i) The function \(f\) is lower semi-continuous. \\
(ii) Whenever the parameters \((A_n,\eta_n)\) and the column player's strategy 
\(\bm y_n\) converge to a feasible limit, any feasible sequence of row-player strategies 
\(\bm x_n\in\Delta_m\) with bounded payoff values has a limit point in \(\Delta_m\). \\
Similarly, Assumption (A2) can be verified by the following two sufficient conditions from Remark 23 of \citet{feinberg2022continuity}: \\
(i) The function \(-f\) is lower semi-continuous. \\
(ii) Whenever the parameters \((A_n,\eta_n)\) and the row player's strategy 
\(\bm x_n\) converge to a feasible limit, any feasible sequence of column-player strategies 
\(\bm y_n\in\Delta_k\) with bounded payoff values has a limit point in \(\Delta_k\). \\
\end{lemma}

\begin{lemma}[\emph{Bound of saddle value difference}]
\label{lemma:sd_dif}
Suppose that $g:X\times Y \times I \subseteq \mathbb{R}^m \times \mathbb{R}^k \times \mathbb{R}^p \rightarrow \mathbb{R}$ is a continuous function, with $X,Y$ compact and $I = [a_1,b_1]\times[a_2,b_2]\times \dots \times [a_p,b_p]$, and also suppose ${\nabla_t} g:X\times Y \times I \subseteq \mathbb{R}^m \times \mathbb{R}^k \times \mathbb{R}^p \rightarrow \mathbb{R}^p$ is its continuous gradient. Then the value function $V:I\subseteq\mathbb{R}^p \rightarrow \mathbb{R}$ is given by
\begin{equation*}
    V(\bm{t})=\min_{\bm{x}\in X} \max_{\bm{y}\in Y}g(\bm{x},\bm{y},\bm{t}) = g(X^\ast(\bm{t}),Y^\ast(\bm{t}),\bm{t}),
\end{equation*}
where $X^\ast(\bm{t})=\arg\min_{\bm{x}\in X} \max_{\bm{y}}g(\bm{x},\bm{y},\bm{t})$ and $Y^\ast(\bm{t})=\arg \max_{\bm{y}}\min_{\bm{x}\in X} g(\bm{x},\bm{y},\bm{t})$. Then we have
\begin{equation}
    \nabla_tg(X^\ast(\bm{t}_1),Y^\ast(\bm{t}_2),\bar{\bar{\bm{t}}})(\bm{t}_1-\bm{t}_2) \le V(\bm{t}_1) - V(\bm{t}_2) \le \nabla_tg(X^\ast(\bm{t}_2),Y^\ast(\bm{t}_1),\bar{\bm{t}})(\bm{t}_1-\bm{t}_2),
\end{equation}
for some $\bar{\bm{t}}$ and $\bar{\bar{\bm{t}}}$ satisfying $\bar{\bm{t}} = 
(1-\theta_1)\bm{t}_1+\theta_1\bm{t}_2$ and $\bar{\bar{\bm{t}}} = 
(1-\theta_2)\bm{t}_1+\theta_2\bm{t}_2$, $0<\theta_1,\theta_2<1$.

\begin{proof}
Building on the foundational work of \citet{de2009general}, who established the absolute continuity of the maximum value with respect to multidimensional parameters in their Theorem 2, we extend this result to minimax problems and saddle values through a combination with Theorem 4 in \citet{milgrom2002envelope}.

By the definition of a saddle point, we can write
\begin{align*}
    V(\bm{t}_1) - V(\bm{t}_2) &= g(X^\ast(\bm{t_1}),Y^\ast(\bm{t_1}),\bm{t_1}) - g(X^\ast(\bm{t_2}),Y^\ast(\bm{t_2}),\bm{t_2}) \\
    & \le g(X^\ast(\bm{t_2}),Y^\ast(\bm{t_1}),\bm{t_1}) - g(X^\ast(\bm{t_2}),Y^\ast(\bm{t_1}),\bm{t_2}) \\
    & \le \nabla_t g(X^\ast(\bm{t_2}),Y^\ast(\bm{t_1}),\bar{\bm{t}})(\bm{t}_1-\bm{t}_2),
\end{align*}
for some $\bar{\bm{t}}$ between $\bm{t}_1$ and $\bm{t}_2$,where in the last inequality we use the the Mean Value Theorem. So we can express $\bar{\bm{t}}$ as $(1-\theta_1)\bm{t}_1+\theta_1\bm{t}_2$ with $0<\theta_1<1$. Similarly, we have
\begin{align*}
    V(\bm{t}_1) - V(\bm{t}_2) &= g(X^\ast(\bm{t_1}),Y^\ast(\bm{t_1}),\bm{t_1}) - g(X^\ast(\bm{t_2}),Y^\ast(\bm{t_2}),\bm{t_2}) \\
    & \ge g(X^\ast(\bm{t_1}),Y^\ast(\bm{t_2}),\bm{t_1}) - g(X^\ast(\bm{t_1}),Y^\ast(\bm{t_2}),\bm{t_2}) \\
    & \ge \nabla_t g(X^\ast(\bm{t_1}),Y^\ast(\bm{t_2}),\bar{\bar{\bm{t}}})(\bm{t}_1-\bm{t}_2),
\end{align*}
for some $\bar{\bm{t}} = (1-\theta_2)\bm{t}_1+\theta_2\bm{t}_2$  with $0<\theta_2<1$. Hence, the proof is completed.
\end{proof}
\end{lemma}

The following Lemma~\ref{lemma1_in_chen} is adapted from Lemma 1 of \cite{chen2021statistical}.
\begin{lemma}[Adapted from Lemma~1 of \citet{chen2021statistical}]
\label{lemma1_in_chen}
Suppose $\{\mathcal{H}_t : t = 1,\cdots,T\}$ is an increasing filtration of $\sigma$-fields. Let $\{W_t : t = 1,\cdots,T\}$ be a sequence of random variables such that $W_t$ is $\mathcal{H}_{t-1}$-measurable and $|W_t| \leq L_z$ almost surely for all $t$. Let $\{e_t : t = 1,\cdots,T\}$ be independent $\sigma$-subgaussian, and $e_t \perp \mathcal{H}_{t-1}$ for all $t$. Let $S=\{s_1,\cdots,s_{|S|}\}\subseteq \{1,\cdots,T\}$ be an index set where $|S|$ is the number of elements in $S$. Then for $\kappa>0$,
\[
P\left(\sum_{s\in S} W_s e_s \geq x\right)
\leq
\exp\left\{
-\frac{\kappa^2}{2|S|\sigma^2 L_z^2}
\right\}.
\]
\end{lemma}

The following Lemma~\ref{lem:matrix-vector} is adapted from Lemma~6 of
\citet{chen2021statistical}.
\begin{lemma}[Adapted from Lemma~6 of \citet{chen2021statistical}]
\label{lem:matrix-vector}
Let $\{X_n\}_{n\ge1}$ be a sequence of $d\times d$ symmetric random matrices, and let $X$ be a $d\times d$ symmetric matrix. Then $X_n \stackrel{p}{\to} X$ if and only if $\bm v^\top X_n \bm v \stackrel{p}{\to} \bm v^\top X \bm v$ for every $\bm v\in\mathbb R^d$.
\end{lemma}

\begin{lemma}[Concentration of the contextual IPW terms]
\label{lem:context_ipw_score_concentration}
Fix an action pair $(i,j)$. Let $a_s^{ij}=\mathbb I\{i_s=i,j_s=j\}$ and $\pi_s(i,j)=\mathbb P\left(i_s=i,j_s=j\mid\mathcal H_{s-1}^{\mathrm{con}},\bm z_s\right)$.
Suppose that $\epsilon_s$ is non-increasing and $\pi_s(i,j)\ge {\epsilon_s^2}/{mk}$ for all $s\le t$.
Suppose that $\|\bm z\|_2 \le L_z$, $|\varphi^{ij}(\bm z) - \bm z^\top\bm\beta_\dagger^{ij}| \le L_\varphi$ for any $\bm{z} \sim P_z$. 
Then, for any $u>0$ and some $\ell \in [d]$,
\begin{equation}
\label{equ:ipw_approximation_score_tail}
\mathbb P
\left(\left|
\frac{1}{t}\sum_{s=1}^t\frac{a_s^{ij}}{\pi_s(i,j)}z_{s,\ell} \left(\varphi^{ij}(\bm z_s) - \bm z_s^\top\bm\beta_\dagger^{ij}\right)\right|> u \right)
\le
2\exp\left\{-\frac{t\epsilon_t^2u^2}{2mk\left(L_z^2L_\varphi^2+L_zL_\varphi u/3\right)}\right\}.    
\end{equation}
Suppose additionally that, whenever $(i_s,j_s)=(i,j)$, the reward noise $e_s$ is conditionally mean-zero and $\sigma_{ij}$-sub-Gaussian. Then, for any $u>0$ and some $\ell \in [d]$,
\begin{equation}
\label{equ:ipw_noise_score_tail}
\mathbb P
\left(\left|
\frac{1}{t}\sum_{s=1}^t\frac{a_s^{ij}}{\pi_s(i,j)}z_{s,\ell}e_s
\right|>u\right)
\le
2\exp\left[-\frac{t\epsilon_t^2}{4mk}\min\left\{\frac{u^2}{L_z^2\sigma_{ij}^2},\frac{u}{L_z\sigma_{ij}}
\right\}
\right].
\end{equation}
\end{lemma}

\begin{proof}
We first prove~\eqref{equ:ipw_approximation_score_tail}. 
By iterated conditional expectation,
\begin{align*}
& \mathbb E
\left[
\frac{ a_s^{ij} }{ \pi_s(i,j) } z_{s,\ell} \left(\varphi^{ij}(\bm z_s) - \bm z_s^\top\bm\beta_\dagger^{ij} \right)
\mid
\mathcal H_{s-1}^{\mathrm{con}}
\right] \\
=&
\mathbb E
\left[\left.
\mathbb E\left(
\left.\frac{a_s^{ij}}{\pi_s(i,j)}z_{s,\ell}
\left(\varphi^{ij}(\bm z_s) - \bm z_s^\top\bm\beta_\dagger^{ij} \right)
\right|
\mathcal H_{s-1}^{\mathrm{con}},\bm z_s
\right)\right|\mathcal H_{s-1}^{\mathrm{con}}\right] \\
=&
\mathbb E
\left[
z_{s,\ell} \left.
\left(\varphi^{ij}(\bm z_s) - \bm z_s^\top\bm\beta_\dagger^{ij} \right) \right|\mathcal H_{s-1}^{\mathrm{con}} 
\right] \\
=& 0,
\end{align*}
where the last equality follows from the first-order optimality condition for the least-false parameter. Hence, $\{\frac{ a_s^{ij} }{ \pi_s(i,j) } z_{s,\ell} \left(\varphi^{ij}(\bm z_s) - \bm z_s^\top\bm\beta_\dagger^{ij} \right)\}_{s=1}^t$ is a martingale-difference sequence.
Moreover, since $\pi_s(i,j) \ge {\epsilon_t^2}/{mk}$, we have
\[
\left|
\frac{ a_s^{ij} }{ \pi_s(i,j) } z_{s,\ell} \left(\varphi^{ij}(\bm z_s) - \bm z_s^\top\bm\beta_\dagger^{ij} \right)
\right|
\le
\frac{mk L_zL_\varphi}{\epsilon_t^2}.
\]
Its conditional variance satisfies
\small{
\begin{equation*}
\mathbb E
\left[
\left.
\left(
\frac{ a_s^{ij} }{ \pi_s(i,j) } z_{s,\ell} \left(\varphi^{ij}(\bm z_s) - \bm z_s^\top\bm\beta_\dagger^{ij} \right)
\right)^2
\right|
\mathcal H_{s-1}^{\mathrm{con}}
\right]
\le
\mathbb E
\left[
\left.
\frac{
z_{s,\ell}^2
\left\{
\left(\varphi^{ij}(\bm z_s) - \bm z_s^\top\bm\beta_\dagger^{ij} \right)
\right\}^2}{\pi_s(i,j)}
\right|
\mathcal H_{s-1}^{\mathrm{con}}
\right]
\le 
\frac{mk L_z^2L_\varphi^2}{\epsilon_t^2}.
\end{equation*}
}
Therefore,
\[
\sum_{s=1}^t
\mathbb E
\left[
\left.
\left(
\frac{ a_s^{ij} }{ \pi_s(i,j) } z_{s,\ell} \left(\varphi^{ij}(\bm z_s) - \bm z_s^\top\bm\beta_\dagger^{ij} \right)
\right)^2
\right|
\mathcal H_{s-1}^{\mathrm{con}}
\right]
\le
\frac{tmkL_z^2L_\varphi^2}{\epsilon_t^2}.
\]
Applying Freedman's inequality yields
\begin{align*}
\mathbb P
\left(
\left|
\sum_{s=1}^t
\frac{ a_s^{ij} }{ \pi_s(i,j) } z_{s,\ell} \left(\varphi^{ij}(\bm z_s) - \bm z_s^\top\bm\beta_\dagger^{ij} \right)
\right|>tu
\right)
\le & \,
2\exp
\left\{
-
\frac{t^2u^2}{2\left[tmkL_z^2L_\varphi^2/\epsilon_t^2
+ (mk L_zL_\varphi/\epsilon_t^2)tu/3
\right]
}
\right\}
\\
= & \,
2\exp\left\{
-\frac{t\epsilon_t^2u^2/mk}{2\left(L_z^2L_\varphi^2+L_zL_\varphi u/3\right)}
\right\}.
\end{align*}
This proves~\eqref{equ:ipw_approximation_score_tail}.

We next prove~\eqref{equ:ipw_noise_score_tail}. Conditional on
$\mathcal H_{s-1}^{\mathrm{con}}$ and $\bm z_s$, we have
\begin{align*}
&
\mathbb E
\left[
\left.
\exp
\left(
q\frac{ a_s^{ij} }{ \pi_s(i,j) } z_{s,\ell} e_s
\right)
\right|
\mathcal H_{s-1}^{\mathrm{con}},
\bm z_s
\right] \\
=&
1-\pi_s(i,j)+ \pi_s(i,j)
\mathbb E
\left[
\left.
\exp
\left\{
\frac{
qz_{s,\ell}e_s}{\pi_s(i,j)}
\right\}
\right|
\mathcal H_{s-1}^{\mathrm{con}},
\bm z_s,i_s=i,j_s=j
\right] \\
\le &
1-\pi_s(i,j)+\pi_s(i,j)
\exp
\left\{
\frac{q^2z_{s,\ell}^2\sigma_{ij}^2}{2\pi_s^2(i,j)}
\right\},
\end{align*}
where the last inequality comes from the conditional sub-Gaussianity of $e_s$. 
For $|q| \le {\epsilon_t^2}/{2mkL_z\sigma_{ij}}$,
we have
\[
\frac{q^2z_{s,\ell}^2\sigma_{ij}^2}{2\pi_s^2(i,j)}
\le
\frac{\left(\frac{\epsilon_t^2}{2mkL_z\sigma_{ij}}\right)^2L_z^2\sigma_{ij}^2}{2\left(\frac{\epsilon_t^2}{mk} \right)^2}
\le \frac18.
\]
Using $\exp(v)-1 \le 2v$ when $0\le v\le\frac18$, we obtain
\begin{equation*}
\mathbb E
\left[
\left.
\exp
\left(
q\frac{ a_s^{ij} }{ \pi_s(i,j) } z_{s,\ell} e_s
\right)
\right|
\mathcal H_{s-1}^{\mathrm{con}},
\bm z_s
\right]
\le
1+\frac{q^2z_{s,\ell}^2\sigma_{ij}^2}{\pi_s(i,j)}
\le
\exp
\left\{
\frac{mk q^2L_z^2\sigma_{ij}^2}{\epsilon_s^2}
\right\}.
\end{equation*}
Iterating this conditional moment-generating-function bound gives
\[
\mathbb E
\left[
\exp
\left(
q\sum_{s=1}^t
\frac{ a_s^{ij} }{ \pi_s(i,j) } z_{s,\ell} e_s
\right)
\right]
\le
\exp
\left\{
\frac{tmkq^2L_z^2\sigma_{ij}^2}{\epsilon_t^2}
\right\}.
\]
We now apply the Chernoff method. Fix any $0<q\le {\epsilon_t^2}/{2mkL_z\sigma_{ij}}$.
Since the exponential function is strictly increasing, we have
\begin{equation*}
\mathbb P
\left(
\frac{1}{t}
\sum_{s=1}^t
\frac{ a_s^{ij} }{ \pi_s(i,j) } z_{s,\ell} e_s
>u\right)=
\mathbb P
\left(
\exp
\left\{q
\sum_{s=1}^t
\frac{ a_s^{ij} }{ \pi_s(i,j) } z_{s,\ell} e_s
\right\}
>
\exp\left\{qtu\right\}\right).
\end{equation*}
Applying Markov's inequality to the nonnegative random variable $\exp
\left\{
q\sum_{s=1}^t
\frac{ a_s^{ij} }{ \pi_s(i,j) } z_{s,\ell} e_s
\right\}$,
we obtain
\begin{align*}
\mathbb P
\left(
\exp
\left\{
q\sum_{s=1}^t
\frac{ a_s^{ij} }{ \pi_s(i,j) } z_{s,\ell} e_s
\right\}
>
\exp
\left\{
qtu
\right\}
\right)
&\le
\exp
\left\{-qtu\right\}
\mathbb E
\left[
\exp
\left\{
q\sum_{s=1}^t
\frac{ a_s^{ij} }{ \pi_s(i,j) } z_{s,\ell} e_s
\right\}
\right].
\end{align*}
Using the moment-generating-function bound derived above, we further have
\begin{align*}
\exp
\left\{
-qtu
\right\}
\mathbb E
\left[
\exp
\left\{
q\sum_{s=1}^t
\frac{ a_s^{ij} }{ \pi_s(i,j) } z_{s,\ell} e_s
\right\}
\right]
\le
\exp
\left\{-qtu+\frac{tmkq^2L_z^2\sigma_{ij}^2}{\epsilon_t^2}
\right\}.
\end{align*}
Therefore, for every $0<q\le\frac{\underline\pi_t}{2L_z\sigma_{ij}}$, we obtain
\begin{equation}
\label{equ:ipw_noise_chernoff}
\mathbb P
\left(
\frac{1}{t}
\sum_{s=1}^t
\frac{ a_s^{ij} }{ \pi_s(i,j) } z_{s,\ell} e_s
>u
\right)
\le
\exp
\left\{
-qtu+\frac{tmkq^2L_z^2\sigma_{ij}^2}{\epsilon_t^2}
\right\}.
\end{equation}
The unconstrained minimizer of the exponent is $q^\ast =\frac{u\epsilon_t^2}{2mkL_z^2\sigma_{ij}^2}$.
If $u\le L_z\sigma_{ij}$, then $q^\ast$ lies in the admissible range, and taking $q=q^\ast$ yields
\[
\mathbb P
\left(
\frac1t
\sum_{s=1}^t
\frac{ a_s^{ij} }{ \pi_s(i,j) } z_{s,\ell} e_s
>u\right)
\le
\exp
\left\{
-\frac{
t\epsilon_t^2u^2}{4mkL_z^2\sigma_{ij}^2}
\right\}.
\]
If $u>L_z\sigma_{ij}$, we take $q=\frac{\epsilon_t^2}{2mkL_z\sigma_{ij}}$, which gives
\[
\mathbb P
\left(
\frac1t
\sum_{s=1}^t
\frac{ a_s^{ij} }{ \pi_s(i,j) } z_{s,\ell} e_s
>u\right)
\le
\exp
\left\{
-\frac{t\epsilon_t^2u}{4mkL_z\sigma_{ij}}
\right\}.
\]
Combining the two cases and applying the same argument to
$-\frac{ a_s^{ij} }{ \pi_s(i,j) } z_{s,\ell} e_s$ yields
\[
\mathbb P
\left(
\left|
\frac1t
\sum_{s=1}^t
\frac{ a_s^{ij} }{ \pi_s(i,j) } z_{s,\ell} e_s
\right|
>u
\right)
\le
2\exp
\left[
-
\frac{t\epsilon_t^2}{4mk}
\min
\left\{
\frac{u^2}{L_z^2\sigma_{ij}^2},
\frac{u}{L_z\sigma_{ij}}
\right\}
\right].
\]
\end{proof}

\begin{lemma}[\emph{Properties of product norm}]
\label{lemma:norm}
Consider the product space $\mathbb{R}^{m \times k} \times \mathbb{R}$ equipped with the norm $\Vert (A,\eta) \Vert_2\coloneqq (\Vert A \Vert_F^2+\eta^2)^{1/2}$, where $\Vert A \Vert_F = \sqrt{\sum_{i=1} ^m \sum_{j=1}^k A_{ij}^2}$. Then $\Vert (\cdot\,,\cdot)\Vert_2$ defines a valid norm on $\mathbb{R}^{m \times k} \times \mathbb{R}$.
\end{lemma}

\begin{proof}
We will show $\Vert (\cdot\,,\cdot)\Vert_2$ satisfies three norm properties. Each property follows directly from the corresponding properties of the Frobenius norm $\|\cdot\|_F$ and the Euclidean norm $|\cdot|$ on $\mathbb{R}$.

\paragraph{1. Positive definiteness:} $\|(A, \eta)\|_2 \geq 0$ for all $(A, \eta) \in \mathbb{R}^{m \times k} \times \mathbb{R}$, with equality if and only if $A = 0$ and $\eta = 0$. 
$$\Vert (A,\eta) \Vert_2= (\Vert A \Vert_F^2+\eta^2)^{\frac{1}{2}}= \sqrt{{\sum_{i=1} ^m \sum_{j=1}^k A_{ij}^2} + \eta^2}=0 \iff A_{ij}=0,\eta=0.$$

\paragraph{2. Absolute Homogeneity:} For any scalar $\alpha \in \mathbb{R}$ and $(A, \eta) \in \mathbb{R}^{m \times k} \times \mathbb{R}$, $\|\alpha(A, \eta)\|_2 = |\alpha| \cdot \|(A, \eta)\|_2$.
$$\|\alpha(A, \eta)\|_2 = (\|\alpha A\|_F^2 + |\alpha\eta|^2)^{1/2} = (|\alpha|^2\|A\|_F^2 + |\alpha|^2\eta^2)^{1/2} = |\alpha|\|(A, \eta)\|_2.$$

\paragraph{3. Triangle Inequality:} For any $(A_1, \eta_1), (A_2, \eta_2) \in \mathbb{R}^{m \times k} \times \mathbb{R}$, $\|(A_1 + A_2, \eta_1 + \eta_2)\|_2 \leq \|(A_1, \eta_1)\|_2 + \|(A_2, \eta_2)\|_2$.\\
By Minkowski's inequality, we have $\|A_1 + A_2\|_F \leq \|A_1\|_F + \|A_2\|_F$ and $|\eta_1 + \eta_2| \leq |\eta_1| + |\eta_2|$.
\begin{equation*}
    \begin{aligned}
        \|(A_1 + A_2, \eta_1 + \eta_2)\|_2 &=\left(\Vert A_1+A_2 \Vert_F^2+(\eta_1+\eta_2)^2\right)^{\frac{1}{2}} \\
        & \le \left(\Vert A_1\Vert_F^2+\Vert A_2 \Vert_F^2+(\vert \eta_1\vert+\vert\eta_2\vert)^2\right)^{\frac{1}{2}} \\
        & \le \left(\Vert A_1\Vert_F^2+\vert \eta_1 \vert^2 \right)^{\frac{1}{2}} + \left(\Vert A_2\Vert_F^2+\vert \eta_2 \vert^2 \right)^{\frac{1}{2}} \\
        & \le \|(A_1, \eta_1)\|_2 + \|(A_2, \eta_2)\|_2.
    \end{aligned}
\end{equation*}
\end{proof}

\end{document}